\documentclass[twoside,11pt]{article}

\usepackage{blindtext}

\usepackage{jmlr2e}

\usepackage[nospace,noadjust]{cite}
\usepackage{amsmath,amssymb,amsfonts,times}
\usepackage{graphicx}
\usepackage{textcomp}
\usepackage{xcolor}
\usepackage{subfigure}
\usepackage{epsfig}
\usepackage{multirow}
\usepackage{algorithm}
\usepackage{algorithmicx}
\usepackage{algpseudocode}
\usepackage{array}
\usepackage{epstopdf}
\usepackage{color}
\usepackage{diagbox}
\usepackage{bm}
\usepackage{tcolorbox}
\usepackage{pdfsync}

\usepackage{paralist}
\usepackage{xcolor}
\usepackage{color}
\usepackage{graphicx}
\graphicspath{{./figs/}}
\usepackage{comment}
\usepackage{multirow}
\usepackage[toc, page]{appendix}

\usepackage{graphicx}
\usepackage{fancyhdr}
\usepackage{cite}
\usepackage{cleveref}

\renewcommand{\O}{\mathbb{O}}
\newcommand{\R}{\mathbb{R}}
\def \real    { \mathbb{R} }

\newcommand{\e}{\begin{equation}}
\newcommand{\ee}{\end{equation}}
\newcommand{\en}{\begin{equation*}}
\newcommand{\een}{\end{equation*}}
\newcommand{\eqn}{\begin{eqnarray}}
\newcommand{\eeqn}{\end{eqnarray}}
\newcommand{\bmat}{\begin{bmatrix}}
\newcommand{\emat}{\end{bmatrix}}

\DeclareMathAlphabet\mathbfcal{OMS}{cmsy}{b}{n}
\renewcommand{\P}[1]{\operatorname{\mathbb{P}}\left(#1\right)}
\newcommand{\E}{\operatorname{\mathbb{E}}}

\newcommand{\vct}[1]{\boldsymbol{#1}}
\newcommand{\mtx}[1]{\boldsymbol{#1}}

\newcommand{\<}{\langle}
\renewcommand{\>}{\rangle}

\newcommand{\trace}{\operatorname{trace}}

\newcommand{\rank}{\operatorname{rank}}

\newcommand{\dist}{\operatorname{dist}}

\newcommand{\set}[1]{\mathbb{#1}}

\DeclareMathOperator*{\argmin}{\text{arg~min}}
\DeclareMathOperator*{\argmax}{\text{arg~max}}
\def \st {\operatorname*{s.t.\ }}

\newcommand{\wh}{\widehat}

\newcommand{\wt}{\widetilde}
\newcommand{\ol}{\overline}

\newcommand{\parans}[1]{\left(#1\right)}

\newcommand{\calA}{\mathcal{A}}
\newcommand{\calB}{\mathcal{B}}

\newcommand{\calH}{\mathcal{H}}
\newcommand{\calI}{\mathcal{I}}

\newcommand{\calN}{\mathcal{N}}

\newcommand{\calP}{\mathcal{P}}

\newcommand{\calX}{\mathcal{X}}
\newcommand{\calY}{\mathcal{Y}}
\newcommand{\calZ}{\mathcal{Z}}

\newcommand{\va}{\vct{a}}

\newcommand{\vh}{\vct{h}}

\newcommand{\vr}{\vct{r}}

\newcommand{\vu}{\vct{u}}
\newcommand{\vv}{\vct{v}}

\newcommand{\vx}{\vct{x}}
\newcommand{\vy}{\vct{y}}

\newcommand{\vepsilon}{\vct{\epsilon}}

\newcommand{\mA}{\mtx{A}}
\newcommand{\mB}{\mtx{B}}
\newcommand{\mC}{\mtx{C}}
\newcommand{\mD}{\mtx{D}}
\newcommand{\mE}{\mtx{E}}

\newcommand{\mH}{\mtx{H}}

\newcommand{\mP}{\mtx{P}}

\newcommand{\mR}{\mtx{R}}

\newcommand{\mU}{\mtx{U}}
\newcommand{\mV}{\mtx{V}}

\newcommand{\mX}{\mtx{X}}
\newcommand{\mY}{\mtx{Y}}

\newcommand{\mSigma}{\mtx{\Sigma}}

\newcommand{\mId}{{\bf I}}

\newcommand{\setX}{\set{X}}

\graphicspath{{./figs/}}

\newlength{\imgwidth}
\newboolean{twoColVersion}
\setboolean{twoColVersion}{false}
\newcommand{\twoCol}[2]{\ifthenelse{\boolean{twoColVersion}} {#1} {#2} }

\usepackage{lastpage}
\jmlrheading{25}{2024}{1-\pageref{LastPage}}{1/24; Revised
9/24}{12/24}{24-0029}{Zhen~Qin, Michael~B.~Wakin, and~Zhihui~Zhu}

\ShortHeadings{Guaranteed Nonconvex Factorization Approach for Tensor Train Recovery}{Qin, Wakin, and Zhu}
\firstpageno{1}

\begin{document}

\title{Guaranteed Nonconvex Factorization Approach \\ for Tensor Train Recovery}

\author{\name Zhen~Qin \email qin.660@osu.edu \\
       \addr Department of Computer Science and Engineering\\
       Ohio State University\\
       Columbus, Ohio 43201, USA
       \AND
       \name Michael~B.~Wakin \email mwakin@mines.edu \\
       \addr Department of Electrical Engineering\\
       Colorado School of Mines\\
       Golden, Colorado 80401, USA
       \AND
       \name Zhihui~Zhu \email zhu.3440@osu.edu \\
       \addr Department of Computer Science and Engineering\\
       Ohio State University\\
       Columbus, Ohio 43201, USA}

\editor{Shiqian Ma}


\maketitle

\begin{abstract}
Tensor train (TT) decomposition represents an order-$N$ tensor using $O(N)$ order-$3$ tensors (i.e., factors of small dimension), achieved through products among these factors. Due to its compact representation, TT decomposition has been widely used in the fields of signal processing, machine learning, and quantum physics. It offers benefits such as reduced memory requirements, enhanced computational efficiency, and decreased sampling complexity. Nevertheless, existing optimization algorithms with  guaranteed performance concentrate exclusively on using the TT format for reducing the optimization space in recovery problems, while still operating on the entire tensor in each iteration. There is a lack of comprehensive theoretical analysis for optimization involving the factors directly, despite the proven efficacy of such factorization methods in practice. In this paper, we provide the first convergence guarantee for the factorization approach in a TT-based recovery problem. Specifically, to avoid the scaling ambiguity and to facilitate theoretical analysis, we optimize over the so-called left-orthogonal TT format which enforces orthonormality among most of the factors.
To ensure the orthonormal structure, we utilize the Riemannian gradient descent (RGD) for optimizing those factors over the Stiefel manifold.
We first delve into the TT factorization/decomposition problem and establish the local linear convergence of RGD. Notably, the rate of convergence only experiences a linear decline as the tensor order increases. We then study the sensing problem that aims to recover a TT format tensor from linear measurements. Assuming the sensing operator satisfies the restricted isometry property (RIP), we show that with a proper initialization, which could be obtained through spectral initialization, RGD also converges to the ground-truth tensor at a linear rate. Furthermore, we expand our analysis to encompass scenarios involving Gaussian noise in the measurements. We prove that RGD can reliably recover the ground truth at a linear rate, with the recovery error exhibiting only polynomial growth in relation to the tensor order $N$. We conduct various experiments to validate our theoretical findings.

\end{abstract}


\section{Introduction}
\label{introduction}

Tensor estimation is a crucial task in various scientific and engineering fields, including signal processing and machine learning \citep{CichockiMagTensor15,SidiropoulosTSPTENSOR17}, communication \citep{SidiropoulosBlind}, chemometrics \citep{Smilde04,AcarUnsup09}, genetic engineering \citep{HoreNature16}, and so on. When dealing with an order-$N$ tensor $\calX\in\R^{d_1\times\dots\times d_N}$, its exponentially increasing size with respect to $N$ poses significant challenges in both memory and computation. To address this issue, tensor decomposition, which provides a compact representation of a tensor, has gained popularity in practical applications.  The widely used tensor decompositions include the canonical polyadic (CP) \citep{Bro97}, Tucker \citep{Tucker66}, and tensor train (TT) \citep{Oseledets11} decompositions.  These three formats have their pros and cons.
The CP decomposition offers a storage advantage as it requires the least amount of storage, scaling linearly with $N$. 
However, determining the CP rank of a tensor is generally an NP-hard problem, as are tasks such as CP decomposition \citep{haastad1989tensor, de2008tensor, kolda2009tensor,cai2019nonconvex}. On the contrary, the Tucker decomposition can be approximately computed using the higher-order singular value decomposition. However, when representing a tensor using the Tucker decomposition, the size of the core tensor still grows exponentially in terms of $N$. This leads to significant memory consumption, making the Tucker decomposition more suitable for low-order tensors than for high-order ones.

In comparison, the \emph{TT format} provides a balanced representation: in many cases it requires $O(N)$ parameters, while its quasi-optimal decomposition can be obtained through a sequential singular value decomposition (SVD) algorithm, commonly referred to as the tensor train SVD (TT-SVD) \citep{Oseledets11}. Specifically, the $(s_1,\dots,s_N)$-th element of $\calX$ in the \emph{TT format} can be expressed as the following matrix product form \citep{Oseledets11}
\begin{eqnarray}
    \label{Definition of Tensor Train in intro}
    \calX(s_1,\dots,s_N)=\mX_1(:,s_1,:)\mX_2(:,s_2,:)\cdots \mX_N(:,s_N,:),
\end{eqnarray}
where the tensor factors ${\mX}_i \in\R^{r_{i-1}\times d_i \times r_i}, i=1,\dots,N$ with $r_0=r_N=1$. See \Cref{The TT fig} for an illustration.
Thus, the TT format can be represented by $N$ order-$3$ tensor factors $\{{\mX}_i\}_{i\geq 1}$, with a total of $O(N\ol d\ol r^2)$ parameters, where $\ol d = \max_i d_i$ and $\ol r = \max_i r_i$. The dimensions $\vr = (r_1,\dots, r_{N-1})$ of such a  decomposition are called the \emph{TT ranks} of $\calX$. Any tensor can be decomposed in the TT format \eqref{Definition of Tensor Train in intro} with sufficiently large TT ranks \citep[Theorem 2.1]{Oseledets11}. Indeed, there always exists a TT decomposition with $r_i \le\min\{\Pi_{j=1}^{i}d_j, \Pi_{j=i+1}^{N}d_j\}$ for any $i\ge 1$. We say a TT format tensor is low-rank if $r_i$ is much smaller compared to $\min\{\Pi_{j=1}^{i}d_j, \Pi_{j=i+1}^{N}d_j\}$ for most indices\footnote{When $i = 1$ or $N-1$, $r_1$ or $r_{N-1}$ may not be much smaller than $d_1$ or $d_N$.} $i$ so that the total number of parameters in the tensor factors $\{\mX_i\}$ is much smaller than the number of entries in $\calX$. We refer to any tensor for which such a low-rank TT decomposition exists as a {\it low-TT-rank} tensor.

\begin{figure}[t]
\centering
\includegraphics[width=12cm, keepaspectratio]%
{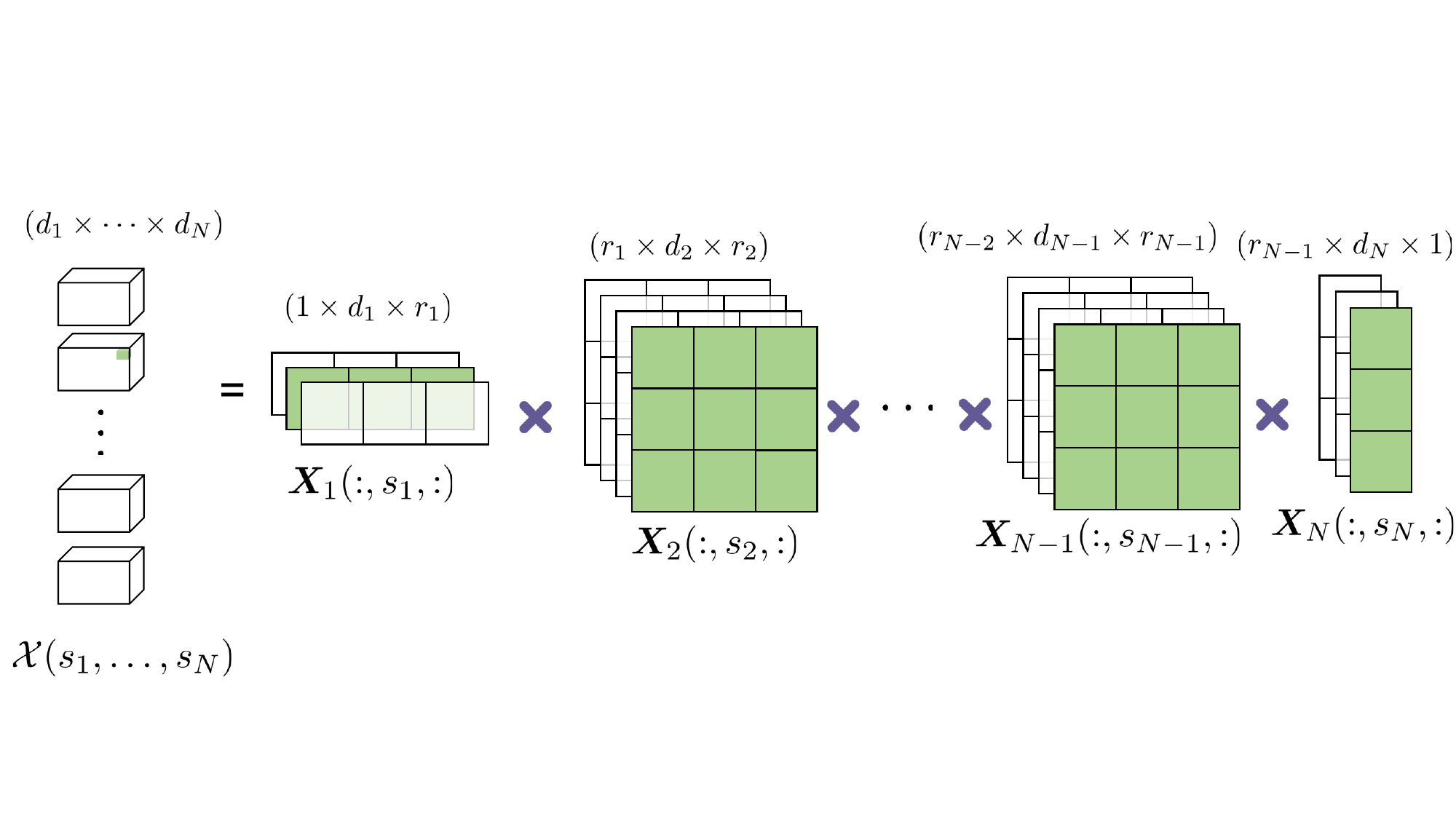}
\vspace{-1.2cm}
\caption{Illustration of the TT format \eqref{Definition of Tensor Train in intro}.}
\label{The TT fig}
\end{figure}

Due to its compact representation, the TT decomposition with small TT ranks has found extensive applications in various fields. For instance, it has been widely used for image compression \citep{latorre2005image, bengua2017efficient}, analyzing theoretical properties of deep networks \citep{khrulkov2017expressive}, network compression or tensor networks \citep{stoudenmire2016supervised, novikov2015tensorizing, yang2017tensor, tjandra2017compressing, yu2017long, ma2019tensorized}, recommendation systems \citep{frolov2017tensor}, probabilistic model estimation \citep{novikov2021tensor}, learning of Hidden Markov Models \citep{kuznetsov2019tensor}, and more.
Notably, as equivalents to the TT decomposition, the matrix product state (MPS) and matrix product operator (MPO) decompositions have been introduced in the quantum physics community for efficiently and concisely representing quantum states.  In this context, the parameter $N$ represents the number of qubits in the many-body system \citep{verstraete2006matrix, verstraete2008matrix, schollwock2011density}. The concise representation provided by MPS and MPO is particularly valuable in quantum state tomography, as it allows us to observe a quantum state using computational and experimental resources that grow polynomially rather than exponentially with the number of qubits $N$ \citep{ohliger2013efficient}.

A fundamental challenge in many of the aforementioned applications is to construct a  low-TT-rank tensor from highly incomplete measurements of that tensor. 
The work \citep{bengua2017efficient,wang2019latent} extends a nuclear-norm based convex relaxation approach from the matrix case to the TT case, but its high computational complexity makes it impractical for higher-order tensors. Alternating minimization \citep{wang2016tensor} and gradient descent \citep{yuan2019high} have been employed to efficiently estimate the factors in the TT format, but theoretical guarantees regarding recovery error or convergence properties are not provided. Besides these heuristic algorithms,  iterative hard thresholding (IHT) \citep{Rauhut17,rauhut2015tensor} and Riemannian gradient descent on the TT manifold \citep{budzinskiy2021tensor,wang2019tensor, Cai2022provable} have been proposed with local convergence guarantees. However, both methods necessitate the estimation of the entire tensor $\calX$ in each iteration, which poses a challenge due to its exponential size in terms of $N$. As a result, both methods demand an exponential amount of storage or memory. 
In addition, theoretical results of IHT hinge on an unverified perturbation bound for TT-SVD.
The Riemannian gradient descent method relies on the curvature information at the target tensor, which is often unknown a priori.

Instead of optimizing over the tensor $\calX$ directly, in this paper, we focus on optimizing over the factors $\{\mX_i\}_{i\geq 1}$ in the TT format. This factorization approach can significantly reduce the memory cost and has found widespread applications. For instance, gradient descent-based optimization on the TT factors has been successfully applied in various areas, including the TT deep computation model \citep{zhang2018tensor}, TT deep neural networks \citep{qi2022exploiting}, TT completion \citep{yuan2019tensor}, channel estimation \citep{zhang2021designing} and quantum tomography \citep{lidiak2022quantum}. However, to the best of our knowledge, there is a lack of rigorous convergence analysis for the TT factorization approaches.

\paragraph{Challenges:} One of the main challenges in studying the convergence analysis of iterative algorithms for the factorization approach lies in the form of products among multiple matrices in \eqref{Definition of Tensor Train in intro}. For instance, the TT factorization is not unique, and there exist infinitely many equivalent factorizations. In particular, for any factorization $\{{\mX}_1,\dots,{\mX}_N\}$,  $\{{\mX}_1 \mP_1, \mP_1^{-1}\mX_2 \mP_2,\ldots,\mP_{N-1}^{-1}\mX_N\}$ is also a TT factorization of $\calX$ for any invertible matrices $\mP_i\in\R^{r_i\times r_i}, i\in[N-1]$, where $[N-1] = \{1,\ldots,N-1\}$ and $\mP_{i-1}^{-1}\mX_i \mP_i$ refers to $\mP_{i-1}^{-1}\mX_i(:,s_i,:) \mP_i$ for all $s_i\in[d_i]$. This implies that the factors could be unbalanced (e.g., $\mP_i = t\mId$ with either very large or small $t$), which makes the convergence analysis difficult.  In matrix factorization, a regularizer is often used to address this scaling ambiguity (i.e., to reduce the search space of factors) and balance the energy of the two factors \citep{Tu16,park2017non,Zhu18TSP}. Motivated by these results, one may adopt the same trick by adding regularizers to balance any pair of consecutive factors $(\mX_i,\mX_{i+1})$. However,  this strategy could be intricate, given that modifying $\mX_{i}$ to achieve a balance between the pair $(\mX_i,\mX_{i+1})$ will similarly impact another pair, namely, $(\mX_{i-1},\mX_{i})$.

\paragraph{Our contributions:}
In this paper, we study the factorization approach for the TT sensing problem, where the goal is to recover the underlying  low-TT-rank tensor $\calX^\star$ through its linear measurements $\vy = \calA(\calX^\star)$, where the linear mapping $\calA:\R^{d_1\times\cdots\times d_N}\rightarrow \R^m$ denotes the sensing operator. To address the  ambiguity issue in the factorization approach,  we consider the so-called left-canonical TT format that restricts all of the factors except the last one to be orthonormal, i.e., $\sum_{s_i = 1}^{d_i}\mX_i^\top(:,s_i,:)\mX_i(:,s_i,:) = \mId_{r_i}, i \in [N-1]$. Further details on the left-canonical form are described in \Cref{GD algorithm in TF}.
For a collection of factors $\{\mX_i\}$, to simplify the notation, we will denote by $[\mX_1,\dots, \mX_N]$ the corresponding TT format tensor $\calX$ with entries expressed in \eqref{Definition of Tensor Train in intro}. We then attempt to recover the underlying  low-TT-rank tensor by solving the following TT factorized optimization problem
\begin{eqnarray}
    \label{Loss Function of general tensor sensing intro}
    \begin{split}
    \min_{\mbox{\tiny$\begin{array}{c}
     {\mX_i}\in\R^{r_{i-1}\times d_i \times r_i},\\
     i\in [N]\end{array}$}} &\frac{1}{2m}\|\calA([\mX_1,\dots, \mX_N]) - \vy\|_2^2,\\
     &\st \ \sum_{s_i = 1}^{d_i}\mX_i^\top(:,s_i,:)\mX_i(:,s_i,:) = \mId_{r_i}, i \in [N-1].
    \end{split}
\end{eqnarray}
Noting that each constraint defines a Stiefel manifold, to guarantee the exact preservation of the orthonormal structure in each iteration, we propose a (hybrid) Riemannian gradient descent (RGD) algorithm to solve the above TT factorization problem. Our main contribution focuses on the convergence analysis of RGD for solving this problem.
\begin{itemize}
\item We first study the TT factorization problem where $\calA$ is an identity operator. With an appropriate distance metric on the factors, we establish the local linear convergence of RGD. Notably, the accuracy requirement on the initialization only depends polynomially on the tensor order $N$ and the rate of convergence only experiences a linear decline
as $N$ increases. This demonstrates potential advantages over  introducing additional regularizers to enforce orthogonality for each factor, as used in \citep{Han20} for the Tucker factorization, which only ensures approximate orthogonality in each iteration of gradient descent and thus is likely to suffer from exponential dependence on the tensor order $N$.
\item  We then extend the convergence analysis to the more general TT sensing problem. Under the assumption that the sensing operator satisfies the restricted isometry property (RIP)---a condition that can be satisfied with $m\gtrsim N\ol d\ol r^2\log(N\ol r)$  generic subgaussian measurements \citep{Rauhut17,qin2024quantum} (where, again, $\ol r=\max_ir_i$ and $\ol d=\max_i d_i$)---we show that RGD, given an appropriate initialization, converges to the ground-truth tensor at a linear rate. Additionally, spectral initialization provides a valid starting point for ensuring the convergence of RGD. Furthermore, we expand our analysis to noisy measurements and prove that RGD can reliably recover the ground truth at a linear rate up to an error proportional to the noise level and exhibiting only polynomial growth in the tensor order $N$.
\end{itemize}

\paragraph{Paper organization} The rest of this paper is organized as follows.
In Section~\ref{Tensor Train Decomposition}, we introduce the basic definitions of the TT format.
Section~\ref{GD algorithm in TF} and Section~\ref{LR TT sensing all} analyze the local convergence of RGD for the TT factorization and sensing problems, respectively.
Section~\ref{Numerical experiments} presents numerical experiments. Lastly, we  conclude the paper in Section~\ref{conclusion}.

\paragraph{Notations} We use calligraphic letters (e.g., $\calY$) to denote tensors,  bold capital letters (e.g., $\mY$) to denote matrices, except for $\mX_i$ which denotes the $i$-th order-$3$ tensor factor in the TT format,  bold lowercase letters (e.g., $\vy$) to denote vectors, and italic letters (e.g., $y$) to denote scalar quantities. $\|\mX\|$ and $\|\mX\|_F$ respectively represent the spectral norm and Frobenius norm of the matrix $\mX$, while
$\sigma_{i}(\mX)$ is the $i$-th singular value of $\mX$.
$\|\vx\|_2$ denotes the $l_2$ norm of the vector $\vx$.  Elements of matrices and tensors are denoted in parentheses, as in Matlab notation. For example, $\calX(s_1, s_2, s_3)$ denotes the element in position
$(s_1, s_2, s_3)$ of the order-3 tensor $\calX$.
The inner product of $\calA, \calB\in\R^{d_1\times\dots\times d_N}$ can be denoted as $\<\calA, \calB \> = \sum_{s_1=1}^{d_1}\cdots \sum_{s_N=1}^{d_N} \calA(s_1,\dots,s_N)\calB(s_1,\dots,s_N) $.
The vectorization of  $\calX\in\R^{d_1\times\dots\times d_N}$, denoted as $\text{vec}(\calX)$, transforms the tensor $\calX$ into a vector. The $(s_1, \dots, s_N)$-th element of $\calX$ can be found in the vector $\text{vec}(\calX)$ at the position $s_1 + d_1(s_2-1) + \cdots + d_1d_2 \cdots d_{N-1}(s_N-1)$.
$\|\calX\|_F = \sqrt{\<\calX, \calX \>}$ is the Frobenius norm of $\calX$.
 $\ol \otimes$ denotes the Kronecker product between submatrices in two block matrices. Its detailed definition and properties are shown in {Appendix} \ref{Technical tools used in proofs}.
For a positive integer $K$, $[K]$ denotes the set $\{1,\dots, K \}$. For two positive quantities $a,b\in \real$, the inequality $b\lesssim a$ or $b = O(a)$ means $b\leq c a$ for some universal constant $c$; likewise, $b\gtrsim a$ or $b = \Omega(a)$ indicates that $b\ge ca$ for some universal constant $c$.

\section{Preliminaries of Tensor Train Decomposition and Stiefel Manifold}
\label{Tensor Train Decomposition}

\subsection{Tensor Train Decomposition}
\label{sec:TT}

Recall the TT format of $\calX$ in \eqref{Definition of Tensor Train in intro}. Since $\mX_{i}(:,s_i,:)$ will be extensively used, we will denote it by  ${\mX_i(s_i)}\in\R^{r_{i-1}\times r_{i}}$; this matrix comprises one ``slice'' of ${\mX}_i$ with the second index being fixed at $s_i$. The $(s_1,\dots,s_N)$-th element in $\calX$ can then be written as $\calX(s_1,\dots,s_N)=\prod_{i=1}^N{\mX_i(s_i)}$.

In addition, for any two TT format tensors $\wt\calX, \wh\calX\in\R^{d_1\times\dots\times d_N}$ with factors $\{\wt \mX_i(s_i) \in \R^{\wt r_{i-1}\times \wt r_i} \}$ and $\{ \wh \mX_i(s_i) \in \R^{\wh r_{i-1}\times \wh r_i} \}$, each element of the summation $\calX =  \wt\calX + \wh\calX$ can be represented by
\begin{align}
    \label{summation of TT format}
    &\hspace{-0.25cm}\calX(s_1,\dots,s_N) = \begin{bmatrix}\wt\mX_1(s_1) \!\!\!\! & \wh\mX_1(s_1) \end{bmatrix}
\begin{bmatrix}\wt\mX_2(s_2) \!\!\!\! & {\bm 0} \\ {\bm 0} \!\!\!\! & \wh\mX_2(s_2) \end{bmatrix}
 \cdots \begin{bmatrix}\wt\mX_{N-1}(s_{N-1}) \!\!\!\! & {\bm 0} \\ {\bm 0} \!\!\!\! & \wh\mX_{N-1}(s_{N-1}) \end{bmatrix} \begin{bmatrix}\wt\mX_N(s_N) \\ \wh\mX_N(s_N) \end{bmatrix},
\end{align}
which implies that $\calX$ can also be represented in the TT format with ranks  $r_i\leq \wt r_i + \wh r_i$ for $i = 1, \dots, N-1$.

\paragraph{Canonical form}

The decomposition of a tensor $\calX$ into the form \eqref{Definition of Tensor Train in intro} is generally not unique: not only are the factors ${\mX_i(s_i)}$ not unique, but also the dimensions of these factors can vary. To introduce the factorization with the smallest possible dimensions $\vr = (r_1,\ldots,r_{N-1})$, for convenience, for each $i$, we put $\{ {\mX_i(s_i)}\}_{s_i=1}^{d_i}$ together into the following two forms
$$L(\mX_i)=\begin{bmatrix}\mX_i(1) \\ \vdots\\  \mX_i(d_i) \end{bmatrix}\in\R^{(r_{i-1}d_i) \times r_i},$$ $$R(\mX_i)=\begin{bmatrix}\mX_i(1) &  \cdots &  \mX_i(d_i) \end{bmatrix}\in\R^{r_{i-1}\times (d_ir_i)},$$
where $L(\mX_i)$ and $R(\mX_i)$ are often called the left and right unfoldings of $\mX_i$, respectively, if we view $\mX_i$ as a tensor. We say the decomposition \eqref{Definition of Tensor Train in intro} is \emph{minimal} if the rank of the left unfolding matrix $L(\mX_i)$ is $r_i$ and the rank of the right unfolding matrix $R(\mX_i)$ is $r_{i-1}$ for all $i$. The dimensions $\vr = (r_1,\dots, r_{N-1})$ of such a minimal decomposition are called the \emph{TT ranks} of $\calX$. To simplify the notation and presentation, we may also refer to $\ol r= \max_i r_i$ as the \emph{TT rank}. According to \citep{holtz2012manifolds}, there is exactly one set of ranks $\vr$ that $\calX$ admits a minimal TT decomposition.

Under the minimal decomposition, there always exists a factorization such that $L(\mX_i)$ are orthonormal matrices for all $i \in [N-1]$:
\begin{eqnarray}
\label{orthogonal property of left orthogonal}
    L^\top(\mX_i)L(\mX_i) = \mId_{r_i}, \ \forall i=1,\dots,N-1.
\end{eqnarray}
Such a decomposition is unique up to the insertion of orthonormal matrices between adjacent factors \citep[Theorem 1]{holtz2012manifolds}. That is, $\Pi_{i=1}^N\mX_i(s_i) = \Pi_{i=1}^N\mR_{i-1}^\top\mX_i(s_i)\mR_i$ for any orthonormal matrix $\mR_i\in\O^{r_i\times r_i}$ (with $\mR_0 = \mR_N = 1$).
The resulting TT factors $\{\mX_i\}$ or the TT decomposition is called the {\it left-orthogonal form}, or {\it left-canonical form}. Similarly, the decomposition is said to be right-orthogonal if $R(\mX_i)$ satisfies $
    R(\mX_i)R^\top(\mX_i) = \mId_{r_{i-1}}, \ \forall i=2,\dots,N.$
Since the two forms are equivalent~\citep{holtz2012manifolds}, in this paper, we always focus on the left-orthogonal form. Unless otherwise specified, we will always assume that the factors are in left-orthogonal form.

Moreover, $r_i$  also relates to the rank of the $i$-th unfolding matrix\footnote{We can also define the $i$-th unfolding matrix as $\calX^{\< i \>} = \mX^{\leq i}\mX^{\geq i+1}$, where each row of the left part $\mX^{\leq i}$ and each column of the right part $\mX^{\geq i+1}$ can be represented as
\begin{eqnarray*}
    \mX^{\leq i}(s_1\cdots s_i,:) = \mX_1(s_1)\cdots \mX_i(s_i),
\end{eqnarray*}
\begin{eqnarray*}
    \mX^{\geq i+1}(:,s_{i+1}\cdots s_N) = \mX_{i+1}(s_{i+1})\cdots \mX_N(s_N).
\end{eqnarray*}
Note that when the factors are in the left-orthogonal form, we have $(\mX^{\leq i})^\top\mX^{\leq i} = \mId_{r_i}$. Similarly, for the right-orthogonal form, $ \mX^{\geq i+1} (\mX^{\geq i+1})^\top = \mId_{r_i}$.} $\calX^{\<i\>}\in\R^{(d_1\cdots d_i)\times (d_{i+1}\cdots d_N)}$ of the tensor $\calX$, where the $(s_1\cdots s_i, s_{i+1}\cdots s_N)$-th element\footnote{ Specifically, $s_1\cdots s_i$ and $s_{i+1}\cdots s_N$ respectively represent the $(s_1+d_1(s_2-1)+\cdots+d_1\cdots d_{i-1}(s_i-1))$-th row and $(s_{i+1}+d_{i+1}(s_{i+2}-1)+\cdots+d_{i+1}\cdots d_{N-1}(s_N-1))$-th column.
} of $\calX^{\<i\>}$ is given by $$\calX^{\<i\>}(s_1\cdots s_i, s_{i+1}\cdots s_N) = \calX(s_1,\dots, s_N).$$
This can also serve as an alternative way to define the TT ranks. With the $i$-th unfolding matrix $\calX^{\<i\>}$ and the TT ranks, we can obtain its smallest singular value $\underline{\sigma}(\calX)=\min_{i=1}^{N-1}\sigma_{r_i}(\calX^{\<i\>})$, its largest singular value $\overline{\sigma}(\calX)=\max_{i=1}^{N-1}\sigma_{1}(\calX^{\<i\>})$ and its condition number $\kappa(\calX)=\frac{\overline{\sigma}(\calX)}{\underline{\sigma}(\calX)}$.

\paragraph{Distance between factors}
We now introduce an appropriate metric to quantify the distinctions between the left-orthogonal form factors $\{\mX_i\}$ and $\{\mX_i^\star\}$ of two TT format tensors  $\calX = [{\mX}_1,\dots,{\mX}_N ]$ and  $\calX^\star = [{\mX}_1^\star,\dots,{\mX}_N^\star ]$.

To capture this rotation ambiguity, by defining the rotated factors $L_{\mR}(\mX_i^\star)$ as
\begin{eqnarray}
\label{A modified left unfolding}
    L_{\mR}(\mX_i^\star) = \begin{bmatrix}\mR_{i-1}^\top\mX_{i}^\star(1)\mR_i\\ \vdots \\ \mR_{i-1}^\top\mX_{i}^\star(d_i)\mR_i\end{bmatrix}, 
\end{eqnarray}
we then define the distance between the two sets of factors as
\begin{eqnarray}
\label{BALANCED NEW DISTANCE BETWEEN TWO TENSORS}
    \text{dist}^2(\{\mX_i \},\{\mX_i^\star \})=\!\!\!\!\min_{\mR_i\in\O^{r_i\times r_i}, \atop i \in [N-1]}\sum_{i=1}^{N-1} \ol{\sigma}^2(\calX^\star)\|L({\mX}_i)-L_{\mR}({\mX}_i^\star)\|_F^2 + \|L({\mX}_N)-L_{\mR}({\mX}_N^\star)\|_2^2,
\end{eqnarray}
where we note that $L({\mX}_N), L_{\mR}({\mX}_N^\star)\in\R^{(r_{N-1}d_N)\times 1}$ are vectors.
Here, the coefficients $\ol{\sigma}^2(\calX^\star)$ and $1$ are incorporated to harmonize the energy between $\{L_{\mR}({\mX}_i^\star)  \}_{i\leq N-1}$ and $L_{\mR}({\mX}_N^\star)$ since $\|L_{\mR}({\mX}_i^\star)\|^2 = 1, i\in[N-1]$ and $\|L({\mX}_N)-L_{\mR}({\mX}_N^\star)\|_2^2 = \|R({\mX}_N)- \mR_{N-1}^\top R({\mX}_N^\star)\|_F^2$, $\|\mR_{N-1}^\top R({\mX}_N^\star)\|^2 = \|R({\mX}_N^\star)\|^2 = \sigma_1^2({\calX^\star}^{\< N-1 \>})\leq \ol{\sigma}^2(\calX^\star)$. The following result establishes a connection between $\text{dist}^2(\{\mX_i \},\{ \mX_i^\star \})$ and $\|\calX-\calX^\star\|_F^2$.
\begin{lemma}
\label{LOWER BOUND OF TWO DISTANCES main paper}
For any two TT format tensors $\calX$ and $\calX^\star $ with ranks $\vr = (r_1,\dots, r_{N-1})$, let $\{\mX_i\}$ and $\{\mX_i^\star\}$ be the corresponding left-orthogonal form factors. Assume $\ol{\sigma}^2(\calX)\leq \frac{9\ol{\sigma}^2(\calX^\star)}{4}$. Then $\|\calX-\calX^\star\|_F^2$ and $\dist^2(\{\mX_i \},\{ \mX_i^\star \})$ defined in \eqref{BALANCED NEW DISTANCE BETWEEN TWO TENSORS} satisfy
\begin{eqnarray}
    \label{LOWER BOUND OF TWO DISTANCES_1 main paper}
    &&\|\calX-\calX^\star\|_F^2\geq\frac{1}{8(N+1+\sum_{i=2}^{N-1}r_i)\kappa^2(\calX^\star)}\dist^2(\{\mX_i \},\{ \mX_i^\star \}),\\
    \label{UPPER BOUND OF TWO DISTANCES_1 main paper}
    &&\|\calX-\calX^\star\|_F^2\leq\frac{9N}{4}\dist^2(\{\mX_i \},\{ \mX_i^\star \}).
\end{eqnarray}
\end{lemma}
The proof is given in {Appendix} \ref{Technical tools used in proofs}. {Lemma} \ref{LOWER BOUND OF TWO DISTANCES main paper} ensures that $\calX$ is close to $\calX^\star$ once the corresponding factors are close with respect to the proposed distance measure, and the convergence behavior of $\|\calX-\calX^\star\|_F^2$ is reflected by the convergence in terms of the factors.
In the next sections, we will study the convergence with respect to the factors.

\subsection{Stiefel Manifold}
Since we will focus on the left-canonical form where the left unfolding matrices of a TT factorization are orthonormal, i.e., reside on the Stiefel manifold,  we will introduce several essential definitions concerning the Stiefel manifold and its tangent space to clarify our discussion of optimization on the Stiefel manifold.
The Stiefel manifold $\text{St}(m,n)=\{\mY\in\R^{m\times n}: \mY^\top\mY=\mId_{n}\}$ is a
Riemannian manifold that is composed of all $m \times n$ orthonormal matrices. We can regard $\text{St}(m,n)$ as an embedded submanifold of a Euclidean space and further define $\text{T}_{\mY} \text{St}:=\{\mA\in\R^{m\times n}: \mA^\top\mY+\mY^\top\mA={\bm 0} \}$  as the tangent space to the Stiefel manifold $\text{St}(m,n)$ at the point $\mY \in \text{St}(m,n)$.
For any $\mB\in\R^{m\times n}$, the projection of $\mB$ onto $\text{T}_{\mY} \text{St}$ is given by \citep{absil2008optimization}
\begin{eqnarray}
    \label{Projection to the Tangent space on the Stiefel manifold}
    \calP_{\text{T}_{\mY} \text{St}}(\mB)=\mB-\frac{1}{2}\mY\bigg(\mB^\top \mY+\mY^\top\mB\bigg),
\end{eqnarray}
and the projection of $\mB$ onto the orthogonal complement of $\text{T}_{\mY} \text{St}$ is given by
\begin{eqnarray}
    \label{orthogonal complement to the Tangent space on the Stiefel manifold}
    \calP_{\text{T}_{\mY} \text{St}}^{\perp}(\mB)=\mB   - \calP_{\text{T}_{\mY} \text{St}}(\mB)  =  \frac{1}{2}\mY\bigg(\mB^\top\mY+\mY^\top\mB\bigg).
\end{eqnarray}
Note that when we have a gradient $\mB$, we can use the projection operator \eqref{Projection to the Tangent space on the Stiefel manifold} to compute the Riemannian gradient $\calP_{\text{T}_{\mY} \text{St}}(\mB)$ on the tangent space of the Stiefel manifold. Riemannian gradient descent involves the update $\wh \mY = \mY - \mu \calP_{\text{T}_{\mY} \text{St}}(\mB)$ with a step size $\mu>0$, which is then projected back to the Stiefel manifold, such as via the polar decomposition-based retraction, i.e.,
\begin{eqnarray}
    \label{polar decomposition-based retraction}
    \text{Retr}_{\mY}(\wh \mY)=\wh \mY(\wh \mY^\top\wh \mY)^{-\frac{1}{2}}.
\end{eqnarray}

\section{Warm-up: Low-rank Tensor-train Factorization}
\label{GD algorithm in TF}
To provide a baseline for the convergence of iterative algorithms for the TT recovery problem with the factorization approach in \eqref{Loss Function of general tensor sensing intro}, we first study the following TT factorization problem
\begin{eqnarray}
    \label{RIEMANNIAN_LOSS_FUNCTION_1}
    \begin{split}
\min_{\mbox{\tiny$\begin{array}{c}
     {\mX_i}\in\R^{r_{i-1}\times d_i \times r_i}\\
     i\in [N]\end{array}$}} & f({\mX}_1,\dots, {\mX}_N) =  \frac{1}{2}\left\|[\mX_1,\dots, \mX_N]  - \calX^\star \right\|_F^2,\\
    &\st \ L^\top({\mX}_i)L({\mX}_i)=\mId_{r_i},i \in [N-1],
    \end{split}
\end{eqnarray}
Except for the scaling difference in the object function, the above problem is a special case of \eqref{Loss Function of general tensor sensing intro}, where the operator $\calA$ is the identity operator, and thus the convergence analysis for \eqref{RIEMANNIAN_LOSS_FUNCTION_1} will provide useful insight for the problem \eqref{Loss Function of general tensor sensing intro}. We will analyze the local convergence of Riemannian gradient descent (RGD) to solve the  factorization problem \eqref{RIEMANNIAN_LOSS_FUNCTION_1} and explore how the convergence speed and requirements depend on the properties of $\calX^\star$, such as the tensor order. We will then extend the analysis to the  sensing problem \eqref{Loss Function of general tensor sensing intro} in the next section.

Specifically, we utilize the following (hybrid) RGD
\begin{eqnarray}
    \label{RIEMANNIAN_GRADIENT_DESCENT_1_1}
    &&\hspace{-2.5cm}L(\mX_i^{(t+1)})=\text{Retr}_{L(\mX_i)}\big(L(\mX_i^{(t)})-\frac{\mu}{\ol{\sigma}^2(\calX^\star)}\calP_{\text{T}_{L({\mX_i})} \text{St}}\big(\nabla_{L({\mX}_{i})}f(\mX_1^{(t)}, \dots, \mX_N^{(t)})\big) \big),  i\in [N-1],  \\
    \label{RIEMANNIAN_GRADIENT_DESCENT_1_2}
    &&\hspace{-2.5cm}L(\mX_N^{(t+1)})=L(\mX_N^{(t)})-\mu\nabla_{L({\mX}_{N})}f(\mX_1^{(t)}, \dots, \mX_N^{(t)}),
\end{eqnarray}
where $\calP_{\text{T}_{L({\mX_i})} \text{St}}$ denotes the projection onto the tangent space of the Stiefel manifold at the point $L({\mX_i})$, as defined in \eqref{Projection to the Tangent space on the Stiefel manifold}, such that $\calP_{\text{T}_{L({\mX_i})} \text{St}}\big(\nabla_{L({\mX}_{i})}f(\mX_1^{(t)}, \dots, \mX_N^{(t)})\big)$ is the Riemannian gradient of the objective function $f$ with respect to the $i$-th factors $L(\mX_i)$. The detailed expression of the gradients  $\nabla_{L({\mX}_{i})}f(\mX_1^{(t)}, \dots, \mX_N^{(t)})$ is presented in {Appendix} \ref{Local Convergence Proof of Riemannman gradient descent} (see \eqref{RIEMANNIAN_GRADIENT_1-1}). We update the factor $\mX_N$ using gradient descent in \eqref{RIEMANNIAN_GRADIENT_DESCENT_1_2} since there is no constraint on this factor. For simplicity, we still refer to the updates in \eqref{RIEMANNIAN_GRADIENT_DESCENT_1_1} and \eqref{RIEMANNIAN_GRADIENT_DESCENT_1_2} as RGD.  Note that we use discrepant step sizes between $L(\mX_i)$ and $L(\mX_N)$ in the proposed RGD algorithm in order to balance the convergence of those factors as they have different energies; $\|L(\mX_i)\|^2 = 1, i\in[N-1]$ and $\|R(\mX_N)\|^2 = \sigma_1^2(\calX^{\< N-1 \>})\leq \ol{\sigma}^2(\calX)$  in each iteration. To simplify the analysis, we employ a discrepant learning rate ratio, i.e., $\ol{\sigma}^2(\calX^\star)$, to balance the two sets of factors.
We note that other choices of discrepant learning rate ratio are also effective in practice, such as $\|\calX^{(t)}\|_F^2$, which can be efficiently computed for TT-format tensors \citep{Oseledets11}. In addition, when the ground-truth tensor $\calX^\star$ is unknown a prior in practice, we can instead use the information of either $\calX^{(0)}$ or $\calX^{(t)}$ to balance the learning rate, given that the iterates will remain close to the target tensor as guaranteed by the following analysis.

\paragraph{Local convergence of RGD algorithm} According to \Cref{Tensor Train Decomposition}, each TT format tensor has an equivalent left-orthogonal form. Let $\{\mX_i^\star\}$ be the left-orthogonal form factors for a minimal TT decomposition of $\calX^\star$. These factors are utilized exclusively for the analysis and are not required for the algorithm. Recall the distance between the two set of factors $\{\mX_i\}$ and $\{\mX_i^\star\}$ as given in \eqref{BALANCED NEW DISTANCE BETWEEN TWO TENSORS}.

We now establish a local linear convergence guarantee for the RGD algorithm in \eqref{RIEMANNIAN_GRADIENT_DESCENT_1_1} and \eqref{RIEMANNIAN_GRADIENT_DESCENT_1_2}.
\begin{theorem}
\label{Local Convergence of Stiefel_Theorem}
Consider a low-TT-rank tensor $\calX^\star$ with ranks $\vr = (r_1,\dots, r_{N-1})$.
Suppose that the RGD in \eqref{RIEMANNIAN_GRADIENT_DESCENT_1_1} and \eqref{RIEMANNIAN_GRADIENT_DESCENT_1_2} is initialized with $\{\mX_i^{(0)} \}$ satisfying
\begin{eqnarray}
    \label{Local Convergence of Stiefel_Theorem initialization}
    \dist^2(\{\mX_i^{(0)} \},\{ \mX_i^\star \})\leq \frac{\underline{\sigma}^2(\calX^\star)}{72(N^2-1)(N+1+\sum_{i=2}^{N-1}r_i)},
\end{eqnarray}
and the step size $\mu\leq\frac{1}{9N-5}$. Then, the iterates $\{\mX_i^{(t)} \}_{t\geq 0}$ generated by RGD will converge linearly to $\{\mX_i^\star \}$ (up to rotation):
\begin{eqnarray}
    \label{Local Convergence of Stiefel_Theorem_1}
    \dist^2(\{\mX_i^{(t+1)} \},\{ \mX_i^\star \})\leq\bigg(1-\frac{1}{64(N+1+\sum_{i=2}^{N-1}r_i)\kappa^2(\calX^\star)}\mu\bigg)\dist^2(\{\mX_i^{(t)} \},\{ \mX_i^\star \}).
\end{eqnarray}
\end{theorem}
Note that \Cref{Local Convergence of Stiefel_Theorem} only establishes local convergence. Since our primary objective is to gain insight into local convergence, and initialization is not our focus, we will omit discussions related to obtaining a valid initialization for this factorization problem. When we address the TT sensing problem in the next section, we will present approaches for finding a suitable initialization. Due to the presence of non-global critical points, linear convergence for first-order methods is likely to be attainable only within a certain region. Moreover, products of multiple (more than two) matrices often lead to the emergence of high-order saddle points or even spurious local minima that are distant from the global minima  \citep{vidal2022optimization}. Relying only on the gradient information is not sufficient to circumvent these high-order saddle points. Therefore, this paper primarily focuses on local convergence.

Remarkably, both terms $O(\frac{\underline{\sigma}^2(\calX^\star)}{N^3\ol r})$ and $O(\frac{1}{ N^2\ol r\kappa^2(\calX^\star)})$ in the initialization requirement \eqref{Local Convergence of Stiefel_Theorem initialization} and convergence rate \eqref{Local Convergence of Stiefel_Theorem_1} only decay polynomially rather than exponentially in terms of the tensor order $N$. The detailed proof of \Cref{Local Convergence of Stiefel_Theorem} is provided in {Appendix} \ref{Local Convergence Proof of Riemannman gradient descent}. Below, we provide a high-level overview of the proof.

\paragraph{Proof sketch} We focus on establishing an error contraction inequality that characterizes the error $\text{dist}^2(\{\mX_i^{(t+1)} \},\{ \mX_i^\star \})$ based on the previous iterate. Utilizing the error metric defined in \eqref{BALANCED NEW DISTANCE BETWEEN TWO TENSORS}, we  define the best rotation matrices to align $\{ \mX_i^{(t)}\}$ and $\{ \mX_i^\star \}$ as
\begin{eqnarray}
    \label{the definition of orthonormal matrix R}
    (\mR_1^{(t)},\dots,\mR_{N-1}^{(t)}) \!= \!\! \argmin_{\mR_i\in\O^{r_i\times r_i}, \atop i \in [N-1]}\!\!\sum_{i=1}^{N-1} \ol {\sigma}^2(\calX^\star)\|L({\mX}_i^{(t)})-L_{\mR}({\mX}_i^\star)\|_F^2 + \|L(\mX_N^{(t)})-L_{\mR}({\mX}_N^\star)\|_2^2,
\end{eqnarray}
where $L_{\mR}({\mX}_i^\star)$ is defined in \eqref{A modified left unfolding}. We now expand $\text{dist}^2(\{\mX_i^{(t+1)} \},\{ \mX_i^\star \})$ as
\begin{eqnarray}
    \label{expansion of distance in tensor factorization-main paper}
&\!\!\!\!\!\!\!\!&\text{dist}^2(\{\mX_i^{(t+1)} \},\{ \mX_i^\star \})\nonumber\\
  &\!\!\!\! = \!\!\!\!& \sum_{i=1}^{N} \gamma_i\bigg\| L(\mX_i^{(t+1)})-L_{\mR^{(t+1)}}(\mX_i^\star) \bigg\|_F^2 \leq\sum_{i=1}^{N} \gamma_i\bigg\| L(\mX_i^{(t+1)})-L_{\mR^{(t)}}(\mX_i^\star) \bigg\|_F^2\nonumber\\
  &\!\!\!\!\leq\!\!\!\!&\sum_{i=1}^{N} \gamma_i\bigg\| L(\mX_i^{(t)})-L_{\mR^{(t)}}(\mX_i^\star) -\frac{\mu}{\gamma_i}\calP_{\text{T}_{L({\mX}_i)} \text{St}}\bigg(\nabla_{L({\mX}_{i})}f(\mX_1^{(t)}, \dots, \mX_N^{(t)})\bigg)\bigg\|_F^2\nonumber\\
    &\!\!\!\!=\!\!\!\!&\text{dist}^2(\{\mX_i^{(t)} \},\{ \mX_i^\star \})+\sum_{i=1}^{N}\frac{\mu^2}{\gamma_i}\bigg\|\calP_{\text{T}_{L({\mX}_i)} \text{St}}\bigg(\nabla_{L({\mX}_{i})}f(\mX_1^{(t)}, \dots, \mX_N^{(t)})\bigg)\bigg\|_F^2 \nonumber\\
    &\!\!\!\!\!\!\!\!&-2\mu\sum_{i=1}^{N} \bigg\< L(\mX_i^{(t)})-L_{\mR^{(t)}}(\mX_i^\star),\calP_{\text{T}_{L({\mX}_i)} \text{St}}\bigg(\nabla_{L({\mX}_{i})}f(\mX_1^{(t)}, \dots, \mX_N^{(t)})\bigg)\bigg\>,
\end{eqnarray}
where to simplify the expression, we define $\gamma_i = \begin{cases}
    \ol {\sigma}^2(\calX^\star), &  i\in[N-1]\\
    1, &  i=N
\end{cases}$ and a projection operator for the last factor as $\calP_{\text{T}_{L({\mX}_N)} \text{St}}=\calI$ such that $\calP_{\text{T}_{L({\mX}_N)} \text{St}}(\nabla_{L({\mX}_{N})}f(\mX_1^{(t)}, \dots, \mX_N^{(t)})) = \nabla_{L({\mX}_{N})}f(\mX_1^{(t)}, \dots, \mX_N^{(t)})$. We note that the second inequality follows from the nonexpansiveness property described in {Lemma} \ref{NONEXPANSIVENESS PROPERTY OF POLAR RETRACTION_1} of {Appendix} \ref{Technical tools used in proofs}.

The remainder of the proof is to quantify the last two terms in \eqref{expansion of distance in tensor factorization-main paper} to ensure the decay of the distance.
On the one hand, we can obtain an upper bound of the second term in \eqref{expansion of distance in tensor factorization-main paper} as
\begin{eqnarray}
    \label{RIEMANNIAN FACTORIZATION SQUARED TERM UPPER BOUND-main paper}
    \sum_{i=1}^{N}\frac{1}{\gamma_i}\bigg\|\calP_{\text{T}_{L({\mX}_i)} \text{St}}\bigg(\nabla_{L({\mX}_{i})}f(\mX_1^{(t)}, \dots, \mX_N^{(t)})\bigg)\bigg\|_F^2
    \leq\frac{9N-5}{4}\|\calX^{(t)}-\calX^\star\|_F^2,
\end{eqnarray}
which ensures a bounded Riemannian gradient when the iterates converge to the target solution. On the other hand, under the assumption that $\text{dist}^2(\{\mX_i^{(t)} \},\{ \mX_i^\star \})\leq \frac{\underline{\sigma}^2(\calX^\star)}{72(N^2-1)(N+1+\sum_{i=2}^{N-1}r_i)}$, we can lower bound the third term in \eqref{expansion of distance in tensor factorization-main paper} as
\begin{eqnarray}
    \label{RIEMANNIAN FACTORIZATION CROSS TERM LOWER BOUND-main paper}
    &\!\!\!\!\!\!\!\!&\sum_{i=1}^{N} \bigg\< L(\mX_i^{(t)})-L_{\mR^{(t)}}(\mX_i^\star),\calP_{\text{T}_{L({\mX}_i)} \text{St}}\bigg(\nabla_{L({\mX}_{i})}f(\mX_1^{(t)}, \dots, \mX_N^{(t)})\bigg)\bigg\>\nonumber\\
    &\!\!\!\!\geq\!\!\!\!&\frac{1}{128(N+1+\sum_{i=2}^{N-1}r_i)\kappa^2(\calX^\star)}\text{dist}^2(\{\mX_i^{(t)} \},\{ \mX_i^\star \}) + \frac{1}{8}\|\calX^{(t)}-\calX^\star\|_F^2,
\end{eqnarray}
which implies that the negative direction of the Riemannian gradient points toward the optimal factors.

Plugging \eqref{RIEMANNIAN FACTORIZATION CROSS TERM LOWER BOUND-main paper} and \eqref{RIEMANNIAN FACTORIZATION SQUARED TERM UPPER BOUND-main paper} into \eqref{expansion of distance in tensor factorization-main paper} yields the convergence of the factors
\begin{eqnarray}
    \label{The final conclusion-main paper}
\text{dist}^2(\{\mX_i^{(t+1)} \},\{ \mX_i^\star \})&\!\!\!\!\leq\!\!\!\!&\bigg(1-\frac{1}{64(N+1+\sum_{i=2}^{N-1}r_i)\kappa^2(\calX^\star)}\mu\bigg)\text{dist}^2(\{\mX_i^{(t)} \},\{ \mX_i^\star \})\nonumber\\
    &\!\!\!\!\!\!\!\!& +\bigg(\frac{9N-5}{4}\mu^2- \frac{\mu}{4}\bigg)\|\calX^{(t)}-\calX^\star\|_F^2\nonumber\\
    &\!\!\!\!\leq\!\!\!\!&\bigg(1-\frac{1}{64(N+1+\sum_{i=2}^{N-1}r_i)\kappa^2(\calX^\star)}\mu\bigg)\text{dist}^2(\{\mX_i^{(t)} \},\{ \mX_i^\star \}),
\end{eqnarray}
where the last line uses the fact that the step size $\mu\leq\frac{1}{9N-5}$.

\paragraph{Connection to matrix and Tucker tensor factorization approaches}
There have been numerous studies on nonconvex matrix estimation \citep{Tu16,Wang17,Jin17,Zhu18TSP,Li20,Ma21TSP,tong2021accelerating} and Tucker tensor estimation \citep{TongTensor21,XiaTC19,Han20}. However, we note that most of the theoretical analyses developed for the matrix case cannot be directly extended to the TT factorization approach since the orthonormal constraint is not applied there.  On the other hand, the highly unbalanced nature of orthonormal matrices and a core tensor in the Tucker tensor make it more likely that the theoretical analysis in the Tucker factorization estimation can be applied to the TT factorization estimation. In the Tucker tensor estimation,  the introduction of an approximately orthonormal structure has led to the development of a regularized gradient descent algorithm \citep{Han20}, which has been demonstrated effectively to achieve a linear convergence rate. However, the results presented in \citep{Han20} primarily pertain to order-$3$ Tucker tensors. When extended to high-order tensors, both the theoretical convergence rate and the initial conditions are likely to deteriorate exponentially with respect to the order $N$.
One reason for this deterioration is that the factor matrices are only guaranteed to be approximately rather than exactly orthogonal in each iteration, which may lead to a high condition number of the product of multiple approximately orthogonal matrices in the theoretical analysis.
This issue is addressed in our approach by strictly enforcing orthonormality of the factors.

\section{Low-rank Tensor-train Sensing}
\label{LR TT sensing all}

In this section, we consider the problem of recovering a low-TT-rank tensor $\calX^\star$ from its linear measurements
\begin{eqnarray}
    \label{Definition of tensor sensing}
    \vy = \calA(\calX^\star) =\begin{bmatrix}
          y_1 \\
          \vdots \\
          y_m
        \end{bmatrix} = \begin{bmatrix}
          \<\calA_1, \calX^\star \> \\
          \vdots \\
          \<\calA_m, \calX^\star \>
        \end{bmatrix}\in\R^m,
\end{eqnarray}
where $\calA(\calX^\star): \R^{d_{1}\times  \cdots \times d_{N}}\rightarrow \R^m$ is a linear map modeling the measurement process.  This problem appears in many applications such as quantum state tomography \citep{lidiak2022quantum,qin2024quantum}, neuroimaging analysis \citep{zhou2013tensor,li2017parsimonious}, 3D imaging  \citep{guo2011tensor}, high-order interaction pursuit \citep{hao2020sparse}, and more.

To enable the recovery of the  low-TT-rank tensor $\calX^\star$ from its linear measurements, the sensing operator is required to satisfy certain properties. One desirable property is the following Restricted Isometry Property (RIP), which has been widely studied and popularized in the compressive sensing literature~\citep{donoho2006compressed,candes2006robust,
candes2008introduction,recht2010guaranteed}, and has been extended for structured tensors~\citep{grotheer2021iterative,Rauhut17,qin2024quantum}.
\begin{definition} (Restricted Isometry Property \citep{Rauhut17})
A linear operator $\calA: \R^{d_{1} \times \cdots \times d_{N}} \\ \rightarrow \R^m$ is said to satisfy the $\ol r$-restricted isometry property ($\ol r$-RIP) with constant $\delta_{\ol r}$ if
\begin{eqnarray}
    \label{RIP condition fro the tensor train sensing}
    (1-\delta_{\ol r})\|\calX\|_F^2\leq \frac{1}{m}\|\calA(\calX)\|_2^2\leq(1+\delta_{\ol r})\|\calX\|_F^2,
\end{eqnarray}
holds for any TT format tensors $\calX\in\R^{d_{1} \times \cdots \times d_{N}}$ with TT  ranks $\vr = (r_1,\ldots,r_{N-1}), r_i \le \ol r$.
\end{definition}
In words, the RIP ensures a stable embedding for TT format tensors and guarantees that the energy $\|\calA(\calX)\|_2^2$ is proportional to $\|\calX\|_F^2$.
The RIP can often be attained by randomly selecting measurement operators from a specific distribution, with subgaussian measurement ensembles serving as a common example.
\begin{definition}(Subgaussian measurement ensembles~\citep{bierme2015modulus})
A real random variable $X$ is called $L$-subgaussian if there exists a constant $L > 0$ such that $\E e^{t X} \le e^{L^2t^2/2}
$ holds for all $t \in \R$. Typical examples include the Gaussian random variable and the Bernoulli random variable. We say that $\calA: \R^{d_1\times \cdots \times d_N} \rightarrow \R^m$  is an $L$-subgaussian measurement ensemble if all the elements of $\calA_k,k=1,\dots,m$ are independent $L$-subgaussian random variables with mean zero and variance one.
\label{def:subgaussian}\end{definition}

The following result shows that the RIP holds with high probability for $L$-subgaussian measurement ensembles.

\begin{theorem} (\citep[Theorem 4]{Rauhut17},\citep[Theorem 2]{qin2024quantum})
\label{RIP condition fro the tensor train sensing Lemma}
Suppose that the linear map $\calA: \R^{d_{1}\times  \cdots \times d_{N}}\rightarrow \R^m$  is an $L$-subgaussian measurement ensemble. Let $\delta_{\ol r}\in(0,1)$ denote a positive constant. Then, with probability at least $1-\epsilon$, $\calA$ satisfies the $\ol r$-RIP as in \eqref{RIP condition fro the tensor train sensing} for any TT format tensors $\calX\in\R^{d_{1} \times \cdots \times d_{N}}$ with TT  ranks $\vr = (r_1,\ldots,r_{N-1}), r_i \le \ol r$, given that
\begin{equation}
m \ge C \cdot \frac{1}{\delta_{\bar{r}}^2} \cdot \max\left\{ N\ol d\ol r^2\log(N\ol r), \log(1/\epsilon)\right \},
\label{eq:mrip}
\end{equation}
where $\ol d=\max_i d_i$ and $C$ is a universal constant depending only on $L$.
\end{theorem}

\Cref{RIP condition fro the tensor train sensing Lemma} ensures the RIP for $L$-subgaussian measurement ensembles with a number of measurements $m$ only scaling linearly, rather than exponentially, with respect to the tensor order $N$. When the RIP holds, then for any two distinct TT format tensors $\calX_1,\calX_2$ with TT ranks smaller than $\ol r$, noting that $\calX_1 - \calX_2$ is also a TT format tensor according to \eqref{summation of TT format}, we have distinct measurements since
\[
\frac{1}{m}\|\calA(\calX_1)- \calA(\calX_2)\|_2^2 = \frac{1}{m}\|\calA(\calX_1- \calX_2)\|_2^2 \ge (1- \delta_{2\ol r}) \|\calX_1- \calX_2\|_F^2,
\]
which guarantees the possibility of exact recovery. We will now examine the convergence of RGD for solving the factorized optimization problem given by \eqref{Loss Function of general tensor sensing intro}.

\subsection{Exact Recovery with Linear Convergence}
\label{noiseless TT sensing}

We start by restating the factorized approach in \eqref{Loss Function of general tensor sensing intro} that minimizes the discrepancy between the measurements $\vy$ and the linear mapping of the estimated  low-TT-rank tensor $\calX$ as
\begin{eqnarray}
    \label{Loss Function of general tensor sensing}
    \begin{split}
    \min_{\mbox{\tiny$\begin{array}{c}
     {\mX_i}\in\R^{r_{i-1}\times d_i \times r_i},\\
     i\in [N]\end{array}$}} & g({\mX}_1,\dots, {\mX}_N) =\frac{1}{2m}\|\calA([\mX_1,\dots, \mX_N]) - \vy\|_2^2,\\
     &\st \ L^\top({\mX}_i)L({\mX}_i)=\mId_{r_i},i \in [N-1].
    \end{split}
\end{eqnarray}
As for \eqref{Loss Function of general tensor sensing}, we solve the above optimization problem over the Stiefel manifold by the following RGD:
\begin{eqnarray}
    \label{updating equation of Riemannian gradient descent algorithm in the tensor sensing 1}
    &&\hspace{-1.8cm}L(\mX_i^{(t+1)}) \! =\! \text{Retr}_{L(\mX_i)}\big(L(\mX_i^{(t)})-\frac{\mu}{\ol {\sigma}^2(\calX^\star)}\calP_{\text{T}_{L({\mX_i})} \text{St}}\big(\nabla_{L({\mX}_{i})} g(\mX_1^{(t)}, \dots, \mX_N^{(t)})\big) \big),  i\in [N-1],\\
    \label{updating equation of Riemannian gradient descent algorithm in the tensor sensing 2}
    &&\hspace{-1.8cm}L(\mX_N^{(t+1)}) \! =\! L(\mX_N^{(t)})-\mu\nabla_{L({\mX}_{N})} g(\mX_1^{(t)}, \dots, \mX_N^{(t)}),
\end{eqnarray}
where expressions for the gradients are given in {Appendix} \ref{Local Convergence Proof of Riemannman gradient descent Tensor Sensing} (see \eqref{The definition of gradient in the tensor sensing noiseless}). As discussed in Section~\ref{GD algorithm in TF}, our primary focus lies in examining the local convergence of the RGD algorithm. Before presenting the local convergence, we first discuss and analyze a spectral initialization approach, which serves as an initial value for the optimization algorithm.

\paragraph{Spectral Initialization}
\label{Spectral Initialization in sensing noiseless}

To provide a good initialization for the RGD algorithm, we apply the following spectral initialization method:
\begin{eqnarray}
    \label{initialization eqn of spetral}
    \calX^{(0)}=\text{SVD}_{\vr}^{tt}\bigg(\frac{1}{m}\sum_{k=1}^m y_k\calA_k\bigg),
\end{eqnarray}
where $\text{SVD}_{\vr}^{tt}(\cdot )$ is the TT-SVD algorithm \citep{Oseledets11} that can efficiently find an approximately optimal TT approximation to any tensor. This spectral initialization approach has been  widely employed for various inverse problems \citep{lu2020phase}, such as phase retrieval \citep{candes2015phase,luo2019optimal}, low-rank matrix recovery \citep{Ma21TSP,tong2021accelerating}, and structured tensor recovery \citep{cai2019nonconvex,TongTensor21,Han20}.
Here, when $\calA$ satisfies the RIP condition, we can ensure that the initialization  $\calX^{(0)}$ is close to $\calX^\star$.

\begin{theorem}
\label{TENSOR SENSING SPECTRAL INITIALIZATION}
When $\calA$ satisfies the $3\ol r$-RIP for TT format tensors with a constant $\delta_{3\ol r}$, the spectral initialization generated by \eqref{initialization eqn of spetral} satisfies
\begin{eqnarray}
    \label{TENSOR SENSING SPECTRAL INITIALIZATION1}
    \|\calX^{(0)}-\calX^\star\|_F\leq \delta_{3\ol r}(1+\sqrt{N-1})\|\calX^\star\|_F.
\end{eqnarray}
\end{theorem}
The proof is provided in {Appendix} \ref{Proof of in the spectral initialization}. Referring to \Cref{TENSOR SENSING SPECTRAL INITIALIZATION}, it can be observed that by ensuring a sufficiently small $\delta_{3\ol r}$, we can always find a suitable initialization within a desired distance to the ground truth.

\paragraph{Exact recovery with linear convergence of RGD}
Again, our analysis utilizes the left-orthogonal form factors $\{\mX_i^\star\}$ of a minimal TT decomposition of $\calX^\star$, although the factors are not required for implementing the RGD algorithm in \eqref{updating equation of Riemannian gradient descent algorithm in the tensor sensing 1} and \eqref{updating equation of Riemannian gradient descent algorithm in the tensor sensing 2}. We now establish a local linear convergence guarantee for RGD.
\begin{theorem}
\label{Local Convergence of Riemannian in the sensing_Theorem}
Consider a low-TT-rank tensor  $\calX^\star$ with ranks $\vr = (r_1,\dots, r_{N-1})$. Assume $\calA$ obeys the $(N+3)\ol r$-RIP with a constant $\delta_{(N+3)\ol r}\leq\frac{4}{15}$ and $\ol r=\max_i r_i$. Given $\vy = \calA(\calX^\star)$,  let  $\{\mX_i^{(t)} \}_{t\geq 0}$ be the iterates generated by the RGD algorithm in \eqref{updating equation of Riemannian gradient descent algorithm in the tensor sensing 1} and \eqref{updating equation of Riemannian gradient descent algorithm in the tensor sensing 2}.
Suppose the algorithm is initialized with $\{\mX_i^{(0)} \}$ satisfying
\begin{eqnarray}
    \label{Local Convergence of Riemannian in the sensing_Theorem initialization}
    \dist^2(\{\mX_i^{(0)} \},\{ \mX_i^\star \})\leq \frac{(4-15\delta_{(N+3)\ol r})\underline{\sigma}^2(\calX^\star)}{8(N+1+\sum_{i=2}^{N-1}r_i)(57N^2+393N-450)}
\end{eqnarray}
and uses step size $\mu\leq \frac{4 - 15\delta_{(N+3)\ol r}}{10(9N -5)(1 + \delta_{(N+3)\ol r})^2}$.  Then, the iterates $\{\mX_i^{(t)} \}_{t\geq 0}$  converge linearly to $\{\mX_i^\star  \}$ (up to rotation):
\begin{eqnarray}
    \label{Local Convergence of Riemannian in the sensing_Theorem_1}
    \dist^2(\{\mX_i^{(t+1)} \},\{ \mX_i^\star \})\leq\bigg(1-\frac{4-15\delta_{(N+3)\ol r}}{320(N+1+\sum_{i=2}^{N-1}r_i)\kappa^2(\calX^\star)}\mu\bigg)\dist^2(\{\mX_i^{(t)} \},\{ \mX_i^\star \}).
\end{eqnarray}
\end{theorem}
The proof is given in {Appendix} \ref{Local Convergence Proof of Riemannman gradient descent Tensor Sensing}. We omit the overview of the proof as it shares a similar structure to the one in Section~\ref{GD algorithm in TF} for the TT factorization problem, with the key difference being the involvement of the sensing operator and the utilization of the RIP.
Our results demonstrate that, similar to the findings presented in \Cref{Local Convergence of Stiefel_Theorem}, when the initial condition $\dist^2(\{\mX_i^{(0)} \},\{ \mX_i^\star \})\leq O(\frac{\underline{\sigma}^2(\calX^\star)}{N^3\ol r})$ is satisfied during the initial stage, RGD exhibits a linear convergence rate of $1 - O(\frac{1}{ N^2\ol r\kappa^2(\calX^\star)})$. Notably, both the convergence rate and the initial requirement depend on the tensor order $N$ only polynomially rather than exponentially. Note that the RIP is required to hold for TT format tensors with rank $(N+3)\ol r$; this is because the analysis involves a summation of $N+3$ TT format tensors of rank $\ol r$ as in \eqref{expansion of distance in tensor factorization-main paper} and \eqref{RIEMANNIAN FACTORIZATION CROSS TERM LOWER BOUND-main paper}.
However, the requirement of $(N+3)\ol r$-RIP may perhaps be improved to $2\ol r$-RIP (or higher, such as $4\ol r$-RIP, but  independent of $N$) by using an alternative approach to bound the term \eqref{Lower BOUND OF cross TERM OF Riemannian GD IN SENSING Conclusion original} in {Appendix} \ref{Local Convergence Proof of Riemannman gradient descent Tensor Sensing}. A thorough investigation of this improvement is left for future work.  To utilize the spectral initialization, we can invoke {Lemma} \ref{LOWER BOUND OF TWO DISTANCES main paper} to rewrite \eqref{Local Convergence of Riemannian in the sensing_Theorem initialization} in terms of the tensors, i.e., $\|\calX^{(0)}-\calX^\star\|_F^2\leq \frac{(4-15\delta_{(N+3)\ol r})\underline{\sigma}^2(\calX^\star)}{64(57N^2+393N-450)(N+1+\sum_{i=2}^{N-1}r_i)^2\kappa^2(\calX^\star)}\lesssim O(\frac{\underline{\sigma}^2(\calX^\star)}{N^4 {\ol r}^2\kappa^2(\calX^\star)})$; see {Lemma} \ref{LOWER BOUND OF TWO DISTANCES} in {Appendix} \ref{Technical tools used in proofs} for the details.  Thus, \Cref{TENSOR SENSING SPECTRAL INITIALIZATION} ensures that spectral initialization satisfies the requirement \eqref{Local Convergence of Riemannian in the sensing_Theorem initialization}, provided that $\delta_{3\ol r} \lesssim \frac{\underline{\sigma}^2(\calX^\star)}{N^{\frac{5}{2}} \ol r \|\calX^\star\|_F^2}$.
\begin{corollary}
\label{summary of noiseless tensor sensing}
Consider a low-TT-rank tensor  $\calX^\star$ with ranks $\vr = (r_1,\dots, r_{N-1})$. Assume that $\calA$ obeys the $(N+3)\ol r$-RIP with constants satisfying $\delta_{3\ol r} \lesssim \frac{\underline{\sigma}^2(\calX^\star)}{N^{\frac{5}{2}} \ol r \|\calX^\star\|_F^2}$ and $ \delta_{(N+3)\ol r}\leq\frac{4}{15}$.
When utilizing the spectral initialization and the step size $\mu\leq \frac{4 - 15\delta_{(N+3)\ol r}}{10(9N -5)(1 + \delta_{(N+3)\ol r})^2}$, RGD converges linearly to a global minimum as in \eqref{Local Convergence of Riemannian in the sensing_Theorem_1}.
\end{corollary}
When the linear map $\calA$ is an  $L$-subgaussian measurement ensemble, using \eqref{eq:mrip}, the RIP requirement is satisfied when
$m\gtrsim \frac{N^6 \ol d \ol r^3\|\calX^\star\|_F^2\kappa^2(\calX^\star)\log (N \ol r)}{\underline{\sigma}^2(\calX^\star)}$.
This requires the number of measurements $m$ grows only polynomially in $N$; however, the order is suboptimal due to the following reasons: $(i)$ TT-SVD, used in the spectral initialization, only provides a quasi-optimal low-rank TT approximation---since computing the optimal low-rank TT approximation is NP-hard \citep{hillar2013most},
$(ii)$ the requirement of $(N+3)\ol r$-RIP in the convergence analysis, and $(iii)$ the relationship between the distances of the factors and of the entire tensors may not be tight in Lemma~\ref{LOWER BOUND OF TWO DISTANCES main paper}.

\paragraph{Special case: Matrix sensing} When $N = 2$, the tensor becomes a matrix and the TT decomposition simplifies to matrix factorization. In this case, the recovery problem reduces to the matrix sensing problem for the rank-$r$ matrix $\mX^\star$. Our \Cref{Local Convergence of Riemannian in the sensing_Theorem} ensures a local linear convergence for RGD with rate of convergence $1 -  \frac{(4-15\delta_{5r})\sigma_r^2(\mX^\star)}{960\sigma_1^2(\mX^\star)}\mu$ when $\calA$ satisfies the $5r$-RIP \citep{CandsTIT11}  with constant $\delta_{5r}\leq \frac{4}{15}$.
Note that the result can be improved to only requiring $3r$-RIP by using {Lemma} \ref{TRANSFORMATION OF NPLUS2 VARIABLES_1} in the analysis of cross terms. In comparison, the work \citep{Tu16} establishes local linear convergence for gradient descent (GD) solving a regularized problem with rate of convergence $1 - \frac{4\sigma_r^2(\mX^\star)}{25\sigma_1^2(\mX^\star)}\mu$, given that the sensing operator $\calA$ satisfies the $6r$-RIP with  constant $\delta_{6r}\leq \frac{1}{6}$. Without relying on any regularizer to balance the two factors, the work  \citep{Ma21TSP} also establishes local linear convergence for GD with rate of convergence $(1-\frac{\sigma_r(\mX^\star)}{50\sigma_1(\mX^\star)}\mu)^2$ when the sensing operator $\calA$ satisfies the $2r$-RIP with constant $\delta_{2r}\leq c_1$ (where $c_1$ is a constant). We can observe that the convergence guarantee for RGD is similar to that of GD in the context of matrix sensing.

We conduct a numerical experiment to compare the performance of the RGD algorithm with GD \citep{Ma21TSP} for matrix sensing. A ground truth matrix $\mX^\star\in\R^{30\times 20}$ of rank $r$ is generated by first creating a random Gaussian matrix with i.i.d.\ entries from a normal distribution, followed by computing a best rank-$r$ approximation. The step size for both the GD and RGD algorithms is set to $\mu = 0.6$. For each experimental setting, we set $m = 150 r$, conduct 20 Monte Carlo trials, and then take the average over the 20 trials to report the results. Figure~\ref{convergence analysis of GD and RGD matrix sensing} shows that both GD and RGD achieve a similar linear convergence rate.

\begin{figure}[!ht]
\centering
\includegraphics[width=5.5cm, keepaspectratio]%
{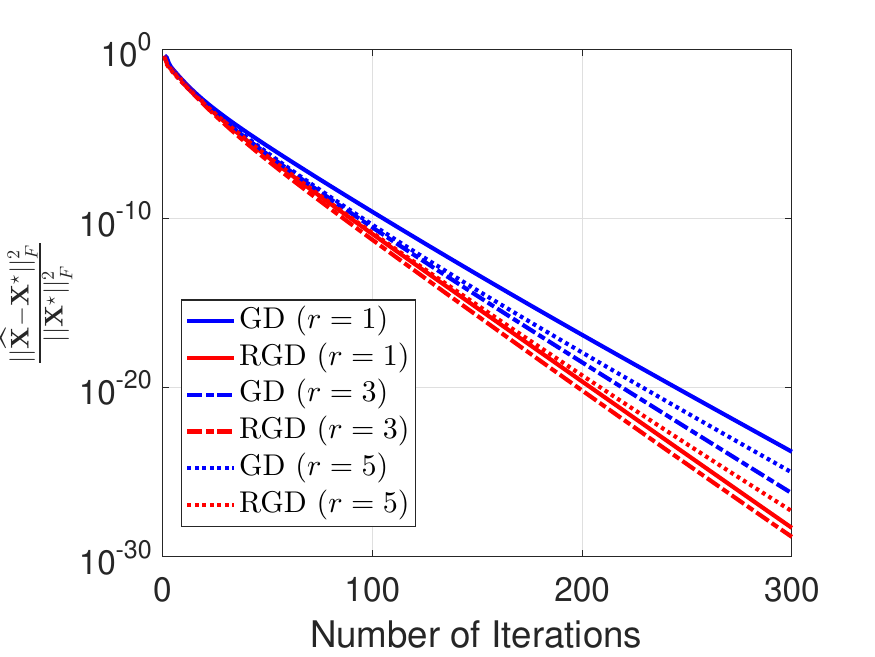}
\caption{Convergence analysis of GD and RGD with $m = 150 r$.}
\label{convergence analysis of GD and RGD matrix sensing}
\end{figure}

\paragraph{Comparison with IHT \citep{Rauhut17} and Riemannian gradient descent on the fixed-rank manifold \citep{budzinskiy2021tensor}}
To the best of our knowledge, our work is the first to offer a convergence guarantee for directly solving the factorization approach for  low-TT-rank  in recovery problems. Two iterative algorithms, iterative hard thresholding (IHT) \citep{Rauhut17} and RGD on the TT manifold \citep{budzinskiy2021tensor}, have been studied with local convergence guarantees. Specifically, the IHT algorithm (GD with TT-SVD truncation) has been proven to converge linearly with a rate of convergence $\frac{a}{4}$ ($a\in(0,1)$), but relies on an unverified perturbation bound of TT-SVD, expressed as
$\|\text{SVD}_{\vr}^{tt}({{\calX}}^{(t)}) - {{\calX}}^{(t)} \|_F\leq (1 + \frac{a^2}{17(1 + \sqrt{1 + \delta_{3r}}\|\calA\| )^2})\|{{\calX}}^{(t)} - \calX^\star \|_F$ where $\calX^{(t)}$ represents the iterate after the gradient update in the $t$-th iteration. As mentioned in \citep{Rauhut17}, $\|\calA\|$ may increase exponentially in terms of $N$, imposing a very strong requirement on the optimality of TT-SVD. Viewing all the TT format tensors of fixed rank as an embedded manifold in $\R^{d_1\times \cdots \times d_N}$ \citep{holtz2012manifolds}, the work \citep{budzinskiy2021tensor} introduces an RGD algorithm on the embedded manifold with TT-SVD retraction to approximate the projection of the estimated tensor onto the embedded manifold. It establishes local linear convergence with the rate $(1+\sqrt{N-1})(\frac{2\delta_{3\ol r} }{1 - \delta_{3\ol r}} + \frac{2}{1 -\delta_{3\ol r}} \frac{\|\calX^{(0)} - \calX^\star \|_F}{\underline{\sigma}(\calX^\star)}  )$. This holds under the conditions  $\|\calX^{(0)}-\calX^\star\|_F^2\lesssim O(\frac{\underline{\sigma}^2(\calX^\star)}{N^2})$, $3\ol r$-RIP with RIP constant $\delta_{3\ol r}\leq\frac{1}{3+2\sqrt{N-1}}$, and an unverified condition $\|\calA^*\calA\|\leq C $ (where $C$ is a constant). Additionly, as mentioned in \citep{Rauhut17}, the Riemannian gradient in the RGD depends on the curvature information at $\calX^\star$ of the embedded manifold, which is often unknown a priori. 

Note that both algorithms require the estimation of the entire tensor $\calX$ in each iteration  and rely on performing the TT-SVD to project the iterates back to the TT format, which demands a substantial amount of storage memory and could be computationally expensive for high-order tensors. In contrast, the factorization approach avoids the need to compute the entire tensor in each iteration.
Specifically, we can employ a tensor contraction operation \citep{cichocki2014tensor} to efficiently compute the gradient without explicitly computing the tensor $\calX$, such as for the term\footnote{Specifically, this term $\<\calA_k, [\mX_1,\dots, \mX_N] \>$ can be efficiently evaluated as $\calA_k\times_1^1 \mX_1\times_{N,1}^{1,2}\mX_2\times_{N-1,1}^{1,2}\cdots \times_{2,1}^{1,2} \mX_{N}$, where we reshape $\mX_1$ and $\mX_N$ as matrices of size ${d_1\times r_1}$ and ${r_{N-1}\times d_N}$, respectively. Here, the tensor contraction operation $\calA_k\times_1^1 \mX_1$ results in a new tensor $\calB$ of size $d_2\times d_3\times \cdots \times d_N\times r_1$, with the  $(s_2,s_3,\ldots,s_{N+1})$-th entry being $\sum_{s_1}\calA_k(s_1,\ldots,s_N)\mX_1(s_1,s_{N+1})$. Likewise, $\calB \times_{N,1}^{1,2}\mX_2$ results in a new tensor of size $d_3\times d_4\times \cdots \times d_N\times r_2$, with the $(s_3,\ldots,s_N,s_{N+2})$-th entry being $\sum_{s_2,s_{N+1}}\calB(s_2,s_3,\ldots,s_{N+1})\mX_2(s_{N+1},s_2,s_{N+2})$.} $\<\calA_k, [\mX_1,\dots, \mX_N] \>$. Intuitively, this is the same as in the matrix case where we can efficiently compute $\<\mA,\vu\vv^\top\>$ as $\<\mA^\top\vu,\vv\>$ without the need of computing $\vu\vv^\top$ for any $ \mA\in\R^{d_1\times d_2}, \vu\in\R^{d_1}, \vv\in \R^{d_2}$.

\subsection{Recovery Guarantee for Noisy TT Format Tensor Sensing}
\label{noisy TT sensing}
In practice, measurements are often noisy, either due to additive noise or the probabilistic nature of the measurements, as seen in quantum state tomography \citep{qin2024quantum}, which causes statistical error in the measurements and can also be modeled as additive noise. In this subsection, we will extend our analysis to TT sensing with noisy measurements of the form
\begin{eqnarray}
    \label{Noisy Definition of tensor sensing}
    \wh \vy = \calA(\calX^\star) + \vepsilon \in\R^m,
\end{eqnarray}
where the noise vector $\vepsilon$ is assumed to be a Gaussian random vector with a mean of zero and a covariance of $\gamma^2\mId_{m}$. Similar to \eqref{Loss Function of general tensor sensing}, we attempt to estimate the target low-TT-rank tensor $\calX^\star$ by solving the following constrained factorized optimization problem:
\begin{eqnarray}
    \label{Noisy Loss Function of general tensor sensing}
    \begin{split}\hspace{-0.5cm}\min_{\mbox{\tiny$\begin{array}{c}
     {\mX_i}\in\R^{r_{i-1}\times d_i \times r_i}\\
     i\in[N]\end{array}$}} & G({\mX}_1,\dots, {\mX}_N) = \frac{1}{2m}\|\calA([\mX_1,\dots, \mX_N]) - \wh \vy\|_2^2,\\
     &\st \ L^\top({\mX}_i)L({\mX}_i)=\mId_{r_i},i \in [N-1].
    \end{split}
\end{eqnarray}

\paragraph{Spectral Initialization}
\label{Spectral Initialization in sensing}
In the midst of a noisy environment, we can still employ the spectral initialization method to obtain an appropriate initialization $\calX^{(0)}$; that is
\begin{eqnarray}
    \label{initialization eqn of spectral noise}
    \calX^{(0)}=\text{SVD}_{\vr}^{tt}\bigg(\frac{1}{m}\sum_{k=1}^m \wh y_k \calA_k\bigg),
\end{eqnarray}
which is guaranteed to be close to the ground-truth $\calX^\star$
when the operator $\calA$ satisfies the RIP.

\begin{theorem}
\label{TENSOR SENSING noisy SPECTRAL INITIALIZATION}
Assume that the operator $\calA$ satisfies the $3\ol r$-RIP for TT format tensors with constant $\delta_{3\ol r}$ and that the additive noise vector $\vepsilon$ is randomly generated from the distribution $\calN({\bm 0},\gamma^2\mId_{m})$. Then with probability at least $1-2e^{-c_1N\ol d\ol r^2 \log N}$, the spectral initialization in \eqref{initialization eqn of spectral noise} satisfies
\begin{eqnarray}
    \label{Noisy TENSOR SENSING SPECTRAL INITIALIZATION1}
    \|\calX^{(0)}-\calX^\star\|_F\leq (1+\sqrt{N-1})\bigg(\delta_{3\ol r}\|\calX^\star\|_F + \frac{c_2\ol r\sqrt{(1+\delta_{3\ol r})N\ol d\log N} }{\sqrt{m}}\gamma\bigg),
\end{eqnarray}
where $c_1,c_2$ are positive constants, $\ol r=\max_{i=1}^{N-1} r_i$ and $\ol d=\max_{i=1}^N d_i$.
\end{theorem}
The proof is provided in {Appendix} \ref{Proof of in the noisy spectral initialization}. Compared to the noiseless case, \eqref{Noisy TENSOR SENSING SPECTRAL INITIALIZATION1} for the noisy scenario includes an additional term caused by the additive noise in the measurement, which only scales polynomially with $N$ thanks to the concise structure of the TT factorization.

\paragraph{Local convergence of RGD algorithm}

The following result ensures that, given an appropriate initialization, RGD will converge to the target low-TT-rank tensor up to a certain distance that is proportional to the noise level.
\begin{theorem}
\label{Local Convergence of Riemannian in the noisy sensing_Theorem}
Consider a low-TT-rank tensor  $\calX^\star$ with ranks $\vr = (r_1,\dots, r_{N-1})$. Assume that $\calA$ obeys the $(N+3)\ol r$-RIP with a constant $\delta_{(N+3)\ol r}\leq\frac{7}{30}$, where $\ol r=\max_i r_i$. Suppose that the RGD algorithm in \eqref{updating equation of Riemannian gradient descent algorithm in the tensor sensing 1} and \eqref{updating equation of Riemannian gradient descent algorithm in the tensor sensing 2} is initialized with $\{\mX_i^{(0)} \}$ satisfying
\begin{eqnarray}
    \label{Local Convergence of Riemannian in the noisy sensing_Theorem initialization}
    \dist^2(\{\mX_i^{(0)} \},\{ \mX_i^\star \})\leq \frac{(7 - 30\delta_{(N+3)\ol r})\underline{\sigma}^2(\calX^\star)}{8(N+1+\sum_{i=2}^{N-1}r_i)(129N^2+7231N-7360)}
\end{eqnarray}
and uses the step size $\mu\leq\frac{7 - 30\delta_{(N+3)\ol r}}{20(9N-5)(1+\delta_{(N+3)\ol r})^2}$.  Then, with probability at least $1 \! - 2Ne^{-\Omega(N\ol d\ol r^2 \log N)} \\  - 2e^{-\Omega(N^3\ol d\ol r^2 \log N)}$, the iterates $\{\mX_i^{(t)} \}_{t\geq 0}$ generated by RGD satisfy
\begin{eqnarray}
    \label{Local Convergence of Riemannian in the noisy sensing_Theorem_1}
    \dist^2(\{\mX_i^{(t+1)} \},\{ \mX_i^\star \})&\!\!\!\!\leq\!\!\!\!&\bigg(1-\frac{7 - 30\delta_{(N+3)\ol r}}{1280(N+1+\sum_{i=2}^{N-1}r_i)\kappa^2(\calX^\star)}\mu\bigg)^{t+1}
    \dist^2(\{\mX_i^{(0)} \},\{ \mX_i^\star \})\nonumber\\
&\!\!\!\!\!\!\!\!&\hspace{-1cm} +  O\bigg(\frac{(N + \mu)(N+1+\sum_{i=2}^{N-1}r_i)(1+\delta_{(N+3)\ol r}) N^2\ol d\ol r^2(\log N) \kappa^2(\calX^\star)\gamma^2 }{m(7 - 30\delta_{(N+3)\ol r})}\bigg)\nonumber\\
\end{eqnarray}
as long as $m\geq C \frac{N^5\ol d \ol r^3 (\log N)\kappa^2(\calX^\star) \gamma^2}{\underline{\sigma}^2(\calX^\star)} $ with a universal constant $C$ and $\ol d=\max_i d_i$.
\end{theorem}
The proof is given in {Appendix} \ref{Local Convergence Proof of Riemannman gradient descent Noisy Tensor Sensing}.
\Cref{Local Convergence of Riemannian in the noisy sensing_Theorem} provides a similar guarantee to that in \Cref{Local Convergence of Riemannian in the sensing_Theorem} for noisy measurements and shows that once the initial condition is satisfied, RGD converges at a linear rate to the target solution, up to a statistical error due to the additive noise. Notably, the second term in \eqref{Local Convergence of Riemannian in the noisy sensing_Theorem_1} scales linearly in terms of the variance $\gamma^2$, and polynomially in terms of the tensor order $N$. This is due to the presence of a polynomial number of degrees of freedom in the TT format, denoted as $O(N\ol d\ol r^2)$, which effectively mitigates the impact of noise. Similar to the discussion following \Cref{Local Convergence of Riemannian in the sensing_Theorem}, the requirement on the initialization \eqref{Local Convergence of Riemannian in the noisy sensing_Theorem initialization} can be rewritten as $\|\calX^{(0)}-\calX^\star\|_F^2\lesssim O(\frac{\underline{\sigma}^2(\calX^\star)}{N^4 {\ol r}^2\kappa^2(\calX^\star)})$. Using \Cref{TENSOR SENSING noisy SPECTRAL INITIALIZATION} and \Cref{RIP condition fro the tensor train sensing Lemma}, for an $L$-subgaussian measurement ensemble, the requirement for initialization and local convergence can be met by using the number of measurements $m\gtrsim\frac{N^6 \ol d \ol r^3 \kappa^2(\calX^\star)(\|\calX^\star\|_F^2\log (N \ol r) + N\ol r \log N  \gamma^2 ) }{\underline{\sigma}^2(\calX^\star)}$, which compared with the noiseless case has an additional factor that scales with noise level $\gamma^2$. 
The recovery guarantee \eqref{Local Convergence of Riemannian in the noisy sensing_Theorem_1} scales optimally with the noise level $\gamma^2$, but is suboptimal with respect to the tensor order $N$ when compared to the following minimax lower bound from \citep[Theorem 3]{qin2024computational}.

\begin{corollary}
\label{minimax bound of error difference}
Consider the tensor sensing problem in \eqref{Noisy Definition of tensor sensing} for TT format tensors $\calX \in \setX_{\vr} = \{\calX \in\R^{d_1\times\cdots\times d_{N}}: \vr = (r_{1},\dots,r_{N-1})\}$. Suppose $\min_{i=1}^{N-1} r_i\geq C$ for some absolute constant $C$, and each element of $\{\calA_i \}_i$ and $\vepsilon$ in \eqref{Noisy Definition of tensor sensing} respectively are drawn independently from the distributions $\calN(0,1)$ and $\calN(0,\gamma^2)$. Then, it holds that
\begin{eqnarray}
    \label{minimax bound of error final_ conclusion}
    \inf_{\wh \calX}\sup_{\calX\in\setX_{\vr}}\E{\|\wh{\calX} - \calX\|_F} \gtrsim \sqrt{\frac{\sum_{j=1}^{N}d_jr_{j-1}r_j}{m}}\gamma.
\end{eqnarray}
\end{corollary}
The above minimax lower bound demonstrates the potential for improving the recovery guarantee in \Cref{Local Convergence of Riemannian in the noisy sensing_Theorem}, which together with {Lemma} \ref{LOWER BOUND OF TWO DISTANCES main paper} states that $\| \calX^{(t)} - \calX^\star\|_F\leq O\bigg( \frac{N^5\ol d \ol r^3 (\log N)\kappa^2(\calX^\star) \gamma^2}{m}\bigg)$ for sufficiently large $t$.
Similar to the discussion right after Corollary~\ref{summary of noiseless tensor sensing}, the suboptimal dependence with tensor order $N$ arise from the following factors:
$(i)$ the requirement of $(N+3)\ol r$-RIP in the convergence analysis, and $(ii)$ the exchange between the distances of the factors and the entire tensors. We leave the improvement as the subject of future work.

\section{Numerical Experiments}
\label{Numerical experiments}

In this section, we conduct numerical experiments to evaluate the performance of the RGD algorithm for tensor train sensing and completion.
In all the experiments, we generate an order-$N$ ground truth tensor $\calX^\star\in\R^{d_1\times\cdots \times d_N}$ in TT format with ranks $\vr = (r_1,\dots, r_{N-1})$ by first generating a random Gaussian tensor with i.i.d. entries from the normal distribution, and then using the sequential SVD algorithm to obtain a TT format tensor, which is finally normalized to unit Frobenius norm, i.e., $\|\calX^\star\|_F=1$. To simplify the selection of parameters, we set $d=d_1=\cdots=d_N$ and $r=r_1=\cdots = r_{N-1}$. For the RGD algorithm in \eqref{updating equation of Riemannian gradient descent algorithm in the tensor sensing 1} and \eqref{updating equation of Riemannian gradient descent algorithm in the tensor sensing 2}, we set $\mu = 0.5$ to compute factors. To avoid the high computational complexity associated with $\ol {\sigma}^2(\calX^\star)$, we replace it with $\|\calX^\star\|_F^2$ in the RGD.  For each experimental setting, we conduct 20 Monte Carlo trials and then take the average over the 20 trials to report the results.

\subsection{TT Format Tensor Sensing}
We first consider the tensor sensing problem by generating each  measurement operator $\calA_i, i=1,\dots,m$ as a random tensor with i.i.d.\ entries drawn from the standard normal distribution. We then obtain noisy measurements $\wh y_i = \langle \calA_i, \calX^\star \rangle + \epsilon_i$, where the noise $\epsilon$ is drawn from a Gaussian distribution with a mean of zero and a variance of $\gamma^2$.

\begin{figure}[!ht]
\centering
\subfigure[]{
\begin{minipage}[t]{0.31\textwidth}
\centering
\includegraphics[width=5.2cm]{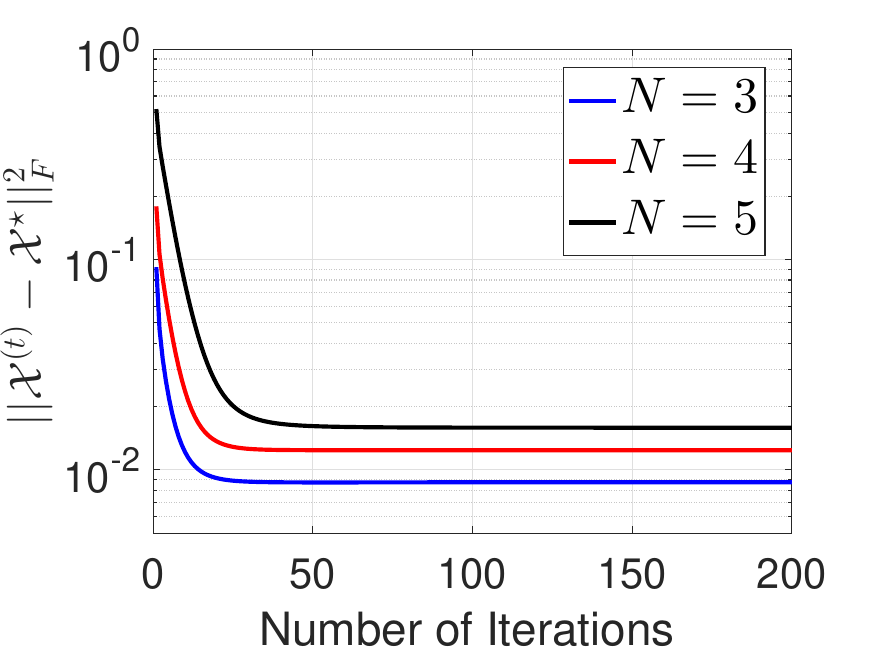}
\end{minipage}
\label{TT_sensing_N}
}
\subfigure[]{
\begin{minipage}[t]{0.31\textwidth}
\centering
\includegraphics[width=5.2cm]{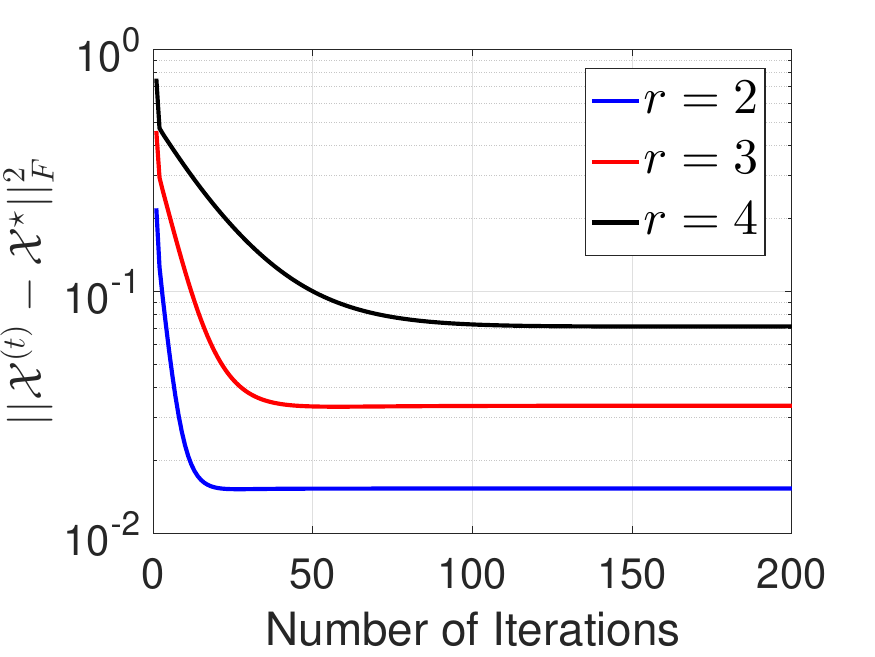}
\end{minipage}
\label{TT_sensing_r}
}
\subfigure[]{
\begin{minipage}[t]{0.31\textwidth}
\centering
\includegraphics[width=5.2cm]{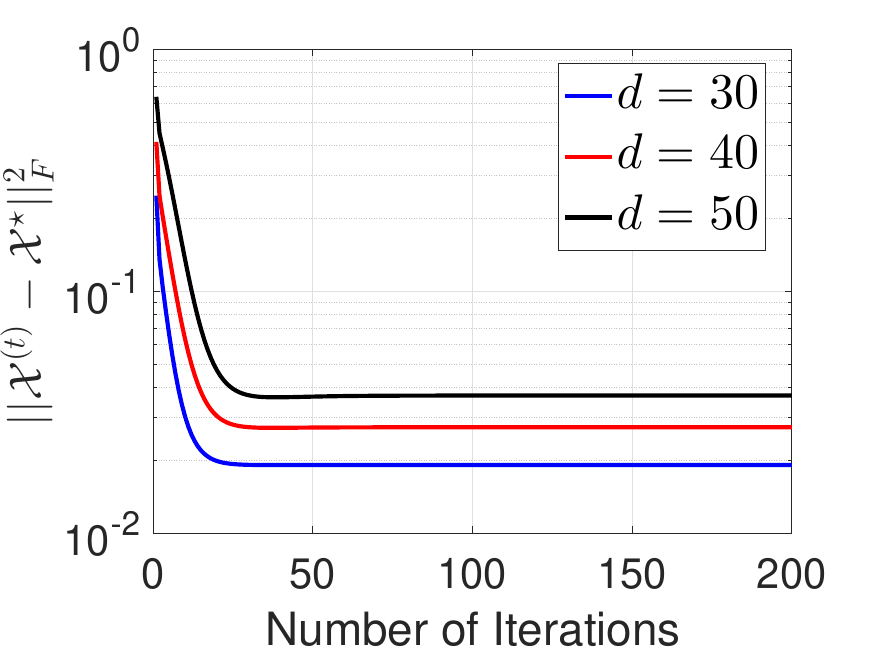}
\end{minipage}
\label{TT_sensing_d}
}
\caption{Convergence of RGD for TT format tensor sensing (a) for different $N$ with $d = 10$, $r = 2$, $m = 1000$, and $\gamma^2 = 0.1$, (b) for different $r$ with $d = 50$, $N = 3$, $m = 3000$, and $\gamma^2 = 0.1$, (c) for different $d$ with $N = 3$, $r = 2$, $m = 1500$, and $\gamma^2 = 0.1$.}
\label{fig:TT_sensing_convergence}
\end{figure}

\begin{figure}[!ht]
\centering
\subfigure[]{
\begin{minipage}[t]{0.31\textwidth}
\centering
\includegraphics[width=5.2cm]{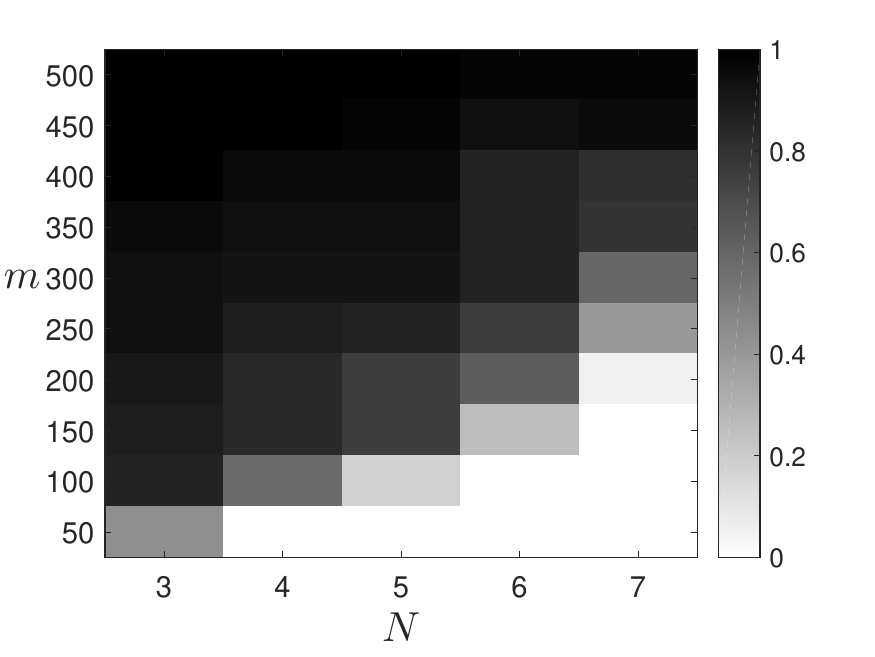}
\end{minipage}
\label{TT_sensing_noiseless different N and m}
}
\subfigure[]{
\begin{minipage}[t]{0.31\textwidth}
\centering
\includegraphics[width=5.2cm]{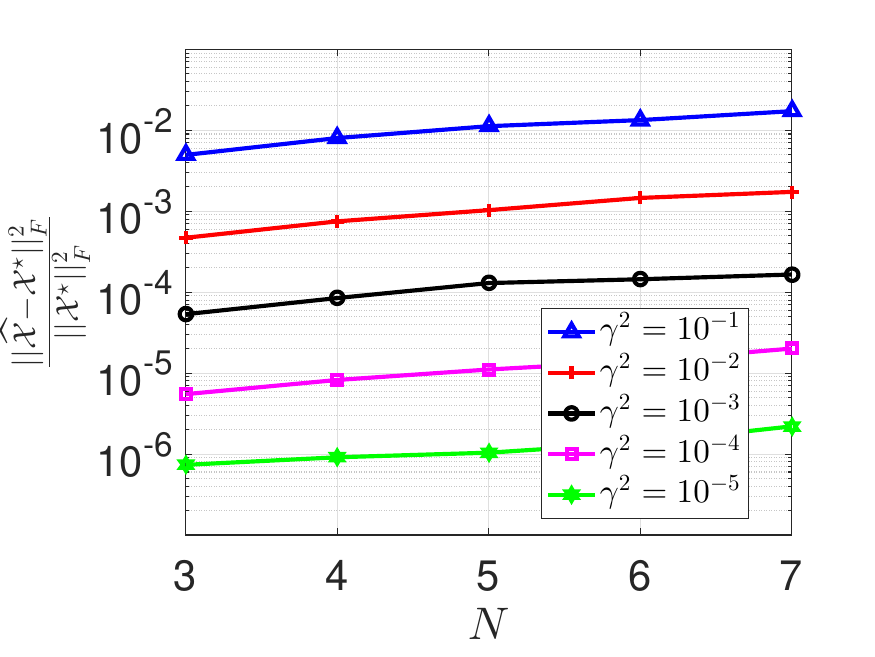}
\end{minipage}
\label{TT_sensing_noisy different N and gamma}
}
\subfigure[]{
\begin{minipage}[t]{0.31\textwidth}
\centering
\includegraphics[width=5.2cm]{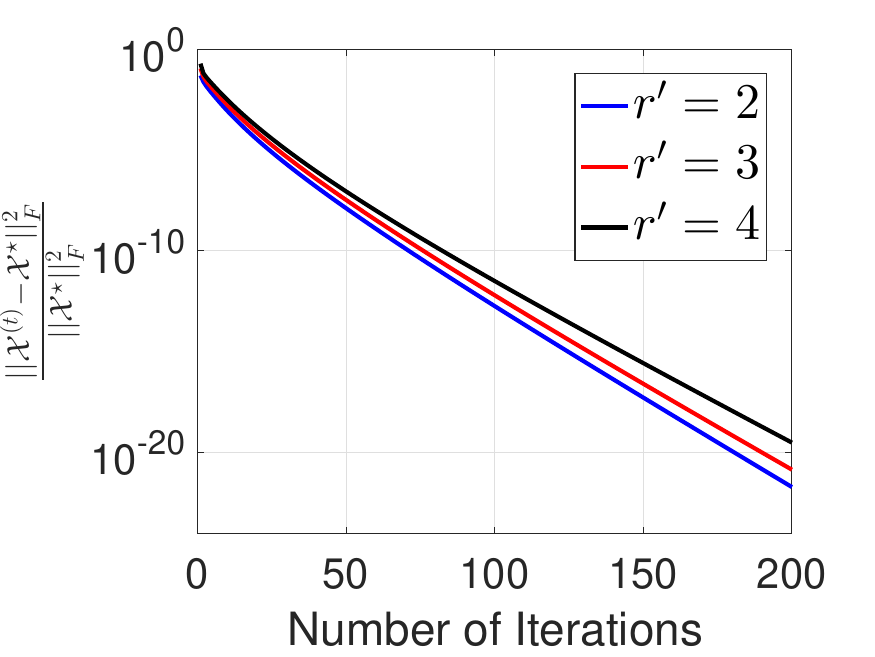}
\end{minipage}
\label{TT_sensing_overpara}
}
\caption{Performance comparison of RGD for TT format tensor sensing (a) for different $N$ and $m$ with $d = 4$, $r = 2$, and $\gamma^2 = 0$, (b) for different $N$ and $\gamma^2$ with $d = 4$, $r = 2$, and $m=500$, (c) for overparameterized tensor sensing where $N=3$, $d = 30$, $r = 2$, $\gamma^2=0$, and $m = 5000$. }
\end{figure}
\paragraph{Convergence of RGD} We first display the convergence of RGD in terms of the tensor, denoted as $\|\calX^{(t)} - \calX^\star\|_F^2$, for different settings in Figure \ref{fig:TT_sensing_convergence}. We observe rapid convergence of RGD across all cases shown in Figure \ref{fig:TT_sensing_convergence}. Furthermore, the plots reveal the following trends when a fixed number of measurements $m$ is maintained, while the values of $N$, $r$, and $d$ increase: $(i)$ the recovery error at the initialization using spectral methods increases,
$(ii)$ RGD converges more slowly,
and $(iii)$ RGD converges to a solution with larger recovery error. These observations align with our theoretical findings as presented in Theorem \ref{TENSOR SENSING noisy SPECTRAL INITIALIZATION} for spectral initialization and Theorem \ref{Local Convergence of Riemannian in the noisy sensing_Theorem} for the convergence guarantee of RGD. Since RGD converges relatively fast, as demonstrated in \Cref{fig:TT_sensing_convergence}, in the following experiments we run RGD for $T = 500$ iterations to obtain the estimated tensor $\wh \calX$.

\paragraph{Exact recovery with clean measurements} In the following two sets of experiments, we fix $d =4$ and $r = 2$, and evaluate the performance for varying tensor order $N$. In the case of clean measurements, we conduct experiments for different $N$ and number of measurements $m$, and we say a recovery is successful if
$\|\wh\calX - \calX^\star \|_F\leq10^{-5}$. We conduct $100$ independent trials to evaluate the success rate for each pair of $N$ and $m$. The result is displayed in \Cref{TT_sensing_noiseless different N and m}. We can observe from \Cref{TT_sensing_noiseless different N and m} that the number of measurements $m$ needed to ensure exact recovery only scales linearly rather than exponentially in terms of the order $N$, consistent with the findings in \Cref{RIP condition fro the tensor train sensing Lemma}. This relationship is indeed an improvement over the polynomial dependence of $m$ on $N$ required to guarantee \Cref{TENSOR SENSING SPECTRAL INITIALIZATION} and \Cref{Local Convergence of Riemannian in the sensing_Theorem}. This also suggests the potential for refining our analysis in future work.

\paragraph{Stable recovery with noisy measurements} In the case of noisy measurements, we fix the number of measurements $m = 500$ and vary the tensor order $N$ and noise level $\gamma^2$. \Cref{TT_sensing_noisy different N and gamma} shows that the performance of RGD remains stable as $N$ increases, with recovery error in the curves increasing polynomially.  This behavior aligns with the findings outlined in \Cref{Local Convergence of Riemannian in the noisy sensing_Theorem}. In addition, the recovery error scales roughly linearly with respect to the noise level $\gamma^2$, consistent with the second term in \eqref{Local Convergence of Riemannian in the noisy sensing_Theorem_1}.

\paragraph{Recovery with over-parameterization} In the fourth experiment, we consider the case where the true rank $r$ is unknown a priori, and we use an estimated rank $r'$ in RGD. Figure~\ref{TT_sensing_overpara} shows the convergence rates of RGD across various values of $r'$ for the setting with clean measurements. While our current theory is only established when the rank is exactly specified (i.e., $r' = r$), we also observe linear convergence when the rank is over-specified (i.e., $r' > r$), albeit with a slightly slower rate of convergence as $r'$ increases. A theoretical study for this overparameterized case is a topic for future research.

\subsection{TT Format Tensor Completion}
We now consider the problem of tensor completion, with the goal of reconstructing the entire tensor $\calX^\star$ based on a subset of its entries. 
Specifically, let $\Omega$ denote the indices of $m$ observed entries and let $\calP_{\Omega}$ denote the corresponding measurement operator that produces the observed measurements. Then, our measurements $\wh \vy$ are obtained by
\[
\wh\vy = \calP_{\Omega}(\calX^\star) + \vepsilon,
\]
where $\vepsilon$ denotes the possible additive noise with each entry being independently drawn from a Gaussian distribution with a mean of zero and a variance of $\gamma^2$. Throughout the experiments, we assume the $m$ observed entries are uniformly sampled. As the measurement operator $\calP_{\Omega}$ can be viewed as a special instance of the linear map $\calA$ in \eqref{Definition of tensor sensing}, tensor completion can be regarded as a special case of tensor sensing. Thus, as in \eqref{Noisy Loss Function of general tensor sensing}, we attempt to recover the underlying tensor by solving the following constrained factorized optimization problem
\begin{align*}
\hspace{-0.5cm}\min_{\mbox{\tiny$\begin{array}{c}
     {\mX_i}\in\R^{r_{i-1}\times d_i \times r_i}\\
     i\in[N]\end{array}$}} & G({\mX}_1,\dots, {\mX}_N) = \frac{1}{2}\|\calP_{\Omega}([\mX_1,\dots, \mX_N]) - \wh \vy\|_2^2,\\
     &\st \ L^\top({\mX}_i)L({\mX}_i)=\mId_{r_i},i \in [N-1].
\end{align*}
We solve this problem using the RGD algorithm in \eqref{updating equation of Riemannian gradient descent algorithm in the tensor sensing 1} and \eqref{updating equation of Riemannian gradient descent algorithm in the tensor sensing 2}. In addition, we employ a modified spectral initialization approach proposed in \citep{Cai2022provable} with incoherent factors and guaranteed performance. We note that, in general, the measurement operator $\calP_{\Omega}$ in tensor completion does not satisfy the RIP condition \citep{rauhut2015tensor}. Therefore, our theory may not be directly applicable in this context. Here, we only present numerical results of RGD for tensor completion, deferring the theoretical analysis to future work.

\begin{figure}[!ht]
\centering
\subfigure[]{
\begin{minipage}[t]{0.31\textwidth}
\centering
\includegraphics[width=5.2cm]{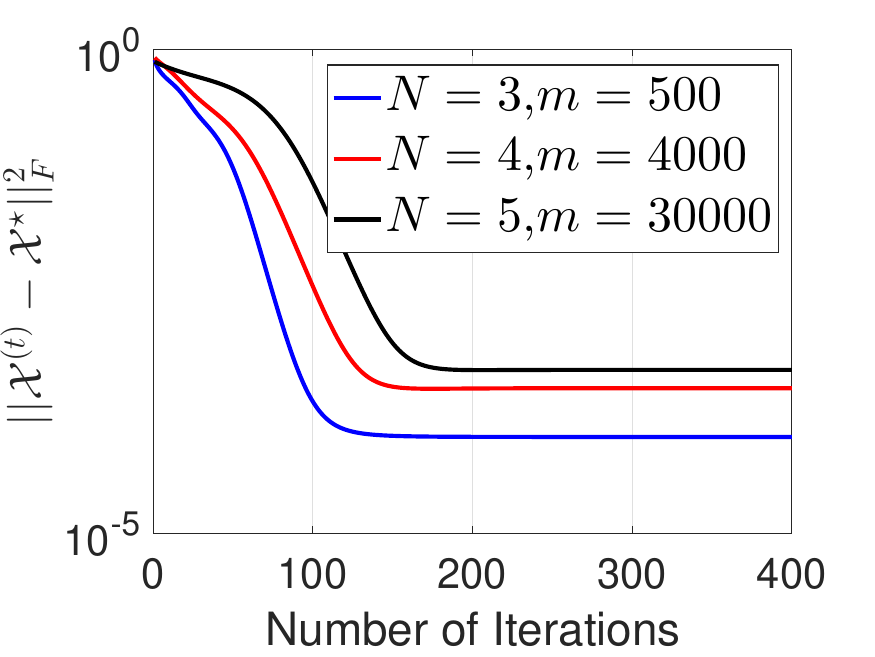}
\end{minipage}
\label{TT_completion_N}
}
\subfigure[]{
\begin{minipage}[t]{0.31\textwidth}
\centering
\includegraphics[width=5.2cm]{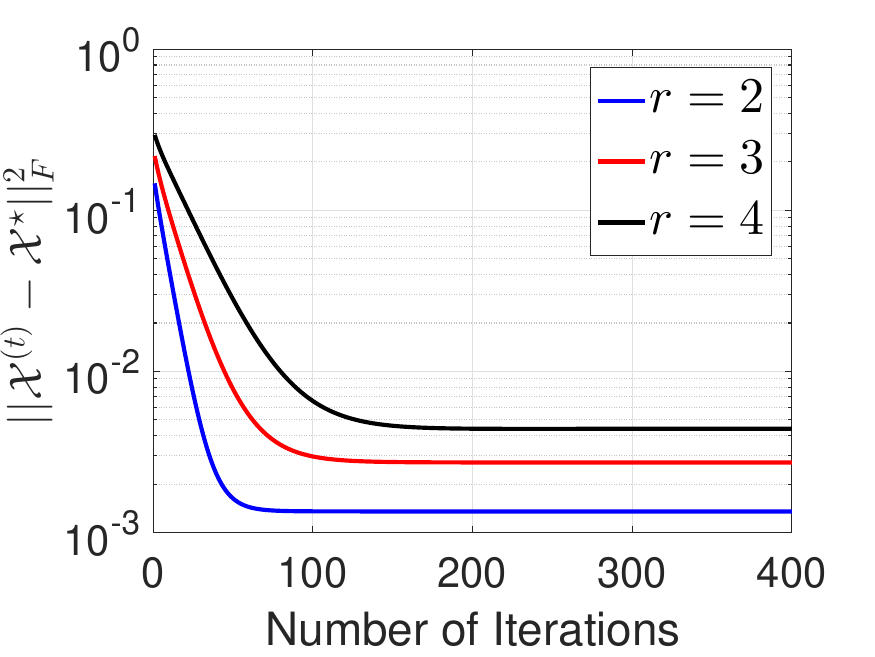}
\end{minipage}
\label{TT_completion_r}
}
\subfigure[]{
\begin{minipage}[t]{0.31\textwidth}
\centering
\includegraphics[width=5.2cm]{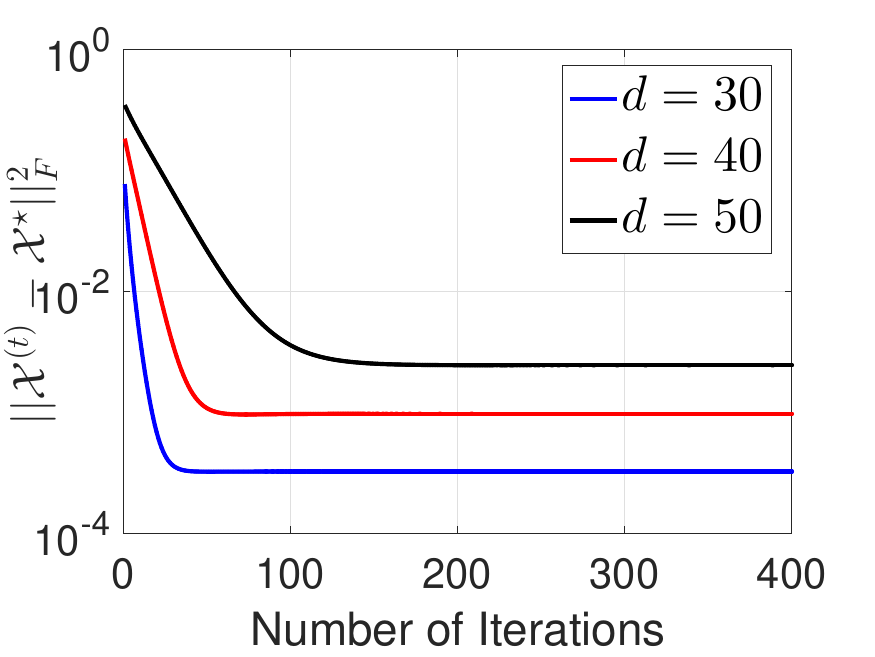}
\end{minipage}
\label{TT_completion_d}
}
\caption{Convergence of RGD for TT format tensor completion (a) for different $N$ and $m$ with $d = 10$, $r = 2$, and $\gamma^2 = 10^{-6}$, (b) for different $r$ with $d = 50$, $N = 3$, $m = 35000$, and $\gamma^2 = 10^{-6}$, (c) for different $d$ with $N = 3$, $r = 2$, $m = 20000$, and $\gamma^2 = 10^{-6}$. }
\end{figure}

\begin{figure}[!ht]
\centering
\subfigure[]{
\begin{minipage}[t]{0.31\textwidth}
\centering
\includegraphics[width=5.2cm]{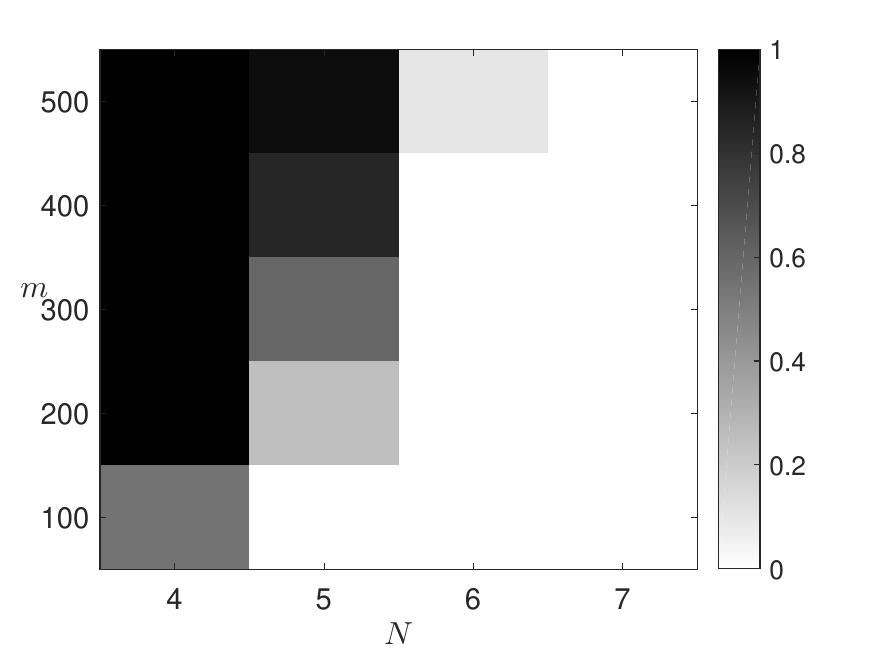}
\end{minipage}
\label{TT_completion_m and d fixed}
}
\subfigure[]{
\begin{minipage}[t]{0.31\textwidth}
\centering
\includegraphics[width=5.2cm]{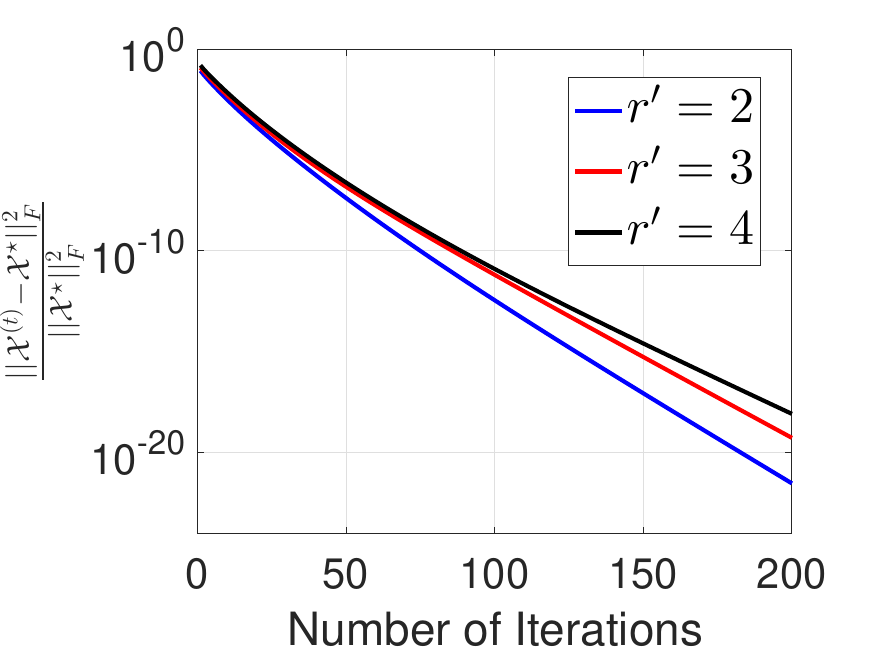}
\end{minipage}
\label{TT_completion_overpara}
}
\caption{Performance comparison of RGD for TT format tensor completion, (a) for different $N$ and $m$ with $d = 4$, $r = 2$, and $\gamma^2 = 0$, (note that for $N=4$ in (a), $m=256$ has been chosen when $m\geq 300$), (b) for overparameterized tensor completion where $N=3$, $d = 30$, $r = 2$, and $m = 20000$. }
\end{figure}

\paragraph{Convergence of RGD }

Continuing with the same experiment conducted in tensor sensing, we begin by demonstrating the convergence rate of RGD in tensor completion under various settings. The results presented in Figures~\ref{TT_completion_N}-\ref{TT_completion_d} reveal a noticeable trend: as the values of $N$, $r$, and $d$ increase, similar to the observations made in tensor sensing, we witness a degradation in the convergence rate, recovery error, and the estimated initialization. Additionally, it is important to emphasize that this consistency in degradation across different parameters reinforces the similarities between tensor completion and tensor sensing when employing RGD.

\paragraph{Exact recovery with clean measurements}

In the second experiment, we set $d = 4$ and $r = 2$, and then assess the performance across various tensor orders $N$. When dealing with clean measurements, we perform experiments using different combinations of $N$ and the number of samples $m$. A successful recovery by RGD is defined as $\|\wh\calX - \calX^\star \|_F\leq10^{-5}$. For each pair of $N$ and $m$, we conduct 100 independent trials to evaluate the success rate.
Note that random initialization is employed here because the sequential second-order moment method \citep{Cai2022provable} is likely to fail when the number of measurements  $m$ is relatively small compared to the total number of entries of the tensor.
The results are presented in Figure~\ref{TT_completion_m and d fixed}. It is evident that the number of samples needed for successful recovery does not exhibit a polynomial relationship with $N$.
This discovery aligns with the theoretical result in \citep{Cai2022provable} that requires the number of samples $m$ to increase exponentially with $N$.

\paragraph{Recovery with over-parameterization}

In the third experiment, we conclude by assessing the performance of RGD with overparameterized rank. This evaluation involves varying pre-defined values of $r'$, as illustrated in Figure~\ref{TT_completion_overpara}. Notably, the results indicate a clear trend as observed in Figure~\ref{TT_sensing_overpara} for the sensing problem: the convergence rate diminishes as $r'$ increases, aligning with the observations from the tensor sensing scenario.

\section{Conclusion and Outlook}
\label{conclusion}

In this paper, we
study the tensor factorization approach for recovering  low-TT-rank tensors from limited numbers of linear measurements. To avoid the ambiguity and to facilitate theoretical analysis, we optimize over the left-orthogonal TT format which enforces orthonormality among all  the factors except for the last one. To ensure the orthonormal structure, we utilize the Riemannian gradient descent (RGD) algorithm for optimizing those factors over the Stiefel manifold. When the sensing operator obeys the RIP, we show that with an appropriate initialization which can be achieved by spectral initialization, RGD converges to the target solution at a linear rate. In the presence of measurement noise, RGD produces a stable recovery with error proportional to the noise level and scaling only polynomially in terms of the tensor order. Our findings support the growing evidence for using the factorization approach for  low-TT-rank tensors and adopting local search algorithms such as gradient descent for solving the corresponding factorized optimization problems.

\paragraph{Extension to tensor factorization approach without orthonormal constraints} In this work, we use the left-orthogonal form to avoid scaling ambiguity among the tensor factors and to facilitate theoretical analysis. This approach is in line with other works on matrix and tensor factorization that use regularizers to balance the factors \citep{Tu16,Han20,cai2019nonconvex, TongTensor21}. However, both empirical observations \citep{Zhu18TSP} and theoretical results \citep{Ma21TSP,li2020global} have shown that such regularizers are not necessary for the convergence of gradient descent (GD) in matrix factorization problems. Similarly, recent work \citep{jameson2024optimal} has empirically demonstrated that GD can efficiently solve the TT factorization problem without orthogonal constraints or regularizers on the factors, as applied in quantum state tomography. Extending the analysis to provide a theoretical justification for this approach will be of interest.

\paragraph{Extension to other TT applications} An important area for future work is the analysis of the convergence properties of RGD in TT completion. Due to the random sampling process, the incoherence condition \citep{Cai2022provable} plays a crucial role in ensuring the even distribution of energy across the entries of the tensor.
Although experimental results have shown a linear convergence rate for RGD, there is a theoretical challenge in guaranteeing the nonexpansiveness property when applying both the orthonormal structure and incoherence condition simultaneously.
Additionally, unlike the tensor itself, the TT rank is often unknown beforehand in practical scenarios.
Building upon recent research efforts \citep{stoger2021small,jiang2022algorithmic,ding2022validation,xu2023power}, a possible extension of our analysis is to consider overparameterized low-rank tensor recovery. This extension would involve investigating the convergence and error analysis of RGD in this context.


\acks{We acknowledge funding support from NSF Grants No.\ CCF-1839232, CCF-2106834, CCF-2241298 and ECCS-2409701. We thank the Ohio Supercomputer Center for providing the computational resources needed in carrying out this work. Finally, we are grateful to Stephen Becker, Alireza Goldar, Zhexuan Gong, Casey Jameson, Jingyang Li, Alexander Lidiak, and Gongguo Tang for many valuable discussions and for helpful comments on the manuscript.}



\appendix
\section{Technical tools used in the proofs}
\label{Technical tools used in proofs}

In this section, we introduce a new operation related to the multiplication of submatrices within the left unfolding matrices $L(\mX_i)=\begin{bmatrix}\mX_i(1) \\ \vdots\\  \mX_i(d_i) \end{bmatrix}\in\R^{(r_{i-1}d_i) \times r_i}, i \in [N]$. For simplicity, we will only consider the case $d_i=2$, but extending to the general case is straightforward.

Let $\mA=\begin{bmatrix}\mA_1 \\ \mA_2 \end{bmatrix}$ and $\mB=\begin{bmatrix}\mB_1 \\ \mB_2 \end{bmatrix}$ be two block matrices, where $\mA_i\in\R^{r_1\times r_2}$ and $\mB_i\in\R^{r_2\times r_3}$ for $i=1,2$. We introduce the notation $\ol \otimes$ to represent the Kronecker product between submatrices in the two block matrices, as an alternative to the standard Kronecker product based on element-wise multiplication.
Specifically, we define $\mA\ol \otimes\mB$ as
\begin{eqnarray}
    \label{KRONECKER PRODUCT VECTORIZATION}
    &&\mA\ol \otimes\mB=\begin{bmatrix}\mA_1 \\ \mA_2 \end{bmatrix}\ol \otimes\begin{bmatrix}\mB_1 \\ \mB_2 \end{bmatrix}=\begin{bmatrix}\mA_1\mB_1 \\ \mA_2\mB_1 \\ \mA_1\mB_2 \\ \mA_2\mB_2 \end{bmatrix}.
\end{eqnarray}

Then we establish the following useful result.
\begin{lemma}
\label{KRONECKER PRODUCT VECTORIZATION2}
For any matrices $\mA=\begin{bmatrix}\mA_1 \\ \mA_2 \end{bmatrix}$ and $\mB=\begin{bmatrix}\mB_1 \\ \mB_2 \end{bmatrix}$ , where $\mA_i\in\R^{r_1\times r_2}$ and $\mB_i\in\R^{r_2\times r_3}$, the following inequalities hold
\begin{eqnarray}
    \label{KRONECKER PRODUCT VECTORIZATION3}
    &&\|\mA\ol \otimes\mB\|_F\leq\|\mA\|\cdot\|\mB\|_F,\\
    \label{KRONECKER PRODUCT VECTORIZATION4}
    &&\|\mA\ol \otimes\mB\|\leq\|\mA\|\cdot\|\mB\|.
\end{eqnarray}
In particular, when $r_3=1$, \eqref{KRONECKER PRODUCT VECTORIZATION3} becomes
\begin{eqnarray}
    \label{KRONECKER PRODUCT VECTORIZATION6}
    \|\mA\ol \otimes\mB\|_2\leq\|\mA\|\cdot\|\mB\|_2.
\end{eqnarray}
\end{lemma}

\begin{proof}
Using the relation $\|\mA\ol \otimes\mB\|_F^2 = \trace((\mA\ol \otimes\mB)^\top(\mA\ol \otimes\mB))$ gives
\begin{eqnarray}
    \label{KRONECKER PRODUCT VECTORIZATION7}
    \|\mA\ol \otimes\mB\|_F^2&\!\!\!\!=\!\!\!\!&\trace((\mA\ol \otimes\mB)^\top(\mA\ol \otimes\mB))\nonumber\\
    &\!\!\!\!=\!\!\!\!&\trace(\mB_1^\top\mA^\top\mA\mB_1+\mB_2^\top\mA^\top\mA\mB_2)\nonumber\\
    &\!\!\!\!\leq\!\!\!\!&\|\mB_1\|_F^2\|\mA\|^2+\|\mB_2\|_F^2\|\mA\|^2\nonumber\\
    &\!\!\!\!=\!\!\!\!&\|\mA\|^2\|\mB\|_F^2,
\end{eqnarray}
where the first inequality utilizes the result that   $\trace(\mC\mD) \le \|\mC\|\trace(\mD)$ holds for any two PSD matrices $\mC,\mD$ (see \citep[Lemma 7]{Zhu21TIT}).

Likewise, by connecting the spectral norms between $\mA$ and $(\mA\ol \otimes\mB)^\top(\mA\ol \otimes\mB)$, we have
\begin{eqnarray}
    \label{KRONECKER PRODUCT VECTORIZATION8}
    \|\mA\ol \otimes\mB\|^2&\!\!\!\!=\!\!\!\!&\lambda_{\max}((\mA\ol \otimes\mB)^\top(\mA\ol \otimes\mB))\nonumber\\
    &\!\!\!\!=\!\!\!\!&\lambda_{\max}(\mB_1^\top\mA^\top\mA\mB_1+\mB_2^\top\mA^\top\mA\mB_2)\nonumber\\
    &\!\!\!\!=\!\!\!\!&\max_{\|{\bm u}\|_2=1}{\bm u}^\top\mB_1^\top\mA^\top\mA\mB_1{\bm u}+{\bm u}^\top\mB_2^\top\mA^\top\mA\mB_2{\bm u}\nonumber\\
    &\!\!\!\!\leq\!\!\!\!&\max_{\|{\bm u}\|_2=1}\lambda_{\max}(\mA^\top\mA){\bm u}^\top\mB_1^\top\mB_1{\bm u}+\lambda_{\max}(\mA^\top\mA){\bm u}^\top\mB_2^\top\mB_2{\bm u}\nonumber\\
    &\!\!\!\!=\!\!\!\!&\max_{\|{\bm u}\|_2=1}\lambda_{\max}(\mA^\top\mA)({\bm u}^\top\mB_1^\top\mB_1{\bm u}+{\bm u}^\top\mB_2^\top\mB_2{\bm u})\nonumber\\
    &\!\!\!\!=\!\!\!\!&\lambda_{\max}(\mA^\top\mA)\lambda_{\max}(\mB^\top\mB)\nonumber\\
    &\!\!\!\!=\!\!\!\!&\|\mA\|^2\|\mB\|^2.
\end{eqnarray}

\end{proof}

The inequality \eqref{KRONECKER PRODUCT VECTORIZATION3} can be viewed as a generalization of the result $\|\mC \mD\|_F \le \|\mC\|\cdot \|\mD\|_F$ for any two matrices $\mC,\mD$ of appropriate sizes. However, unlike the matrix product case which also satisfies $\|\mC \mD\|_F \le \|\mC\|_F\cdot \|\mD\|$, $\|\mA\ol \otimes\mB\|_F\leq\|\mA\|_F\|\mB\|$ does not always hold. To upper bound $\|\mA\ol \otimes\mB\|_F$ with the spectral norm of $\mB$, we will instead use $\|\mA\ol \otimes\mB\|_F\leq \|\mA\|_F\cdot \|\mB\|_F \leq {\rm rank}(\mB)\|\mA\|_F\cdot \|\mB\|$. This discrepancy will account for the term $\sum_{i=2}^{N-1}r_i$ in the subsequent {Lemma} \ref{LOWER BOUND OF TWO DISTANCES}.

Applying {Lemma} \ref{KRONECKER PRODUCT VECTORIZATION2} to  the left-orthogonal TT format tensor $\calX^\star = [\mX_1^\star,\dots, \mX_N^\star]$ gives the following useful results:
\begin{eqnarray}
    \label{KRONECKER PRODUCT VECTORIZATION11}
    \hspace{-2cm}&&\|\calX^\star\|_F = \|\text{vec}(\calX^\star)\|_2=\|L(\mX_1^\star) \ol \otimes \cdots \ol \otimes L(\mX_N^\star)\|_2 = \|L(\mX_N^\star)\|_2, \\
    \label{KRONECKER PRODUCT VECTORIZATION11 - 2}
    \hspace{-2cm}&&\|L(\mX_i^\star) \ol \otimes \cdots  \ol \otimes L(\mX_j^\star)\|\leq \Pi_{l = i}^{j}\|L(\mX_l^\star)\| = 1, \ \ i\leq j, \ \ \forall i, j\in[N-1], \\
    \label{KRONECKER PRODUCT VECTORIZATION11 - 3}
    \hspace{-2cm}&&\|L(\mX_i^\star) \ol \otimes \cdots  \ol \otimes L(\mX_j^\star)\|_F\leq \Pi_{l = i}^{j-1}\|L(\mX_l^\star)\| \|L(\mX_j^\star)\|_F = \sqrt{r_j}, \ \ i\leq j, \ \ \forall i, j\in[N-1].
\end{eqnarray}
In addition, according to \eqref{KRONECKER PRODUCT VECTORIZATION}, each row in $L(\mX_1^\star) \ol \otimes \cdots \ol \otimes L(\mX_i^\star)$ can be represented as
\begin{eqnarray}
    \label{each row in Kronecker product L}
    &\!\!\!\!\!\!\!\!&(L(\mX_1^\star) \ol \otimes \cdots \ol \otimes L(\mX_i^\star))(s_1\cdots s_i,:) \nonumber\\
    &\!\!\!\!=\!\!\!\!&(L(\mX_1^\star) \ol \otimes \cdots \ol \otimes L(\mX_i^\star))(s_1 + d_1(s_2-2) + \cdots + d_1\cdots d_{i-1}(s_i-1),:) \nonumber\\
    &\!\!\!\!=\!\!\!\!& \mX_1(s_1)\cdots \mX_i(s_i).
\end{eqnarray}

Next, we provide some useful lemmas in terms of products of matrices.
\begin{lemma}
\label{EXPANSION_A1TOAN-B1TOBN_1}
For any $\mA_i,\mA^\star_i\in\R^{r_{i-1}\times r_i},i\in[N]$, we have
\begin{eqnarray}
    \label{EXPANSION_A1TOAN-B1TOBN_2}
    \mA_1\mA_2\cdots \mA_N-\mA_1^\star\mA_2^\star\cdots \mA_N^\star = \sum_{i=1}^N \mA_1^\star \cdots \mA_{i-1}^\star (\mA_{i} - \mA_i^\star) \mA_{i+1} \cdots \mA_N.
\end{eqnarray}

\end{lemma}

\begin{proof}

The result can be obtained by summing up the following equations:
\begin{align*}
\mA_1\mA_2\cdots \mA_N - (\mA_1 - \mA_1^\star)\mA_2\cdots \mA_N & = \mA_1^\star\mA_2\cdots \mA_N \\
\mA_1^\star\mA_2\cdots \mA_N - \mA_1^\star(\mA_2 - \mA_2^\star) \mA_3\cdots \mA_N & = \mA_1^\star\mA_2^\star \mA_3\cdots \mA_N\\
& ~~~ \vdots\\
\mA_1^\star\mA_2^\star\cdots \mA_{N-1}^\star\mA_N - \mA_1^\star\mA_2^\star\cdots \mA_{N-1}^\star(\mA_N - \mA_N^\star) & = \mA_1^\star\mA_2^\star\cdots \mA_N^\star.
\end{align*}

\end{proof}

\begin{lemma}
\label{TRANSFORMATION OF NPLUS2 VARIABLES_1}
For any $\mA_i,\mA^\star_i\in\R^{r_{i-1}\times r_i}, i \in[N]$, we have
\begin{eqnarray}
    \label{TRANSFORMATION OF NPLUS2 VARIABLES_2}
    &\!\!\!\!\!\!\!\!&\mA_1^\star\cdots \mA_N^\star-\mA_1\dots\mA_{N-1}\mA_N + \sum_{i=1}^{N} \mA_1 \cdots \mA_{i-1} (\mA_i- \mA_i^\star) \mA_{i+1} \cdots \mA_N\nonumber\\
    &\!\!\!\!=\!\!\!\!& \sum_{i=1}^{N-1}\sum_{j=i+1}^N \mA_1\cdots \mA_{i-1}(\mA_i-\mA_i^\star)\mA_{i+1}^\star \cdots \mA_{j-1}^\star (\mA_{j}-\mA_{j}^\star)\mA_{j+1}\cdots \mA_N,
 \end{eqnarray}
where the right-hand side of \eqref{TRANSFORMATION OF NPLUS2 VARIABLES_2} contains a total of $\frac{N(N-1)}{2}$ terms.

\end{lemma}

\begin{proof}
We first rewrite the term $(\mA_1-\mA_1^\star)\mA_2\cdots\mA_N$ as
\begin{eqnarray}
    \label{TRANSFORMATION OF NPLUS2 VARIABLES_4}
    &&(\mA_1-\mA_1^\star)\mA_2\cdots\mA_N = (\mA_1-\mA_1^\star) \mA_2^\star\cdots\mA_N^\star + (\mA_1-\mA_1^\star)\parans{ \mA_2\cdots\mA_N - \mA_2^\star\cdots\mA_N^\star } \nonumber\\
    &\!\!\!\!=\!\!\!\!& (\mA_1-\mA_1^\star) \mA_2^\star\cdots\mA_N^\star + (\mA_1-\mA_1^\star) \parans{\sum_{j=2}^N \mA_{2}^\star \cdots \mA_{j-1}^\star (\mA_{j}-\mA_{j}^\star)\mA_{j+1}\cdots \mA_N},
\end{eqnarray}
where the second line uses {Lemma} \ref{EXPANSION_A1TOAN-B1TOBN_1} for expanding the difference $\mA_2\cdots\mA_N - \mA_2^\star\cdots\mA_N^\star$. We can apply the same approach for $i=2,\dots, N$ to get
\begin{eqnarray}
    \label{TRANSFORMATION OF NPLUS2 VARIABLES_5}
    &\!\!\!\!\!\!\!\! &\mA_1 \cdots \mA_{i-1} (\mA_i- \mA_i^\star) \mA_{i+1} \cdots \mA_N\nonumber\\
    &\!\!\!\! = \!\!\!\!& \mA_1\cdots \mA_{i-1}(\mA_i - \mA_i^\star) \mA_{i+1}^\star \cdots \mA_N^\star \nonumber\\
    &\!\!\!\!  \!\!\!\!& +  \mA_1\cdots \mA_{i-1}(\mA_i - \mA_i^\star)\parans{\sum_{j=i+1}^N \mA_{i+1}^\star \cdots \mA_{j-1}^\star (\mA_{j}-\mA_{j}^\star)\mA_{j+1}\cdots \mA_N}.
\end{eqnarray}
Noting that the sum of the second terms in the right-hand side of \eqref{TRANSFORMATION OF NPLUS2 VARIABLES_4} and \eqref{TRANSFORMATION OF NPLUS2 VARIABLES_5} equals the right-hand side of \eqref{TRANSFORMATION OF NPLUS2 VARIABLES_2}, we complete the proof by checking the rest of the terms:
\begin{eqnarray}
    \label{TRANSFORMATION OF NPLUS2 VARIABLES_7}
    &\!\!\!\!\!\!\!\!&\mA_1^\star\mA_2^\star\cdots \mA_N^\star-\mA_1\cdots\mA_{N-1}\mA_N + \sum_{i=1}^{N}\mA_1\cdots \mA_{i-1}(\mA_i - \mA_i^\star) \mA_{i+1}^\star \cdots \mA_N^\star =0,
\end{eqnarray}
which follows from {Lemma} \ref{EXPANSION_A1TOAN-B1TOBN_1}.

\end{proof}

\begin{lemma} (\citep{Cai2022provable,Han20})
\label{left ortho upper bound}
For any two matrices $\mX, \mX^\star$ with rank $r $, let $\mU \mSigma \mV^\top$ and $\mU^\star \mSigma^\star {\mV^\star}^\top$  respectively represent the compact singular value decompositions (SVDs) of $\mX$ and $\mX^\star$. Supposing that $\mR = \argmin_{\wt{\mR}\in\O^{r\times r}}\|\mU - \mU^\star\wt{\mR}  \|_F $, we have
\begin{eqnarray}
    \label{relationship between left O with full tensor SVD}
    \|\mU - \mU^\star\mR \|_F \leq \frac{2\|\mX-\mX^\star\|_F}{\sigma_r(\mX^\star)}.
\end{eqnarray}
\end{lemma}

\begin{lemma}
\label{LOWER BOUND OF TWO DISTANCES}
For any two TT format tensors $\calX$ and $\calX^\star $ with ranks $\vr = (r_1,\dots, r_{N-1})$. Let $\{\mX_i\}$ and $\{\mX_i^\star\}$ be the corresponding left-orthogonal form factors. Assume $\ol{\sigma}^2(\calX)\leq \frac{9\ol{\sigma}^2(\calX^\star)}{4}$. Then we have
\begin{eqnarray}
    \label{LOWER BOUND OF TWO DISTANCES_1}
&&\|\calX-\calX^\star\|_F^2\geq\frac{\underline{\sigma}^2(\calX^\star)}{8(N+1+\sum_{i=2}^{N-1}r_i)\ol{\sigma}^2(\calX^\star)}\dist^2(\{\mX_i \},\{ \mX_i^\star \}),\\
    \label{UPPER BOUND OF TWO DISTANCES_1}
    &&\|\calX-\calX^\star\|_F^2\leq\frac{9N}{4}\dist^2(\{\mX_i \},\{ \mX_i^\star \}),
\end{eqnarray}
where $\dist^2(\{\mX_i \},\{ \mX_i^\star \})$ is defined in \eqref{BALANCED NEW DISTANCE BETWEEN TWO TENSORS}.

\end{lemma}
\begin{proof}
By the definition of the
$i$-th unfolding of the tensor $\calX$, we have
\begin{eqnarray}
    \label{expansion of unfolding in tensor}
    \calX^{\< i \>} = \mX^{\leq i}\mX^{\geq i+1},
\end{eqnarray}
where each row of the left part $\mX^{\leq i}$ and each column of the right part $\mX^{\geq i+1}$ can be represented as
\begin{eqnarray}
    \label{each row of left part}
    \mX^{\leq i}(s_1\cdots s_i,:) = \mX_1(s_1)\cdots \mX_i(s_i),
\end{eqnarray}
\begin{eqnarray}
    \label{each column of right part}
    \mX^{\geq i+1}(:,s_{i+1}\cdots s_N) = \mX_{i+1}(s_{i+1})\cdots \mX_N(s_N).
\end{eqnarray}

According to \citep[Lemma 1]{Cai2022provable}, the left part in the left-orthogonal TT format satisfies ${\mX^{\leq i}}^\top \mX^{\leq i} = \mId_{r_i}, i\in[N-1]$. Furthermore, based on the analysis of \citep{Cai2022provable} stated in {Lemma} \ref{left ortho upper bound}, we have
\begin{eqnarray}
    \label{relationship between left O with  tensor SVD 1}
    \max_{i=1,\dots,N-1}\|\mX^{\leq i} - {\mX^\star}^{\leq i}\mR_i \|_F \leq \frac{2\|\calX-\calX^\star\|_F}{\underline{\sigma}(\calX^\star)},
\end{eqnarray}
where $\mR_i = \argmin_{\wt{\mR}_i\in\O^{r_i\times r_i}}\|\mX^{\leq i} - {\mX^\star}^{\leq i}\wt{\mR}_i  \|_F $.

By the definition of $L(\mX_1^\star) \ol \otimes \cdots \ol \otimes L(\mX_N^\star)$ in \eqref{each row in Kronecker product L}, we have
\begin{eqnarray}
    \label{equation between two left parts}
    \mX^{\leq i} - {\mX^\star}^{\leq i}\mR_i = L(\mX_1)\ol \otimes \cdots \ol \otimes L(\mX_i)   -  L_{\mR}(\mX_1^\star) \ol \otimes \cdots \ol \otimes L_{\mR}(\mX_i^\star),
\end{eqnarray}
which together with the above equation gives
\begin{eqnarray}
    \label{LOWER BOUND OF TWO DISTANCES_2}
    \|L(\mX_1)\ol \otimes \cdots \ol \otimes L(\mX_i)   -  L_{\mR}(\mX_1^\star) \ol \otimes \cdots \ol \otimes L_{\mR}(\mX_i^\star)  \|_F^2  \leq\frac{4\|\calX-\calX^\star\|_F^2}{\underline{\sigma}^2(\calX^\star)}, i \in [N-1].
\end{eqnarray}

We now use this result to upper bound $\|L(\mX_i)  -  L_{\mR}(\mX_i^\star)  \|_F^2$ for each $i$. First, setting $i = 1$ in the above equation directly yields
\begin{eqnarray}
    \label{LOWER BOUND OF TWO DISTANCES_2 11}
    \|L(\mX_1)  -  L_{\mR}(\mX_1^\star)  \|_F^2  \leq\frac{4\|\calX-\calX^\star\|_F^2}{\underline{\sigma}^2(\calX^\star)}.
\end{eqnarray}

With \eqref{KRONECKER PRODUCT VECTORIZATION11 - 3} and \eqref{LOWER BOUND OF TWO DISTANCES_2}, we can obtain the result for $i =2$ as
\begin{eqnarray}
    \label{LOWER BOUND OF TWO DISTANCES_5}
    &\!\!\!\!\!\!\!\!&\|L({\mX}_2)-L_{\mR}({\mX}_2^\star)\|_F^2\nonumber\\
    &\!\!\!\!=\!\!\!\!&\|L_{\mR}({\mX}_1^\star)\ol \otimes L({\mX}_2)- L_{\mR}({\mX}_1^\star)\ol \otimes L_{\mR}({\mX}^\star_2)\|_F^2\nonumber\\
    &\!\!\!\!=\!\!\!\!&\|L_{\mR}({\mX}_1^\star)\ol \otimes L({\mX}_2)-L({\mX}_1)\ol \otimes L({\mX}_2)+L({\mX}_1)\ol \otimes L({\mX}_2)- L_{\mR}({\mX}_1^\star)\ol \otimes L_{\mR}({\mX}^\star_2)\|_F^2\nonumber\\
    &\!\!\!\!\leq\!\!\!\!&2\|L_{\mR}({\mX}_1^\star)\ol \otimes L({\mX}_2)-L({\mX}_1)\ol \otimes L({\mX}_2)\|_F^2+2\|L({\mX}_1)\ol \otimes L({\mX}_2)- L_{\mR}({\mX}_1^\star)\ol \otimes L_{\mR}({\mX}_2^\star)\|_F^2\nonumber\\
    &\!\!\!\!\leq\!\!\!\!&2\|L({\mX}_2)\|_F^2\|L({\mX}_1)-L_{\mR}({\mX}_1^\star)\|_F^2+2\|L({\mX}_1)\ol \otimes L({\mX}_2)- L_{\mR}({\mX}_1^\star)\ol \otimes L_{\mR}({\mX}_2^\star)\|_F^2\nonumber\\
    &\!\!\!\!\leq\!\!\!\!&\frac{(8r_2+8)\|\calX-\calX^\star\|_F^2}{\underline{\sigma}^2(\calX^\star)}.
\end{eqnarray}
A similar derivation also gives
\begin{eqnarray}
    \label{LOWER BOUND OF TWO DISTANCES_6}
    \|L({\mX}_{i})-L_{\mR}({\mX}_{i}^\star)\|_F^2\leq\frac{(8r_{i}+8)\|\calX-\calX^\star\|_F^2}{\underline{\sigma}^2(\calX^\star)}, \ \  i=3,\dots, N-1.
\end{eqnarray}
Finally, we bound the term for $i = N$ as follows:
\begin{eqnarray}
    \label{LOWER BOUND OF TWO DISTANCES_7}
    &\!\!\!\!\!\!\!\!&\|L({\mX}_N)-L_{\mR}({\mX}_N^\star)\|_2^2\nonumber\\
    &\!\!\!\!=\!\!\!\!&\|L_{\mR}({\mX}_1^\star)\ol \otimes\cdots\ol \otimes L_{\mR}({\mX}^\star_{N-1}) \ol \otimes(L({\mX}_N)-L_{\mR}({\mX}_N^\star))\|_2^2\nonumber\\
    &\!\!\!\!=\!\!\!\!&\|L_{\mR}({\mX}_1^\star)\ol \otimes\cdots\ol \otimes L_{\mR}({\mX}^\star_{N-1}) \ol \otimes L({\mX}_N)-L({\mX}_1)\ol \otimes\cdots\ol \otimes L({\mX}_{N-1}) \ol \otimes L({\mX}_N)\nonumber\\
    &&\!\!\!\!+ L({\mX}_1)\ol \otimes\cdots\ol \otimes L({\mX}_{N-1}) \ol \otimes L({\mX}_N)-L_{\mR}({\mX}_1^\star)\ol \otimes\cdots\ol \otimes L_{\mR}({\mX}^\star_{N-1}) \ol \otimes L_{\mR}({\mX}_N^\star)\|_2^2\nonumber\\
    &\!\!\!\!\leq\!\!\!\!&2\|R({\mX}_{N})\|^2\|L({\mX}_1)\ol \otimes\cdots\ol \otimes L({\mX}_{N-1}) -L_{\mR}({\mX}_1^\star)\ol \otimes\cdots\ol \otimes L_{\mR}({\mX}^\star_{N-1})\|_F^2+2\|\calX-\calX^\star\|_F^2\nonumber\\
    &\!\!\!\!\leq\!\!\!\!&\frac{18\ol{\sigma}^2(\calX^\star)\|\calX-\calX^\star\|_F^2}{\underline{\sigma}^2(\calX^\star)}+2\|\calX-\calX^\star\|_F^2\nonumber\\
    &\!\!\!\!\leq\!\!\!\!&\frac{20\ol{\sigma}^2(\calX^\star)\|\calX-\calX^\star\|_F^2}{\underline{\sigma}^2(\calX^\star)},
\end{eqnarray}
where $R({\mX}_{N})$ is the right unfolding matrix of $\mX_{N}$ and the first inequality follows because $\|\mA_1 \ol \otimes L({\mX}_N)- \mA_2 \ol \otimes L({\mX}_N)\|_F = \|(\mA_1 - \mA_2) R({\mX}_{N}) \|_F \leq \|\mA_1 - \mA_2\|_F\|R({\mX}_{N})\| = \|\mA_1 - \mA_2\|_F\sigma_1(\calX^{\<N-1 \>})$ with $\mA_1 = L_{\mR}({\mX}_1^\star)\ol \otimes\cdots \ol \otimes L_{\mR}({\mX}^\star_{N-1})$ and $\mA_2 = L({\mX}_1)\ol \otimes\cdots\ol \otimes L({\mX}_{N-1})$.

Combining \eqref{LOWER BOUND OF TWO DISTANCES_5}, \eqref{LOWER BOUND OF TWO DISTANCES_6} and \eqref{LOWER BOUND OF TWO DISTANCES_7} together gives
\begin{eqnarray}
    \label{LOWER BOUND OF TWO DISTANCES_8}
    \text{dist}^2(\{\mX_i \},\{ \mX_i^\star \})\leq\frac{8(N+1+\sum_{i=2}^{N-1}r_i)\ol{\sigma}^2(\calX^\star)}{\underline{\sigma}^2(\calX^\star)}\|\calX-\calX^\star\|_F^2.
\end{eqnarray}

On the other hand, invoking {Lemma} \ref{EXPANSION_A1TOAN-B1TOBN_1} gives
\begin{eqnarray}
    \label{UPPER BOUND OF TWO DISTANCES_2}
    \|\mathcal{X}-\mathcal{X}^\star\|_F^2 &\!\!\!\!=\!\!\!\!& \|L({\bm X}_1^\star)\overline{\otimes} \cdots \overline{\otimes} L({\bm X}_{i-1}^\star) \overline{\otimes} (L({\bm X}_{i})-L({\bm X}_{i}^\star)) \overline{\otimes} L({\bm X}_{i+1})\overline{\otimes} \cdots \overline{\otimes} L({\bm X}_{N})\|_F^2\nonumber\\
    &\!\!\!\!\leq\!\!\!\!& N\bigg(\frac{9 \ol{\sigma}^2(\calX^\star)}{4}\sum_{i=1}^{N-1}\|L({\bm X}_i)-L({\bm X}_i^\star)\|_F^2+\|L({\bm X}_N)-L({\bm X}_N^\star)\|_F^2\bigg)\nonumber\\
    &\!\!\!\!\leq\!\!\!\!&\frac{9N}{4}\text{dist}^2(\{\mX_i \},\{ \mX_i^\star \}).
\end{eqnarray}
where the first inequality follows $\|L({\bm X}_1^\star)\overline{\otimes} \cdots \overline{\otimes} (L({\bm X}_{i})-L({\bm X}_{i}^\star)) \overline{\otimes} L({\bm X}_{i+1})\overline{\otimes} \cdots \overline{\otimes} L({\bm X}_{N})\|_F = \|L({\bm X}_1^\star)\overline{\otimes} \cdots \overline{\otimes} (L({\bm X}_{i})-L({\bm X}_{i}^\star))  {\calX}^{\geq i+1}  \|_F\leq \|L({\bm X}_1^\star)\overline{\otimes} \cdots \overline{\otimes} (L({\bm X}_{i})-L({\bm X}_{i}^\star))\|_F \| {\calX}^{\geq i+1}\|\leq \|L({\bm X}_1^\star)\|\cdots \|L({\bm X}_{i-1}^\star)\| \|L({\bm X}_{i})-L({\bm X}_{i}^\star)\|_F \ol{\sigma}(\calX)$.
\end{proof}

\begin{lemma}(\citep[Lemma 1]{LiSIAM21})
\label{NONEXPANSIVENESS PROPERTY OF POLAR RETRACTION_1}
Let ${\mX}\in\text{St}(n,r)$ and ${\bm \xi}\in\text{T}_{\mX} \text{St}$ be given. Consider the point ${\mX}^+={\mX}+{\bm \xi}$. Then, the polar decomposition-based retraction $\text{Retr}_{\mX}({\mX}^+)={\mX}^+({{\mX}^+}^\top{\mX}^+)^{-\frac{1}{2}}$ satisfies
\begin{eqnarray}
    \label{NONEXPANSIVENESS PROPERTY OF POLAR RETRACTION_2}
    \|\text{Retr}_{\mX}(\mX^+)-\overline{\mX}\|_F\leq\|{\mX}^+-\overline{\mX}\|_F=\|{\mX}+{\bm \xi}-\overline{\mX}\|_F, \  \forall\overline{\mX}\in\text{St}(n,r).
\end{eqnarray}
\end{lemma}

\section{Proof of \Cref{Local Convergence of Stiefel_Theorem} in Tensor-train Factorization}
\label{Local Convergence Proof of Riemannman gradient descent}

\begin{proof}
Before proving \Cref{Local Convergence of Stiefel_Theorem}, we first present a useful property for the factors $L(\mX_i^{(t)})$. Due to the retraction, the factors $L(\mX_i^{(t)}), i\in[N-1]$ are always orthonormal. Assuming that \begin{align}
\text{dist}^2(\{\mX_i^{(t)} \},\{ \mX_i^\star \})\leq \frac{\underline{\sigma}^2(\calX^\star)}{72(N^2-1)(N+1+\sum_{i=2}^{N-1}r_i)},
\label{eq:ini-cond-for-XN-factorization}\end{align}
which is true for $t = 0$ and will be proved later for $t\ge 1$,
we obtain that
\begin{eqnarray}
\label{upper bound TT spctral norm}
    \sigma_1^2({\calX^{(t)}}^{\<i \>}) &\!\!\!\!= \!\!\!\!&  \|{\calX^{(t)}}^{\geq i+1} \|^2 \leq \min_{\mR_i\in\O^{r_i\times r_i}}2\|\mR_{i}^\top{\calX^{\star}}^{\geq i+1} \|^2 + 2\|{\calX^{(t)}}^{\geq i+1} - \mR_{i}^\top{\calX^{\star}}^{\geq i+1} \|^2\nonumber\\
    &\!\!\!\!\leq \!\!\!\!&2\ol{\sigma}^2(\calX^\star) + \min_{\mR_i\in\O^{r_i\times r_i}}2 \|{\calX^{(t)}}^{\<i \>} - {\calX^\star}^{\<i \>} + {\calX^\star}^{\<i \>} - {\calX^{(t)}}^{\leq i}\mR_{i}^\top{\calX^{\star}}^{\geq i+1}  \|^2\nonumber\\
    &\!\!\!\!\leq \!\!\!\!&  2\ol{\sigma}^2(\calX^\star) + 4\|\calX^{(t)} - \calX^\star \|_F^2 + \min_{\mR_i\in\O^{r_i\times r_i}}4 \|\mR_{i}^\top{\calX^{\star}}^{\geq i+1} \|^2 \|{\calX^{(t)}}^{\leq i} - {\calX^\star}^{\leq i} \mR_i \|_F^2\nonumber\\
    &\!\!\!\!\leq \!\!\!\!& 2\ol{\sigma}^2(\calX^\star) + \bigg(4 + \frac{16\ol{\sigma}^2(\calX^\star)}{\underline{\sigma}^2(\calX^\star)}\bigg)\|\calX^{(t)}-\calX^\star\|_F^2\nonumber\\
    &\!\!\!\!\leq \!\!\!\!& 2\ol{\sigma}^2(\calX^\star) + \frac{45N\ol{\sigma}^2(\calX^\star)}{\underline{\sigma}^2(\calX^\star)} \text{dist}^2(\{\mX_i^{(t)} \},\{ \mX_i^\star \})\leq \frac{9\ol{\sigma}^2(\calX^\star)}{4}, i\in[N-1].
\end{eqnarray}
where each element of ${\calX^{(t)}}^{\leq i}$ and ${\calX^{(t)}}^{\geq i+1} $ are respectively defined in \eqref{each row of left part} and \eqref{each column of right part}. The fourth and last lines respectively follow \eqref{relationship between left O with  tensor SVD 1} and \eqref{UPPER BOUND OF TWO DISTANCES_1}. Note that  $\ol{\sigma}^2(\calX^{(t)}) = \max_{i=1}^{N-1}\sigma_1^2({\calX^{(t)}}^{\<i \>})\leq \frac{9\ol{\sigma}^2(\calX^\star)}{4}$.

We now prove the decay of $\text{dist}^2(\{\mX_i^{(t+1)} \},\{ \mX_i^\star \})$. First recall from~\eqref{expansion of distance in tensor factorization-main paper} that
\begin{eqnarray}
    \label{expansion of distance in tensor factorization}
&\!\!\!\!\!\!\!\!&\hspace{-1.1cm}\text{dist}^2(\{\mX_i^{(t+1)} \},\{ \mX_i^\star \})\nonumber\\
&\!\!\!\!\!\!\!\!&\hspace{-1.5cm}\leq\text{dist}^2(\{\mX_i^{(t)} \},\{ \mX_i^\star \})-2\mu\sum_{i=1}^{N} \bigg\< L(\mX_i^{(t)})-L_{\mR^{(t)}}(\mX_i^\star),\calP_{\text{T}_{L({\mX}_i)} \text{St}}\bigg(\nabla_{L({\mX}_{i})}f(\mX_1^{(t)}, \dots, \mX_N^{(t)})\bigg)\bigg\>\nonumber\\
&\!\!\!\!\!\!\!\!&\hspace{-1.5cm}+\mu^2\bigg(\frac{1}{\ol{\sigma}^2(\calX^\star)}\sum_{i=1}^{N-1}\bigg\|\calP_{\text{T}_{L({\mX}_i)} \text{St}}\bigg(\nabla_{L({\mX}_{i})}f(\mX_1^{(t)}, \dots, \mX_N^{(t)})\bigg)\bigg\|_F^2+\bigg\|\nabla_{L({\mX}_{N})}f(\mX_1^{(t)}, \dots, \mX_N^{(t)}) \bigg\|_2^2\bigg),\nonumber\\
\end{eqnarray}
where to unify the notation for all $i$, we define the projection onto the tangent space for $i = N$ as $\calP_{\text{T}_{L({\mX}_N)} \text{St}}=\calI$ since there is no constraint on the $N$-th factor. Note that the gradient $\nabla_{L({\mX}_{i})}f(\mX_1^{(t)}, \dots, \mX_N^{(t)})$ is defined as
\begin{eqnarray}
    \label{RIEMANNIAN_GRADIENT_1-1}
    \nabla_{L({\mX}_{i})}f(\mX_1^{(t)}, \dots, \mX_N^{(t)}) = \begin{bmatrix}\nabla_{\mX_{i}(1)} f(\mX_1^{(t)}, \dots, \mX_N^{(t)})\\ \vdots \\ \nabla_{\mX_{i}(d_i)} f(\mX_1^{(t)}, \dots, \mX_N^{(t)})  \end{bmatrix},
\end{eqnarray}
where the gradient with respect to each factor $\mX_{i}(s_i)$ can be computed as
\begin{align*}
\nabla_{\mX_{i}(s_i)}f(\mX_1^{(t)}, \dots, \mX_N^{(t)})
=&\sum_{s_1,\ldots,s_{i-1},s_{i+1},\ldots,s_N } \Big( \big(\calX^{(t)}(s_1, \dots,s_N)-\calX^\star(s_1,\dots,s_N) \big)\cdot \\
&\hspace{0.5cm}
\mX_{i-1}^{(t)}(s_{i-1})^\top \cdots
     \mX_{1}^{(t)}(s_{1})^\top \mX_{N}^{(t)}(s_{N})^\top   \cdots  \mX_{i+1}^{(t)}(s_{i+1})^\top\Big).
\end{align*}
Note that computing this gradient only requires $\calX^\star$ and does not rely on the knowledge of the factors in a TT decomposition of $\calX^\star$.

The following is to bound the second and third terms in \eqref{expansion of distance in tensor factorization}, respectively.

\paragraph{Upper bound of the third term in \eqref{expansion of distance in tensor factorization}}
We first define three matrices for $i\in[N]$ as follows:
\begin{eqnarray}
    \label{The definition of D1}
    \mD_1(i) &\!\!\!\!=\!\!\!\!& \begin{bmatrix} \mX_{i-1}^{(t)}(1)^\top\!\!\cdots\!\mX_{1}^{(t)}(1)^\top \ \ \ \   \cdots \ \ \ \   \mX_{i-1}^{(t)}(d_{i-1})^\top\!\!\cdots\!\mX_{1}^{(t)}(d_{1})^\top   \end{bmatrix}\nonumber\\
    &\!\!\!\!=\!\!\!\!&L^\top(\mX_{i-1}^{(t)})\ol \otimes \cdots \ol \otimes L^\top(\mX_1^{(t)})\in\R^{r_i\times(d_1\cdots d_{i-1})},\\
    \mD_2(i) &\!\!\!\!=\!\!\!\!& \begin{bmatrix} \mX_{N}^{(t)}(1)^\top\cdots \mX_{i+1}^{(t)}(1)^\top\\ \vdots \\ \mX_{N}^{(t)}(d_N)^\top\cdots \mX_{i+1}^{(t)}(d_{i+1})^\top \end{bmatrix}\in\R^{(d_{i+1}\cdots d_N )\times r_i },
\end{eqnarray}
where we note that $\mD_1(1) = 1$ and $\mD_2(N)=1$.
Moreover, for each $s_i\in [d_i]$, we define matrix $\mE(s_i)\in\R^{(d_1\cdots d_{i-1})\times (d_{i+1}\cdots d_N)}$ whose $(s_1\cdots s_{i-1}, s_{i+1}\cdots s_N)$-th element  is given by $$\mE(s_i)(s_1\cdots s_{i-1}, s_{i+1}\cdots s_N) =  \mX^{(t)}(s_1,\dots,s_i,\dots,s_N) - \mX^\star(s_1,\dots,s_i,\dots,s_N).$$

Based on the aforementioned notations, we bound $\left\| \nabla_{L({\mX}_{i})}f(\mX_1^{(t)}, \dots, \mX_N^{(t)})\right\|_F^2$ by
\begin{eqnarray}
    \label{PROJECTED GRADIENT DESCENT SQUARED TERM 1 to N}
    \bigg\| \nabla_{L({\mX}_{i})}f(\mX_1^{(t)}, \dots, \mX_N^{(t)})\bigg\|_F^2&\!\!\!\!=\!\!\!\!& \sum_{s_i=1}^{d_i}\bigg\| \nabla_{\mX_{i}(s_i)} f(\mX_1^{(t)}, \dots, \mX_N^{(t)})\bigg\|_F^2\nonumber\\
    &\!\!\!\!=\!\!\!\!&\sum_{s_i=1}^{d_i}\|\mD_1(i) \mE(s_i) \mD_2(i)   \|_F^2\nonumber\\
&\!\!\!\!\leq\!\!\!\!& \sum_{s_i=1}^{d_i}\|L^\top(\mX_{i-1}^{(t)})\ol \otimes \cdots \ol \otimes L^\top(\mX_1^{(t)})\|^2\| \mD_2(i)\|^2 \|\mE(s_i)\|_F^2\nonumber\\
    &\!\!\!\!\leq\!\!\!\!&\|L(\mX_1^{(t)})\|^2\cdots\|L(\mX_{i-1}^{(t)})\|^2 \|{\calX^{(t)}}^{\geq i+1} \|^2 \|\calX^{(t)}-\calX^\star\|_F^2\nonumber\\
    &\!\!\!\!\leq\!\!\!\!&\begin{cases}
    \frac{9\ol{\sigma}^2(\calX^\star)}{4}\|\calX^{(t)}-\calX^\star\|_F^2, & i \in [N-1],\\
    \|\calX^{(t)}-\calX^\star\|_F^2, & i = N,
  \end{cases}
\end{eqnarray}
where we use \eqref{KRONECKER PRODUCT VECTORIZATION11 - 2}, $\|\mD_2(i)\| = \|(\mD_2(i))^\top\| = \|{\calX^{(t)}}^{\geq i+1} \| = \sigma_1({\calX^{(t)}}^{\<i \>})\leq \frac{3 \ol{\sigma}(\calX^\star)}{2}$ and $\sum_{s_i=1}^{d_i}\|\mE(s_i)\|_F^2 = \|\calX^{(t)}-\calX^\star\|_F^2$ in the second inequality.

Using \eqref{PROJECTED GRADIENT DESCENT SQUARED TERM 1 to N}, we  obtain an upper bound for the third term in \eqref{expansion of distance in tensor factorization} as
    \begin{eqnarray}
    \label{RIEMANNIAN FACTORIZATION SQUARED TERM UPPER BOUND}
    &\!\!\!\!\!\!\!\!&\frac{1}{\ol{\sigma}^2(\calX^\star)}\sum_{i=1}^{N-1}\bigg\|\calP_{\text{T}_{L({\mX}_i)} \text{St}}\bigg(\nabla_{L({\mX}_{i})}f(\mX_1^{(t)}, \dots, \mX_N^{(t)})\bigg)\bigg\|_F^2+\bigg\|\nabla_{L({\mX}_{N})}f(\mX_1^{(t)}, \dots, \mX_N^{(t)}) \bigg\|_2^2\nonumber\\
    &\!\!\!\!\leq\!\!\!\!&\frac{1}{\ol{\sigma}^2(\calX^\star)}\sum_{i=1}^{N-1}\bigg\|\nabla_{L({\mX}_{i})}f(\mX_1^{(t)}, \dots, \mX_N^{(t)})\bigg\|_F^2+\bigg\|\nabla_{L({\mX}_{N})}f(\mX_1^{(t)}, \dots, \mX_N^{(t)}) \bigg\|_2^2\nonumber\\
    &\!\!\!\!\leq\!\!\!\!&\frac{9N-5}{4}\|\calX^{(t)}-\calX^\star\|_F^2,
\end{eqnarray}
where the first inequality follows from the fact that for any matrix $\mB = \calP_{\text{T}_{L({\mX}_i)} \text{St}}(\mB) + \calP_{\text{T}_{L({\mX}_{i})} \text{St}}^{\perp}(\mB)$ where $\calP_{\text{T}_{L({\mX}_i)} \text{St}}(\mB)$ and $\calP_{\text{T}_{L({\mX}_{i})} \text{St}}^{\perp}(\mB)$ are orthogonal,  we have $\|\calP_{\text{T}_{L({\mX}_i)} \text{St}}(\mB)\|_F^2\leq \|\mB\|_F^2$.

\paragraph{Lower bound of the second term in \eqref{expansion of distance in tensor factorization}}
We first expand the second term in \eqref{expansion of distance in tensor factorization} as follows:
\begin{eqnarray}
    \label{RIEMANNIAN FACTORIZATION CROSS TERM LOWER BOUND_original}
    &\!\!\!\!\!\!\!\!&\sum_{i=1}^{N} \bigg\< L(\mX_i^{(t)})-L_{\mR^{(t)}}(\mX_i^\star),\calP_{\text{T}_{L({\mX}_i)} \text{St}}\bigg(\nabla_{L({\mX}_{i})}f(\mX_1^{(t)}, \dots, \mX_N^{(t)})\bigg)\bigg\>\nonumber\\
    &\!\!\!\!=\!\!\!\!&\sum_{i=1}^{N} \bigg\< L(\mX_i^{(t)})-L_{\mR^{(t)}}(\mX_i^\star), \nabla_{L({\mX}_{i})}f(\mX_1^{(t)}, \dots, \mX_N^{(t)})\bigg\> - T_1\nonumber\\
    &\!\!\!\!=\!\!\!\!&\bigg\<L(\mX_1^{(t)}) \ol \otimes \cdots \ol \otimes L(\mX_N^{(t)})-L_{\mR^{(t)}}({\mX}_1^\star) \ol \otimes \cdots \ol \otimes L_{\mR^{(t)}}({\mX}_N^\star),L(\mX_1^{(t)}) \ol \otimes \cdots \ol \otimes L(\mX_N^{(t)})\nonumber\\
    &\!\!\!\!\!\!\!\!& ~~ -L_{\mR^{(t)}}(\mX_1^\star) \ol \otimes \cdots \ol \otimes L_{\mR^{(t)}}(\mX_N^\star)+\vh^{(t)}\bigg\>-T_1,
\end{eqnarray}
where we define
\begin{eqnarray}
    \label{H_T IN THE CROSS TERM}
    &\!\!\!\!\!\!\!\!& \hspace{-1.3cm}\vh^{(t)}=L_{\mR^{(t)}}(\mX_1^\star)\ol \otimes  \cdots \ol \otimes L_{\mR^{(t)}}({\mX}_N^\star) - L(\mX_1^{(t)}) \ol \otimes \cdots \ol \otimes L(\mX_{N-1}^{(t)}) \ol \otimes L_{\mR^{(t)}}(\mX_{N}^\star)\nonumber\\
    &\!\!\!\!\!\!\!\!& \hspace{-0.6cm} +\sum_{i=1}^{N-1}L(\mX_1^{(t)})\ol \otimes\cdots\ol \otimes L(\mX_{i-1}^{(t)})\ol \otimes ( L({\mX}_i^{(t)}) -L_{\mR^{(t)}}({\mX}_i^\star) ) \ol \otimes L(\mX_{i+1}^{(t)})\ol \otimes \cdots \ol \otimes L(\mX_N^{(t)}),
\end{eqnarray}
and
    \begin{eqnarray}
        \label{The definition of T1}
        T_1=\sum_{i=1}^{N-1}\bigg\<\calP^{\perp}_{\text{T}_{L({\mX}_i)} \text{St}}(L(\mX_i^{(t)})-L_{\mR^{(t)}}(\mX_i^\star)), \nabla_{L({\mX}_{i})}f(\mX_1^{(t)}, \dots, \mX_N^{(t)}) \bigg\>.
    \end{eqnarray}

Recalling the definition for orthogonal complement projection  in \eqref{orthogonal complement to the Tangent space on the Stiefel manifold}, we can rewrite the term
$\calP^{\perp}_{\text{T}_{L({\mX}_i)} \text{St}}(\cdot)$ as
    \begin{eqnarray}
        \label{Projection orthogonal in the Riemannian gradient descent}
        &\!\!\!\!\!\!\!\!&\calP^{\perp}_{\text{T}_{L({\mX}_i)} \text{St}}(L(\mX_i^{(t)})-L_{\mR^{(t)}}({\mX}_i^\star))\nonumber\\
        &\!\!\!\!=\!\!\!\!&\frac{1}{2}L(\mX_i^{(t)})\bigg((L(\mX_i^{(t)})-L_{\mR^{(t)}}({\mX}_i^\star))^\top L(\mX_i^{(t)})
        +L^\top(\mX_i^{(t)})(L(\mX_i^{(t)})-L_{\mR^{(t)}}({\mX}_i^\star))\bigg)\nonumber\\
        &\!\!\!\!=\!\!\!\!&\frac{1}{2}L(\mX_i^{(t)})\bigg(2\mId_{r_i}-L_{\mR^{(t)}}^\top({\mX}_i^\star)L(\mX_i^{(t)})-L^\top(\mX_i^{(t)})L_{\mR^{(t)}}({\mX}_i^\star)\bigg)\nonumber\\
        &\!\!\!\!=\!\!\!\!&\frac{1}{2}L(\mX_i^{(t)})\bigg((L(\mX_i^{(t)})-L_{\mR^{(t)}}({\mX}_i^\star))^\top(L(\mX_i^{(t)})-L_{\mR^{(t)}}({\mX}_i^\star))\bigg).
    \end{eqnarray}

To derive the lower bound of \eqref{RIEMANNIAN FACTORIZATION CROSS TERM LOWER BOUND_original}, we first utilize  {Lemma} \ref{TRANSFORMATION OF NPLUS2 VARIABLES_1} to obtain the upper bound on $\|\vh^{(t)}\|_2^2$ as follows:
\begin{eqnarray}
    \label{H_T IN THE CROSS TERM UPPER BOUND}
    \|\vh^{(t)}\|_2^2&\!\!\!\!=\!\!\!\!&\|\sum_{i=1}^{N-1}\sum_{j=i+1}^N L(\mX_1^{(t)})\ol \otimes \cdots \ol \otimes L(\mX_{i-1}^{(t)})\ol \otimes (L(\mX_i^{(t)})-L_{\mR^{(t)}}(\mX_{i}^\star)) \ol \otimes L_{\mR^{(t)}}(\mX_{i+1}^\star) \ol \otimes \nonumber\\
    &\!\!\!\!\!\!\!\!&\cdots \ol \otimes L_{\mR}(\mX_{j-1}^\star)\ol \otimes (L(\mX_{j}^{(t)})-L_{\mR^{(t)}}(\mX_{j}^\star))\ol \otimes L(\mX_{j+1}^{(t)}) \ol \otimes \cdots \ol \otimes L(\mX_N^{(t)}) \|_2^2\nonumber\\
    &\!\!\!\!\leq\!\!\!\!&\frac{N(N-1)}{2}\bigg(\sum_{j=1}^{N-1}\|L(\mX_{j}^{(t)})-L_{\mR^{(t)}}({\mX_j}^\star)\|_F^2 \cdot\nonumber\\
    &\!\!\!\!\!\!\!\!&\bigg(\sum_{i=j+1}^{N-1} \frac{9\ol{\sigma}^2(\calX^\star)}{4}\|L(\mX_i^{(t)})-L_{\mR^{(t)}}({\mX}_i^\star)\|_F^2 + \|L(\mX_N^{(t)})-L_{\mR^{(t)}}(\mX_N^\star)\|_2^2\bigg)\bigg)\nonumber
    \end{eqnarray}
    \begin{eqnarray}
    &\!\!\!\!\leq\!\!\!\!&\frac{9N(N-1)}{8}\sum_{i=1}^{N-1}\|L(\mX_i^{(t)})-L_{\mR^{(t)}}(\mX_i^\star)\|_F^2\text{dist}^2(\{\mX_i^{(t)} \},\{ \mX_i^\star \})\nonumber\\
    &\!\!\!\!\leq\!\!\!\!&\frac{9N(N-1)}{8\ol{\sigma}^2(\calX^\star)}\text{dist}^4(\{\mX_i^{(t)} \},\{ \mX_i^\star \}),
\end{eqnarray}
where \eqref{KRONECKER PRODUCT VECTORIZATION11 - 3} and $\|\mA\ol \otimes L(\mX_{j+1}^{(t)}) \ol \otimes \cdots \ol \otimes L(\mX_N^{(t)}) \|_2 = \|\mA {\calX^{(t)}}^{\geq j+1}\|_F\leq \|{\calX^{(t)}}^{\geq j+1}\|\|\mA\|_F\leq \\ \frac{3 \ol{\sigma}(\calX^\star)}{2}\|\mA\|_F$ with $\mA = L(\mX_1^{(t)})\ol \otimes \cdots \ol \otimes L(\mX_{i-1}^{(t)})\ol \otimes (L(\mX_i^{(t)})-L_{\mR^{(t)}}(\mX_{i}^\star)) \ol \otimes L_{\mR^{(t)}}(\mX_{i+1}^\star) \ol \otimes \cdots \\ \ol \otimes L_{\mR}(\mX_{j-1}^\star)\ol \otimes (L(\mX_{j}^{(t)})-L_{\mR^{(t)}}(\mX_{j}^\star))$ are used in the first inequality. We then establish the upper bound of $T_1$ as follows:
    \begin{eqnarray}
        \label{PROJECTION ORTHOGONAL IN RIEMANNIAN UPPER BOUND}
    T_1&\!\!\!\!\leq\!\!\!\!&\sum_{i=1}^{N-1}\frac{1}{2}\|L(\mX_i^{(t)})\|\|L(\mX_i^{(t)})-L_{\mR^{(t)}}({\mX}_i^\star)\|_F^2\bigg|\bigg|\nabla_{L({\mX}_{i})}f(\mX_1^{(t)}, \dots, \mX_N^{(t)})\bigg|\bigg|_F\nonumber\\
        &\!\!\!\!\leq\!\!\!\!&\sum_{i=1}^{N-1}\frac{3\ol{\sigma}(\calX^\star)}{4}\|L(\mX_i^{(t)})-L_{\mR^{(t)}}({\mX}_i^\star)\|_F^2\|\calX^{(t)}-\calX^\star\|_F\nonumber\\
        &\!\!\!\!\leq\!\!\!\!&\frac{1}{4}\|\calX^{(t)}-\calX^\star\|_F^2
        +\frac{9(N-1)\ol{\sigma}^2(\calX^\star)}{16}\sum_{i=1}^{N-1}\|L(\mX_i^{(t)})-L_{\mR^{(t)}}({\mX}_i^\star)\|_F^4\nonumber\\
        &\!\!\!\!\leq\!\!\!\!&\frac{1}{4}\|\calX^{(t)}-\calX^\star\|_F^2+\frac{9(N-1)}{16\ol{\sigma}^2(\calX^\star)}\text{dist}^4(\{\mX_i^{(t)} \},\{ \mX_i^\star \}),
    \end{eqnarray}
    where the second inequality follows from \eqref{PROJECTED GRADIENT DESCENT SQUARED TERM 1 to N}.

Now plugging \eqref{H_T IN THE CROSS TERM UPPER BOUND} and \eqref{PROJECTION ORTHOGONAL IN RIEMANNIAN UPPER BOUND} into~\eqref{RIEMANNIAN FACTORIZATION CROSS TERM LOWER BOUND_original} gives
\begin{eqnarray}
    \label{RIEMANNIAN FACTORIZATION CROSS TERM LOWER BOUND}
    &\!\!\!\!\!\!\!\!&\sum_{i=1}^{N} \bigg\< L(\mX_i^{(t)})-L_{\mR^{(t)}}(\mX_i^\star),\calP_{\text{T}_{L({\mX}_i)} \text{St}}\bigg(\nabla_{L({\mX}_{i})}f(\mX_1^{(t)}, \dots, \mX_N^{(t)})\bigg)\bigg\>\nonumber\\
    &\!\!\!\!\geq\!\!\!\!&\frac{1}{2}\|\calX^{(t)}-\calX^\star\|_F^2-\frac{1}{2}\|\vh^{(t)}\|_2^2-\frac{1}{4}\|\calX^{(t)}-\calX^\star\|_F^2-\frac{9(N-1)}{16\ol{\sigma}^2(\calX^\star)}
    \text{dist}^4(\{\mX_i^{(t)} \},\{ \mX_i^\star \})\nonumber\\
    &\!\!\!\!\geq\!\!\!\!&\frac{1}{4}\|\calX^{(t)}-\calX^\star\|_F^2-\frac{9N(N-1)}{16\ol{\sigma}^2(\calX^\star)}\text{dist}^4(\{\mX_i^{(t)} \},\{ \mX_i^\star \})-\frac{9(N-1)}{16\ol{\sigma}^2(\calX^\star)}
    \text{dist}^4(\{\mX_i^{(t)} \},\{ \mX_i^\star \})\nonumber\\
    &\!\!\!\!\geq\!\!\!\!&\frac{\underline{\sigma}^2(\calX^\star)}{128(N+1+\sum_{i=2}^{N-1}r_i)\ol{\sigma}^2(\calX^\star)}\text{dist}^2(\{\mX_i^{(t)} \},\{ \mX_i^\star \}) + \frac{1}{8}\|\calX^{(t)}-\calX^\star\|_F^2,
\end{eqnarray}
where the first and second inequalities follow from~\eqref{PROJECTION ORTHOGONAL IN RIEMANNIAN UPPER BOUND} and \eqref{H_T IN THE CROSS TERM UPPER BOUND}, and
we utilize {Lemma} \ref{LOWER BOUND OF TWO DISTANCES} along with  the initial condition $\text{dist}^2(\{\mX_i^{(0)} \},\{ \mX_i^\star \})\leq \frac{\underline{\sigma}^2(\calX^\star)}{72(N^2-1)(N+1+\sum_{i=2}^{N-1}r_i)}$ in the last line.

\paragraph{Contraction}

Taking \eqref{RIEMANNIAN FACTORIZATION SQUARED TERM UPPER BOUND} and \eqref{RIEMANNIAN FACTORIZATION CROSS TERM LOWER BOUND} into \eqref{expansion of distance in tensor factorization}, we have
\begin{eqnarray}
    \label{LGNWTFSM5_1}
    \text{dist}^2(\{\mX_i^{(t+1)} \},\{ \mX_i^\star \})&\!\!\!\!\leq\!\!\!\!&\bigg(1-\frac{\underline{\sigma}^2(\calX^\star)}{64(N+1+\sum_{i=2}^{N-1}r_i)\ol{\sigma}^2(\calX^\star)}\mu\bigg)\text{dist}^2(\{\mX_i^{(t)} \},\{ \mX_i^\star \})\nonumber\\
    &\!\!\!\!\!\!\!\!& +\bigg(\frac{9N-5}{4}\mu^2- \frac{\mu}{4}\bigg)\|\calX^{(t)}-\calX^\star\|_F^2\nonumber\\
    &\!\!\!\!\leq\!\!\!\!&\bigg(1-\frac{\underline{\sigma}^2(\calX^\star)}{64(N+1+\sum_{i=2}^{N-1}r_i)\ol{\sigma}^2(\calX^\star)}\mu\bigg)\text{dist}^2(\{\mX_i^{(t)} \},\{ \mX_i^\star \}),
\end{eqnarray}
where we assume $\mu\leq\frac{1}{9N-5}$ in the last line.

\paragraph{Proof of \eqref{eq:ini-cond-for-XN-factorization}} We now prove \eqref{eq:ini-cond-for-XN-factorization} by induction. First note that \eqref{eq:ini-cond-for-XN-factorization} holds for $t = 0$. We now assume it holds at $t = t'$, which implies that $\sigma_1^2({\calX^{(t')}}^{\<i \>})  \leq\frac{9\ol{\sigma}^2(\calX^\star)}{4}, i\in[N-1]$. By invoking \eqref{LGNWTFSM5_1}, we have $\text{dist}^2(\{\mX_i^{(t'+1)}\}, \{\mX_i^\star \}) \le \text{dist}^2(\{\mX_i^{(t)} \},\{ \mX_i^\star \})$. Consequently, \eqref{eq:ini-cond-for-XN-factorization} also holds  at $ t= t'+1$. By induction, we can conclude that \eqref{eq:ini-cond-for-XN-factorization} holds for all $t\ge 0$. This completes the proof.

\end{proof}

\section{Proof of \Cref{TENSOR SENSING SPECTRAL INITIALIZATION} in Spectral Initialization}
\label{Proof of in the spectral initialization}

We first provide one useful lemma. As an immediate consequence of the RIP, the inner product between two low-rank TT format tensors is also nearly preserved if $\calA$ satisfies the RIP.
\begin{lemma} (\citep{CandsTIT11,Rauhut17} )
\label{RIP CONDITION FRO THE TENSOR TRAIN SENSING OTHER PROPERTY}
Suppose that $\calA$ obeys the $2\ol r$-RIP with a constant $\delta_{2\ol r}$. Then for any left-orthogonal TT formats $\calX_1,\calX_2\in\R^{d_{1} \times \cdots \times d_{N}}$ of rank at most $\ol r$, one has
\begin{eqnarray}
    \label{RIP CONDITION FRO THE TENSOR TRAIN SENSING OTHER PROPERTY_1}
    \bigg|\frac{1}{m}\<\calA(\calX_1),\calA(\calX_2)\>-\<\calX_1,\calX_2\>\bigg|\leq \delta_{2\ol r}\|\calX_1\|_F\|\calX_2\|_F,
\end{eqnarray}
or equivalently,
\begin{eqnarray}
    \label{RIP CONDITION FRO THE TENSOR TRAIN SENSING OTHER PROPERTY_1_1}
    \bigg|\bigg\<\bigg(\frac{1}{m}\calA^*\calA-\mathcal{I}\bigg)(\calX_1), \calX_2\bigg\>\bigg|\leq \delta_{2\ol r}\|\calX_1\|_F\|\calX_2\|_F,
\end{eqnarray}
where $\calA^*$ is the adjoint operator of $\calA$ and is defined as $\calA^*({\vx})=\sum_{i=1}^m x_i\calA_i$.
\end{lemma}

\begin{proof}[Proof of \Cref{TENSOR SENSING SPECTRAL INITIALIZATION}]
Before analyzing the spectral initialization, we first define the following restricted Frobenius norm for any tensor $\calH\in\R^{d_1\times\cdots\times d_N}$:
\begin{eqnarray}
    \label{Definition of the restricted F norm}
    \|\calH\|_{F,\ol r} &=&\max_{i \in [N-1]}\sqrt{\sum_{j=1}^{r_i}\sigma_j^2(\calH^{\<i \>})} \nonumber\\
    &=& \max_{\mV_i\in\R^{d_{i+1}\cdots d_N\times r_i}, \atop \mV_i\mV_i^\top = \mId_{r_i}, i\in [N-1]}\|\calH^{\< i\>} \mV_i \|_F   \nonumber\\
    &=& \max_{\calX\in\R^{d_1\times\cdots\times d_N}, \|\calX\|_F\leq 1, \atop {\rm rank}(\calX)=(r_1,\dots,r_{N-1}) } \<\calH,  \calX \>,
\end{eqnarray}
where $\rank(\calX)$ denotes the TT ranks of $\calX$.
Similar forms for the matrix case are provided in \citep{zhang2021preconditioned,tong2022scaled}.
We now upper bound $\|\calX^{(0)}-\calX^\star\|_F$ as
\begin{eqnarray}
    \label{LGNWTSSM6_2}
    &\!\!\!\!\!\!\!\!&\|\calX^{(0)}-\calX^\star\|_F\nonumber\\
    &\!\!\!\!=\!\!\!\!&\bigg\|\text{SVD}_{\vr}^{tt}\bigg(\frac{1}{m}\sum_{k=1}^my_k\calA_k\bigg)-\calX^\star\bigg\|_{F,2\ol r}\nonumber\\
    &\!\!\!\!\leq\!\!\!\!&\bigg\|\text{SVD}_{\vr}^{tt}\bigg(\frac{1}{m}\sum_{k=1}^my_k\calA_k\bigg)-\frac{1}{m}\sum_{k=1}^my_k\calA_k\bigg\|_{F,2\ol r}
    +\bigg\|\frac{1}{m}\sum_{k=1}^my_k\calA_k-\calX^\star\bigg\|_{F,2\ol r}\nonumber\\
    &\!\!\!\!\leq\!\!\!\!&\sqrt{N-1}\bigg\|\text{opt}_{\vr}\bigg(\frac{1}{m}\sum_{k=1}^my_k\calA_k\bigg)-\frac{1}{m}\sum_{k=1}^my_k\calA_k\bigg\|_{F,2\ol r}+
    \bigg\|\frac{1}{m}\sum_{k=1}^my_k\calA_k-\calX^\star\bigg\|_{F,2\ol r}\nonumber
    \end{eqnarray}
    \begin{eqnarray}
    &\!\!\!\!\leq\!\!\!\!&(1+\sqrt{N-1})\bigg\|\frac{1}{m}\sum_{k=1}^my_k\calA_k-\calX^\star\bigg\|_{F,2\ol r}\nonumber\\
    &\!\!\!\!=\!\!\!\!&(1+\sqrt{N-1})\max_{\calZ\in\R^{d_1\times\cdots\times d_N}, \|\calZ\|_F\leq 1, \atop {\rm rank}(\calZ)=(2r_1,\dots,2r_{N-1}) }\bigg|\bigg\<\bigg(\frac{1}{m}\calA^*\calA-\mathcal{I}\bigg)(\calX^\star), \calZ  \bigg\>\bigg|\nonumber\\
    &\!\!\!\!\leq\!\!\!\!& \delta_{3\ol r}(1+\sqrt{N-1})\|\calX^\star\|_F,
\end{eqnarray}
where $\text{opt}_{\vr}(\frac{1}{m}\sum_{k=1}^my_k\calA_k)$ is the best TT-approximation of ranks $\vr$ to $\frac{1}{m}\sum_{k=1}^my_k\calA_k$ in the Frobenius norm, the second inequality utilizes the quasi-optimality property of TT-SVD projection \citep{Oseledets11}, the third inequality follows because the definition of $\text{opt}_{\vr}(\cdot)$ and $\calX^\star$ has ranks $\vr$, and the last uses \eqref{RIP CONDITION FRO THE TENSOR TRAIN SENSING OTHER PROPERTY_1_1}.

\end{proof}

\section{Proof of \Cref{Local Convergence of Riemannian in the sensing_Theorem} in Tensor-train Sensing}
\label{Local Convergence Proof of Riemannman gradient descent Tensor Sensing}

\begin{proof}
The proof follows a similar approach to that for \Cref{Local Convergence of Stiefel_Theorem} in {Appendix} \ref{Local Convergence Proof of Riemannman gradient descent}. We first present useful properties for the factors $L(\mX_i^{(t)})$. Due to the retraction, $L(\mX_i^{(t)}), i\in[N-1]$ are always orthonormal. Assuming that \begin{align}
\text{dist}^2(\{\mX_i^{(t)} \},\{ \mX_i^\star \})\leq \frac{(4-15\delta_{(N+3)\ol r})\underline{\sigma}^2(\calX^\star)}{8(N+1+\sum_{i=2}^{N-1}r_i)(57N^2+393N-450)},
\label{eq:ini-cond-for-XN}\end{align}
which is true for $t = 0$ and will be proved later for $t\ge 1$, based on the derivation of \eqref{upper bound TT spctral norm}, we have
\begin{eqnarray}
\label{upper bound TT spctral norm sensing}
    \sigma_1^2({\calX^{(t)}}^{\<i \>})  &\!\!\!\!= \!\!\!\!& \|{\calX^{(t)}}^{\geq i+1} \|^2 \leq 2\ol{\sigma}^2(\calX^\star) + \bigg(4 + \frac{16\ol{\sigma}^2(\calX^\star)}{\underline{\sigma}^2(\calX^\star)}\bigg)\|\calX^{(t)}-\calX^\star\|_F^2\nonumber\\
    &\!\!\!\!\leq \!\!\!\!& 2\ol{\sigma}^2(\calX^\star) + \frac{45N\ol{\sigma}^2(\calX^\star)}{\underline{\sigma}^2(\calX^\star)} \text{dist}^2(\{\mX_i^{(t)} \},\{ \mX_i^\star \})\leq \frac{9\ol{\sigma}^2(\calX^\star)}{4}, i\in[N-1].
\end{eqnarray}
Note that  $\ol{\sigma}^2(\calX^{(t)}) = \max_{i=1}^{N-1}\sigma_1^2({\calX^{(t)}}^{\<i \>})\leq \frac{9\ol{\sigma}^2(\calX^\star)}{4}$.

We now prove the decay of the distance based on the result $\|{\calX^{(t)}}^{\geq i+1} \|^2\leq \frac{9\ol{\sigma}^2(\calX^\star)}{4}, i\in[N-1]$.
First recall  \eqref{expansion of distance in tensor factorization-main paper}:
\begin{eqnarray}
    \label{Expansion of distance in the noiseless tensor sensing}
    &\!\!\!\!\!\!\!\!&\hspace{-1.1cm}\text{dist}^2(\{\mX_i^{(t+1)} \},\{ \mX_i^\star \})\nonumber\\
    &\!\!\!\!\!\!\!\!&\hspace{-1.5cm}\leq\text{dist}^2(\{\mX_i^{(t)} \},\{ \mX_i^\star \})-2\mu\sum_{i=1}^{N} \bigg\< L(\mX_i^{(t)})-L_{\mR^{(t)}}(\mX_i^\star),\calP_{\text{T}_{L({\mX}_i)} \text{St}}\bigg(\nabla_{L({\mX}_{i})} g(\mX_1^{(t)}, \dots, \mX_N^{(t)})\bigg)\bigg\>\nonumber\\
    &\!\!\!\!\!\!\!\!&\hspace{-1.5cm}+\mu^2\bigg(\frac{1}{\ol{\sigma}^2(\calX^\star)}\sum_{i=1}^{N-1}\bigg\|\calP_{\text{T}_{L({\mX}_i)} \text{St}}\bigg(\nabla_{L({\mX}_{i})} g(\mX_1^{(t)}, \dots, \mX_N^{(t)})\bigg)\bigg\|_F^2+\bigg\|\nabla_{L({\mX}_{N})} g(\mX_1^{(t)}, \dots, \mX_N^{(t)}) \bigg\|_2^2\bigg).\nonumber\\
\end{eqnarray}
Note that the gradient is defined as
\begin{eqnarray}
    \label{The definition of gradient in the tensor sensing noiseless}
    \nabla_{L({\mX}_{i})} g(\mX_1^{(t)}, \dots, \mX_N^{(t)}) = \begin{bmatrix}\nabla_{\mX_{i}(1)} g(\mX_1^{(t)}, \dots, \mX_N^{(t)})\\ \vdots \\ \nabla_{\mX_{i}(d_i)} g(\mX_1^{(t)}, \dots, \mX_N^{(t)})  \end{bmatrix},
\end{eqnarray}
where the gradient with respect to each factor $\mX_{i}(s_i)$ can be obtained as
\begin{align*}
\nabla_{\mX_i(s_i)}g(\mX_1^{(t)}, \dots, \mX_N^{(t)})
=\frac{1}{m}\sum_{k=1}^{m} &(\<\calA_k,\calX^{(t)}\>-y_k) \sum_{s_1,\ldots,s_{i-1},s_{i+1},\ldots,s_N } \Big( \calA_k(s_1,\dots,s_N)\cdot \\
&
\mX_{i-1}^{(t)}(s_{i-1})^\top \cdots
     \mX_{1}^{(t)}(s_{1})^\top \mX_{N}^{(t)}(s_{N})^\top   \cdots  \mX_{i+1}^{(t)}(s_{i+1})^\top\Big).
\end{align*}

\paragraph{Upper bound of the third term in \eqref{Expansion of distance in the noiseless tensor sensing}} Using the RIP, we begin by quantifying the difference in the gradients of $g$ and $f$ through
\begin{eqnarray}
    \label{Riemannian GRADIENT DESCENT SQUARED TERM 1 to N noiseless sensing}
    &\!\!\!\!\!\!\!\!&\bigg\| \nabla_{L({\mX}_{i})} g(\mX_1^{(t)}, \dots, \mX_N^{(t)}) - \nabla_{L({\mX}_{i})} f(\mX_1^{(t)}, \dots, \mX_N^{(t)}) \bigg\|_F\nonumber\\
    &\!\!\!\!=\!\!\!\!&\max_{\mH_i\in\R^{r_{i-1}\times d_i\times r_i} \atop \|\mH_i\|_F\leq 1}\<\nabla_{L({\mX}_{i})} g(\mX_1^{(t)}, \dots, \mX_N^{(t)}) - \nabla_{L({\mX}_{i})} f(\mX_1^{(t)}, \dots, \mX_N^{(t)}), L(\mH_i)\>\nonumber\\  &\!\!\!\!=\!\!\!\!&\max_{\mH_i\in\R^{r_{i-1}\times d_i\times r_i} \atop \|\mH_i\|_F\leq 1}\<(\frac{1}{m}\calA^*\calA-\mathcal{I})(\calX^{(t)}-\calX^\star), [\mX_1^{(t)},\dots, \mH_i, \dots,\mX_N^{(t)} ]   \>\nonumber\\
    &\!\!\!\!\leq\!\!\!\!& \delta_{3\ol r}\|\calX^{(t)} - \calX^\star\|_F\|[\mX_1^{(t)},\dots, \mH_i, \dots,\mX_N^{(t)} ]\|_F\nonumber\\
    &\!\!\!\!\leq\!\!\!\!&\begin{cases}
    \frac{3\ol{\sigma}(\calX^\star)}{2}\delta_{3\ol r}\|\calX^{(t)}-\calX^\star\|_F, & i \in [N-1],\\
    \delta_{3\ol r}\|\calX^{(t)}-\calX^\star\|_F, & i = N,
  \end{cases}
\end{eqnarray}
where the first inequality  follows \eqref{RIP CONDITION FRO THE TENSOR TRAIN SENSING OTHER PROPERTY_1_1} and the second inequality uses \eqref{KRONECKER PRODUCT VECTORIZATION11 - 3} and $\|[\mX_1^{(t)},\dots, \mH_i, \dots, \\ \mX_N^{(t)} ]\|_F = \|(L(\mX_1^{(t)})\ol \otimes \cdots \ol \otimes L(\mH_i)){\calX^{(t)}}^{\geq i+1} \|_F \leq \|L(\mX_1^{(t)})\ol \otimes \cdots \ol \otimes L(\mH_i)\|_F\|{\calX^{(t)}}^{\geq i+1}\|$. This together with the upper bound for $\bigg\| \nabla_{L({\mX}_{i})} f(\mX_1^{(t)}, \dots, \mX_N^{(t)})  \bigg\|_F$ in \eqref{PROJECTED GRADIENT DESCENT SQUARED TERM 1 to N} gives
\begin{eqnarray}
    \label{Riemannian GRADIENT DESCENT SQUARED TERM 1 to N noiseless sensing 1}
    \bigg\| \nabla_{L({\mX}_{i})} g(\mX_1^{(t)}, \dots, \mX_N^{(t)})  \bigg\|_F
    \leq\begin{cases}
    \frac{3\ol{\sigma}(\calX^\star)}{2}(1+\delta_{3\ol r})\|\calX^{(t)}-\calX^\star\|_F, & i \in [N-1],\\
    (1+\delta_{3\ol r})\|\calX^{(t)}-\calX^\star\|_F, & i = N.
  \end{cases}
\end{eqnarray}

Plugging this into the third term in \eqref{Expansion of distance in the noiseless tensor sensing} and following the same analysis of \eqref{RIEMANNIAN FACTORIZATION SQUARED TERM UPPER BOUND}, we can obtain
\begin{eqnarray}
    \label{UPPER BOUND OF SQUARED TERM OF Riemannian GD IN SENSING Conclusion}
    &\!\!\!\!\!\!\!\!&\frac{1}{\ol{\sigma}^2(\calX^\star)}\sum_{i=1}^{N-1}\bigg\|\calP_{\text{T}_{L({\mX}_i)} \text{St}}\bigg(\nabla_{L({\mX}_{i})} g(\mX_1^{(t)}, \dots, \mX_N^{(t)})\bigg)\bigg\|_F^2+\bigg\|\nabla_{L({\mX}_{N})} g(\mX_1^{(t)}, \dots, \mX_N^{(t)}) \bigg\|_2^2\nonumber\\
    &\!\!\!\!\leq\!\!\!\!&\frac{1}{\ol{\sigma}^2(\calX^\star)}\sum_{i=1}^{N-1}\bigg\|\nabla_{L({\mX}_{i})} g(\mX_1^{(t)}, \dots, \mX_N^{(t)})\bigg\|_F^2+\bigg\|\nabla_{L({\mX}_{N})} g(\mX_1^{(t)}, \dots, \mX_N^{(t)}) \bigg\|_2^2\nonumber\\
    &\!\!\!\!\leq\!\!\!\!&\frac{9N-5}{2}(1+\delta_{3\ol r})^2\|\calX^{(t)}-\calX^\star\|_F^2.
\end{eqnarray}

\paragraph{Lower bound of the second term in \eqref{Expansion of distance in the noiseless tensor sensing}}
We first expand the second term of \eqref{Expansion of distance in the noiseless tensor sensing} as follows:
\begin{eqnarray}
    \label{Lower BOUND OF cross TERM OF Riemannian GD IN SENSING Conclusion original}
    &\!\!\!\!\!\!\!\!&\sum_{i=1}^{N} \bigg\< L(\mX_i^{(t)})-L_{\mR^{(t)}}(\mX_i^\star),\calP_{\text{T}_{L({\mX}_i)} \text{St}}\bigg(\nabla_{L({\mX}_{i})} g(\mX_1^{(t)}, \dots, \mX_N^{(t)})\bigg)\bigg\>\nonumber\\
    &\!\!\!\!=\!\!\!\!&\sum_{i=1}^{N} \bigg\< L(\mX_i^{(t)})-L_{\mR^{(t)}}(\mX_i^\star), \nabla_{L({\mX}_{i})} g(\mX_1^{(t)}, \dots, \mX_N^{(t)})\bigg\> - T_2\nonumber\\
    &\!\!\!\!=\!\!\!\!&\frac{1}{m}\sum_{k=1}^{m}\<\va_k,L(\mX_1^{(t)}) \ol \otimes \cdots \ol \otimes L(\mX_N^{(t)})-L_{\mR^{(t)}}(\mX_1^\star) \ol \otimes \cdots \ol \otimes L_{\mR^{(t)}}({\mX}_N^\star)\>\<\va_k, \vh^{(t)}\> \nonumber\\
    &\!\!\!\!\!\!\!\!& + \frac{1}{m}\|\calA(\calX^{(t)} - \calX^\star)\|_2^2 - T_2,
\end{eqnarray}
where $\va_k=\text{vec}(\calA_k)$ and  $T_2$ is defined as
\begin{eqnarray}
    \label{PROJECTION ORTHOGONAL IN RIEMANNIAN UPPER BOUND in the sensing formula}
    T_2=\sum_{i=1}^{N-1}\bigg\<\calP^{\perp}_{\text{T}_{L({\mX}_i)} \text{St}}(L(\mX_i^{(t)})-L_{\mR^{(t)}}(\mX_i^\star)), \nabla_{L({\mX}_{i})} g(\mX_1^{(t)}, \dots, \mX_N^{(t)}) \bigg\>.
\end{eqnarray}
$T_2$ can be upper bonded by
\begin{eqnarray}
    \label{PROJECTION ORTHOGONAL IN RIEMANNIAN UPPER BOUND in the sensing}
    T_2&\!\!\!\!\leq\!\!\!\!&\frac{1}{2}\sum_{i=1}^{N-1}\|L(\mX_i^{(t)})\|\|L(\mX_i^{(t)})-L_{\mR^{(t)}}({\mX}_i^\star)\|_F^2\bigg|\bigg|\nabla_{L({\mX}_{i})} g(\mX_1^{(t)}, \dots, \mX_N^{(t)})\bigg|\bigg|_F\nonumber\\
    &\!\!\!\!\leq\!\!\!\!&\frac{3\ol{\sigma}(\calX^\star)}{2}\sum_{i=1}^{N-1}\|L(\mX_i^{(t)})-L_{\mR^{(t)}}({\mX}_i^\star)\|_F^2\|\calX^{(t)}-\calX^\star\|_F\nonumber\\
    &\!\!\!\!\leq\!\!\!\!&\frac{1}{10}\|\calX^{(t)}-\calX^\star\|_F^2+\frac{45(N-1)\ol{\sigma}^2(\calX^\star)}{8}\sum_{i=1}^{N-1}\|L(\mX_i^{(t)})-L_{\mR^{(t)}}(\mX_i^\star)\|_F^4\nonumber\\
    &\!\!\!\!\leq\!\!\!\!&\frac{1}{10}\|\calX^{(t)}-\calX^\star\|_F^2+\frac{45(N-1)}{8\ol{\sigma}^2(\calX^\star)}\text{dist}^4(\{\mX_i^{(t)} \},\{ \mX_i^\star \}),
\end{eqnarray}
where the second inequality follows \eqref{Riemannian GRADIENT DESCENT SQUARED TERM 1 to N noiseless sensing 1} with $\delta_{3 \ol r }= 1$. We now plug this into \eqref{Lower BOUND OF cross TERM OF Riemannian GD IN SENSING Conclusion original} to get
\begin{eqnarray}
    \label{Lower BOUND OF cross TERM OF Riemannian GD IN SENSING Conclusion}
    &\!\!\!\!\!\!\!\!&\sum_{i=1}^{N} \bigg\< L(\mX_i^{(t)})-L_{\mR^{(t)}}(\mX_i^\star),\calP_{\text{T}_{L({\mX}_i)} \text{St}}\bigg(\nabla_{L({\mX}_{i})} g(\mX_1^{(t)}, \dots, \mX_N^{(t)})\bigg)\bigg\>\nonumber\\
    &\!\!\!\!\geq\!\!\!\!&(1-\delta_{2\ol r})\|\calX^{(t)}-\calX^\star\|_F^2+\<L(\mX_1^{(t)}) \ol \otimes \cdots \ol \otimes L(\mX_N^{(t)})-L_{\mR^{(t)}}(\mX_1^\star) \ol \otimes \cdots \ol \otimes L_{\mR^{(t)}}({\mX}_N^\star), \vh^{(t)}\>\nonumber\\
    &\!\!\!\!\!\!\!\!&-\delta_{(N+3)\ol r}\|\calX^{(t)}-\calX^\star\|_F\|\vh^{(t)}\|_2 -\frac{1}{10}\|\calX^{(t)}-\calX^\star\|_F^2-\frac{45(N-1)}{8\ol{\sigma}^2(\calX^\star)}\text{dist}^4(\{\mX_i^{(t)} \},\{ \mX_i^\star \})\nonumber\\
    &\!\!\!\!\geq\!\!\!\!&\bigg(\frac{9}{10}-\delta_{(N+3)\ol r}\bigg)\|\calX^{(t)}-\calX^\star\|_F^2-\frac{1+\delta_{(N+3)\ol r}}{2}(\|\calX^{(t)}-\calX^\star\|_F^2+\|\vh^{(t)}\|_2^2)\nonumber\\
    &\!\!\!\!\!\!\!\!&-\frac{45(N-1)}{8\ol{\sigma}^2(\calX^\star)}\text{dist}^4(\{\mX_i^{(t)} \},\{ \mX_i^\star \})\nonumber\\
    &\!\!\!\!\geq\!\!\!\!&\frac{4 - 15\delta_{(N+3)\ol r}}{20}\|\calX^{(t)}-\calX^\star\|_F^2-\frac{19}{30}\|\vh^{(t)}\|_2^2-\frac{45(N-1)}{8\ol{\sigma}^2(\calX^\star)}
    \text{dist}^4(\{\mX_i^{(t)} \},\{ \mX_i^\star \})\nonumber\\
    &\!\!\!\!\geq\!\!\!\!&\frac{(4-15\delta_{(N+3)\ol r})\underline{\sigma}^2(\calX^\star)}{640(N+1+\sum_{i=2}^{N-1}r_i)\ol{\sigma}^2(\calX^\star)}\text{dist}^2(\{\mX_i^{(t)} \},\{ \mX_i^\star \}) + \frac{4 - 15\delta_{(N+3)\ol r}}{40}\|\calX^{(t)}-\calX^\star\|_F^2,
\end{eqnarray}
where  we utilize \eqref{PROJECTION ORTHOGONAL IN RIEMANNIAN UPPER BOUND in the sensing}, \Cref{RIP condition fro the tensor train sensing Lemma}, and {Lemma} \ref{RIP CONDITION FRO THE TENSOR TRAIN SENSING OTHER PROPERTY} in the first inequality. Note that according to the definition of $\vh^{(t)}$ in \eqref{H_T IN THE CROSS TERM}, it has TT ranks at most $((N+1)r_1,\dots,(N+1)r_{N-1})$. Therefore, with the TT format $L(\mX_1^{(t)}) \ol \otimes \cdots \ol \otimes L(\mX_N^{(t)})-L_{\mR^{(t)}}(\mX_1^\star) \ol \otimes \cdots \ol \otimes L_{\mR^{(t)}}({\mX}_N^\star)$, which has TT ranks $(2r_1,\dots,2r_{N-1})$,  the measurement operator $\calA$ needs to satisfy the $(N+3)\ol r$-RIP which is assumed. The third inequality follows because $\delta_{2\ol r}\leq\delta_{(N+3)\ol r}\leq\frac{4}{15}$.
The last line utilizes {Lemma} \ref{LOWER BOUND OF TWO DISTANCES}, \eqref{H_T IN THE CROSS TERM UPPER BOUND}, and the initial condition $\text{dist}^2(\{\mX_i^{(0)} \},\{ \mX_i^\star \})\leq \frac{(4-15\delta_{(N+3)\ol r})\underline{\sigma}^2(\calX^\star)}{8(N+1+\sum_{i=2}^{N-1}r_i)(57N^2+393N-450)} $.

\paragraph{Contraction}

Taking \eqref{UPPER BOUND OF SQUARED TERM OF Riemannian GD IN SENSING Conclusion} and \eqref{Lower BOUND OF cross TERM OF Riemannian GD IN SENSING Conclusion} into \eqref{Expansion of distance in the noiseless tensor sensing}, we can get
\begin{eqnarray}
    \label{Conclusion of Riemannian gradient descent in the sensing}
    &\!\!\!\!\!\!\!\!&\text{dist}^2(\{\mX_i^{(t+1)} \},\{ \mX_i^\star \})\nonumber\\
    &\!\!\!\!\leq\!\!\!\!&\bigg(1-\frac{(4-15\delta_{(N+3)\ol r})\underline{\sigma}^2(\calX^\star)}{320(N+1+\sum_{i=2}^{N-1}r_i)\ol{\sigma}^2(\calX^\star)}\mu\bigg)\text{dist}^2(\{\mX_i^{(t)} \},\{ \mX_i^\star \})\nonumber\\
    &\!\!\!\!\!\!\!\!&+ \bigg(\frac{9N-5}{2}(1+\delta_{3\ol r})^2\mu^2 - \frac{4 - 15\delta_{(N+3)\ol r}}{20}\mu\bigg) \|\calX^{(t)}-\calX^\star\|_F^2\nonumber\\
    &\!\!\!\!\leq\!\!\!\!&\bigg(1-\frac{(4-15\delta_{(N+3)\ol r})\underline{\sigma}^2(\calX^\star)}{320(N+1+\sum_{i=2}^{N-1}r_i)\ol{\sigma}^2(\calX^\star)}\mu\bigg)\text{dist}^2(\{\mX_i^{(t)} \},\{ \mX_i^\star \}),
\end{eqnarray}
where we use $\mu\leq \frac{4 - 15\delta_{(N+3)\ol r}}{10(9N -5)(1 + \delta_{(N+3)\ol r})^2}$ in the last line.

\paragraph{Proof of \eqref{eq:ini-cond-for-XN}} This can be proved by using the same induction argument in \eqref{eq:ini-cond-for-XN-factorization} together with $\delta_{(N+3)\ol r}\leq \frac{4}{15}$. This completes the proof.

\end{proof}

\section{Proof of \Cref{TENSOR SENSING noisy SPECTRAL INITIALIZATION} for Noisy Spectral Initialization}
\label{Proof of in the noisy spectral initialization}

\begin{proof}

Recalling the definition of $ \|\cdot\|_{F,\ol r}$ in \eqref{Definition of the restricted F norm}, we follow the same approach in \eqref{LGNWTSSM6_2} to quantify $\|\calX^{(0)}-\calX^\star\|_F$:
\begin{eqnarray}
    \label{The expansion of noisy spectral ini 1}
    &\!\!\!\!\!\!\!\!&\|\calX^{(0)}-\calX^\star\|_F\nonumber\\
    &\!\!\!\!=\!\!\!\!&\bigg\|\text{SVD}_{\vr}^{tt}\bigg(\frac{1}{m}\sum_{k=1}^m(y_k + \epsilon_k)\calA_k\bigg)-\calX^\star\bigg\|_{F,2\ol r}\nonumber\\
    &\!\!\!\!\leq\!\!\!\!&(1+\sqrt{N-1})\bigg\|\frac{1}{m}\sum_{k=1}^m(y_k + \epsilon_k)\calA_k-\calX^\star\bigg\|_{F,2\ol r}\nonumber\\
    &\!\!\!\!\leq\!\!\!\!&(1+\sqrt{N-1})\bigg\|\frac{1}{m}\sum_{k=1}^my_k\calA_k-\calX^\star\bigg\|_{F,2\ol r} + (1+\sqrt{N-1})\bigg\|\frac{1}{m}\sum_{k=1}^m\epsilon_k\calA_k \bigg\|_{F,2\ol r}\nonumber\\
    &\!\!\!\!\leq\!\!\!\!& \delta_{3\ol r}(1+\sqrt{N-1})\|\calX^\star\|_{F} + (1+\sqrt{N-1})\bigg\|\frac{1}{m}\sum_{k=1}^m\epsilon_k\calA_k \bigg\|_{F,2\ol r}.
\end{eqnarray}

Next, we will use an $\epsilon$-net and a covering argument to bound the second term in the last line:
\begin{eqnarray}
    \label{epsilon net variable}
    \hspace{-0.5cm}\bigg\|\frac{1}{m}\sum_{k=1}^m\epsilon_k\calA_k \bigg\|_{F,2\ol r}= \!\!\!\!\max_{\calH\in\R^{d_1\times\cdots\times d_N}, \|\calH\|_F\leq 1, \atop {\rm rank}(\calH)=(2r_1,\dots,2r_{N-1})} \<\frac{1}{m}\sum_{k=1}^m\epsilon_k\calA_k,  \calH \>=\!\!\!\!\!\!\!\! \max_{\calH\in\R^{d_1\times\cdots\times d_N},  \|\calH\|_F\leq 1, \atop {\rm rank}(\calH)=(2r_1,\dots,2r_{N-1})} \frac{1}{m} \<\vepsilon,  \calA(\calH) \>.
\end{eqnarray}
To begin, according to \citep{zhang2018tensor}, for each $i \in [N-1]$, we can construct an $\epsilon$-net $\{L(\mH_i^{(1)}), \dots, L(\mH_i^{(n_i)})  \}$ with the covering number $n_i\leq (\frac{4+\epsilon}{\epsilon})^{d_ir_{i-1}r_i}$ for the set of factors $\{L(\mH_i)\in\R^{d_ir_{i-1}\times r_i}: \|L(\mH_i)\|\leq 1\}$ such that
\begin{eqnarray}
    \label{ProofOf<H,X>forSubGaussian_proof1}
    \sup_{L(\mH_i): \|L(\mH_i)\|\leq 1}\min_{p_i\leq n_i} \|L(\mH_i)-L(\mH_i^{(p_i)})\|\leq \epsilon.
\end{eqnarray}
Similarly, we can construct an $\epsilon$-net $\{ L(\mH_N^{(1)}), \dots, L(\mH_N^{(n_N)}) \}$ with the covering number $n_N\leq (\frac{2+\epsilon}{\epsilon})^{d_Nr_{N-1}}$ for $\{L(\mH_N)\in\R^{d_Nr_{N-1}\times 1}: \|L(\mH_N)\|_2\leq 1  \}$ such that
\begin{eqnarray}
    \label{ProofOf<H,X>forSubGaussian_proof2}
    \sup_{L(\mH_N): \|L(\mH_N)\|_2\leq 1}\min_{p_N\leq n_N} \|L(\mH_N)-L(\mH_N^{(p_N)})\|_2\leq \epsilon.
\end{eqnarray}
Therefore, we can construct an $\epsilon$-net $\{\calH^{(1)},\ldots,\calH^{(n_1\cdots n_N)}\}$ with covering number
\[
\Pi_{i=1}^N n_i \leq (\frac{4+\epsilon}{\epsilon})^{d_1r_1+\sum_{i=2}^{N-1}d_ir_{i-1}r_i+d_Nr_{N-1}} \leq (\frac{4+\epsilon}{\epsilon})^{N\ol d\ol r^2 }
\]
(where $\ol r=\max_{i=1}^{N-1}r_i$ and $\ol d=\max_{i=1}^{N}d_i$) for any TT format tensors $\calH = [\mH_1,\dots, \mH_N]\in\R^{d_1\times \cdots \times d_N}$ with TT ranks $(r_1,\dots, r_{N-1})$.

Denote by $T$ the value of \eqref{epsilon net variable}, i.e.,
\begin{eqnarray}
    \label{ProofOf<H,X>forSubGaussian_proof3}
    &&[\widetilde\mH_1,\dots, \widetilde\mH_N]=  \argmax_{\mbox{\tiny$\begin{array}{c}
     L(\mH_i)\in\R^{2d_ir_{i-1}\times 2r_i}\\
     \|L(\mH_i)\|\leq 1, i \in [N-1]\\
     \|L(\mH_N)\|_2\leq 1 \end{array}$}}\frac{1}{m}\sum_{k=1}^m\<\epsilon_k\calA_k, [\mH_1,\dots, \mH_N]  \>,\\
    \label{ProofOf<H,X>forSubGaussian_proof4}
    &&T:= \frac{1}{m}\sum_{k=1}^m\<\epsilon_k\calA_k, [\widetilde\mH_1,\dots, \widetilde\mH_N]  \>.
\end{eqnarray}
Using $\calI$ to denote the index set $[n_1]\times \cdots \times [n_N]$, then according to the construction of the $\epsilon$-net, there exists $p=(p_1,\dots, p_N)\in\calI$ such that
\begin{eqnarray}
    \label{ProofOf<H,X>forSubGaussian_proof5}
    \|L(\widetilde\mH_i) - L(\mH_i^{(p_i)}) \|\leq\epsilon, \ \  i\in [N-1] \ \ \  \text{and}  \ \ \  \|L(\widetilde\mH_N) - L(\mH_N^{(p_N)})\|_2\leq\epsilon.
\end{eqnarray}
Now taking $\epsilon=\frac{1}{2N}$ gives
\begin{eqnarray}
    \label{ProofOf<H,X>forSubGaussian_proof6}
    \hspace{-0.3cm}T&\!\!\!\!=\!\!\!\!&\frac{1}{m}\sum_{k=1}^m\<\epsilon_k\calA_k, [\mH_1^{(p_1)},\dots, \mH_N^{(p_N)}]  \>+ \frac{1}{m}\sum_{k=1}^m\<\epsilon_k\calA_k, [\widetilde\mH_1,\dots, \widetilde\mH_N] - [\mH_1^{(p_1)},\dots, \mH_N^{(p_N)}]  \>\nonumber\\
    \hspace{-0.3cm}&\!\!\!\!=\!\!\!\!&\frac{1}{m}\sum_{k=1}^m\<\epsilon_k\calA_k, [\mH_1^{(p_1)},\dots, \mH_N^{(p_N)}]  \>+ \frac{1}{m}\sum_{k=1}^m\<\epsilon_k\calA_k, \sum_{a_1=1}^N[\mH_1^{(p_1)},\dots, \mH_{a_1}^{(p_{a_1})}-\widetilde\mH_{a_1},  \dots, \widetilde\mH_N]\>\nonumber\\
    \hspace{-0.3cm}&\!\!\!\!\leq\!\!\!\!&\frac{1}{m}\sum_{k=1}^m\<\epsilon_k\calA_k, [\mH_1^{(p_1)},\dots, \mH_N^{(p_N)}]  \> + N \epsilon T\nonumber\\
    \hspace{-0.3cm}&\!\!\!\!=\!\!\!\!&\frac{1}{m}\sum_{k=1}^m\<\epsilon_k\calA_k, [\mH_1^{(p_1)},\dots, \mH_N^{(p_N)}]  \>+\frac{T}{2},
\end{eqnarray}
where the second line uses {Lemma} \ref{EXPANSION_A1TOAN-B1TOBN_1} to rewrite $[\widetilde\mH_1,\dots, \widetilde\mH_N] - [\mH_1^{(p_1)},\dots, \mH_N^{(p_N)}]$ into a sum of $N$ terms.

Note that when conditioned on $\{\calA_k\}_{k=1}^m$, for any fixed $\calH^{(p)}\in\R^{d_1\times\cdots\times d_N}$, $\frac{1}{m}\<\vepsilon,  \calA(\calH^{(p)}) \>$ has a normal distribution with zero mean and variance $\frac{\gamma^2\|\calA(\calH^{(p)})\|_2^2}{m^2}$, which implies that
\begin{eqnarray}
    \label{the tail function of fixed gaussian random variable}
    \P{\frac{1}{m}|\<\vepsilon,  \calA(\calH^{(p)}) \>| \geq t | \{\calA_k\}_{k=1}^m}\leq e^{-\frac{m^2t^2}{2\gamma^2\|\calA(\calH^{(p)})\|_2^2}}.
\end{eqnarray}
Furthermore, under the event $F:=\{\calA \text{ satisfies $2\ol r$-RIP with constant $\delta_{2\ol r}$}\}$, which implies that  $\frac{1}{m}\|\calA(\calH^{(p)})\|_2^2\leq(1+\delta_{2\ol r})\|\calH^{(p)}\|_F^2$, plugging this together with the fact $\|\calH^{(p)}\|_F\leq 1$ into the above further gives
\begin{eqnarray}
    \label{the tail function of fixed gaussian random variable1}
    \P{\frac{1}{m}|\<\vepsilon,  \calA(\calH^{(p)}) \>| \geq t | F}\leq e^{-\frac{mt^2}{2(1+\delta_{2\ol r})\gamma^2}}.
\end{eqnarray}

We now apply this tail bound to \eqref{ProofOf<H,X>forSubGaussian_proof6} and get
\begin{eqnarray}
    \label{the tail function of fixed gaussian random variable 2}
    \P{T \geq t | F} &\!\!\!\!\leq\!\!\!\!& \P{\max_{p_1,\dots, p_n} \frac{1}{m}\sum_{k=1}^m\<\epsilon_k\calA_k, [\mH_1^{(p_1)},\dots, \mH_N^{(p_N)}]  \> \geq \frac{t}{2} | F}\nonumber\\
    &\!\!\!\!\leq\!\!\!\!& \bigg(\frac{4+\epsilon}{\epsilon}\bigg)^{4N\ol d\ol r^2 }e^{-\frac{mt^2}{8(1+\delta_{2\ol r})\gamma^2}}\leq e^{-\frac{mt^2}{8(1+\delta_{2\ol r})\gamma^2} + c_1N\ol d\ol r^2 \log N},
\end{eqnarray}
where $c_1$ is a constant and based on the assumption in \eqref{ProofOf<H,X>forSubGaussian_proof6}, $\frac{4+\epsilon}{\epsilon}=\frac{4+\frac{1}{2N}}{\frac{1}{2N}}=8N+1$.

Hence, we can take $t = \frac{c_2\ol r\sqrt{(1+\delta_{2\ol r})N\ol d(\log N)}}{\sqrt{m}}\gamma$ with a constant $c_2$ and further derive
\begin{eqnarray}
    \label{the tail function of fixed gaussian random variable 3}
    &\!\!\!\!\!\!\!\!&\P{T \leq \frac{c_2\ol r\sqrt{(1+\delta_{2\ol r})N\ol d(\log N)}}{\sqrt{m}}\gamma } \nonumber\\
    &\!\!\!\!\geq\!\!\!\!& \P{T \leq \frac{c_2\ol r\sqrt{(1+\delta_{2\ol r})N\ol d(\log N)}}{\sqrt{m}}\gamma \cap F } \nonumber\\
    &\!\!\!\!\geq\!\!\!\!&P(F) \P{T \leq \frac{c_2\ol r\sqrt{(1+\delta_{2\ol r})N\ol d(\log N)}}{\sqrt{m}}\gamma |F }\nonumber\\
    &\!\!\!\!\geq\!\!\!\!&(1-e^{-c_3N\ol d\ol r^2 \log N})(1-e^{-c_4N\ol d\ol r^2 \log N})\geq 1-2e^{-c_5N\ol d\ol r^2 \log N},
\end{eqnarray}
where $c_i,i=3,4,5$ are constants. Note that $P(F)$ is obtained via \Cref{RIP condition fro the tensor train sensing Lemma} by setting $\epsilon$ in \eqref{eq:mrip} to be $e^{-c_3N\ol d\ol r^2 \log N}$.

Combing \eqref{The expansion of noisy spectral ini 1} and \eqref{the tail function of fixed gaussian random variable 3}, we finally obtain that with probability $1-2e^{-c_5N\ol d\ol r^2 \log N}$,
\begin{eqnarray}
    \label{Noisy TENSOR SENSING SPECTRAL INITIALIZATION1 conclusion}
    \|\calX^{(0)}-\calX^\star\|_F\leq (1+\sqrt{N-1})\bigg(\delta_{3\ol r}\|\calX^\star\|_F + \frac{c_2\ol r\sqrt{(1+\delta_{3\ol r})N\ol d\log N} }{\sqrt{m}}\gamma\bigg),
\end{eqnarray}
where $\delta_{2\ol r}\leq \delta_{3\ol r}$ is used.

\end{proof}

\section{Proof of \Cref{Local Convergence of Riemannian in the noisy sensing_Theorem} for Noisy TT Format Tensor Sensing}
\label{Local Convergence Proof of Riemannman gradient descent Noisy Tensor Sensing}

\begin{proof}
By the same analysis in the beginning of {Appendix} \ref{Local Convergence Proof of Riemannman gradient descent Tensor Sensing}, we can get that $L(\mX_i^{(t)}), i \in [N-1]$ are orthonormal matrices and $\sigma_1^2({\calX^{(t)}}^{\<i \>}) =   \|{\calX^{(t)}}^{\geq i+1} \|^2\leq \frac{9\ol{\sigma}^2(\calX^\star)}{4}, i\in[N-1]$, $t\geq 0$ by assuming
\begin{eqnarray}
    \label{required condition for LN condition}
    \text{dist}^2(\{\mX_i^{(t)} \},\{ \mX_i^\star \})\leq\frac{\underline{\sigma}^2(\calX^\star)}{180N},
\end{eqnarray}
which will be proved later. Now recall \eqref{expansion of distance in tensor factorization-main paper}:
\begin{eqnarray}
    \label{Expansion of distance in the noisy tensor sensing}
    &\!\!\!\!\!\!\!\!&\hspace{-1.1cm}\text{dist}^2(\{\mX_i^{(t+1)} \},\{ \mX_i^\star \})\nonumber\\
    &\!\!\!\!\!\!\!\!&\hspace{-1.5cm}\leq\text{dist}^2(\{\mX_i^{(t)} \},\{ \mX_i^\star \})-2\mu\sum_{i=1}^{N} \bigg\< L(\mX_i^{(t)})-L_{\mR^{(t)}}(\mX_i^\star),\calP_{\text{T}_{L({\mX}_i)} \text{St}}\bigg(\nabla_{L({\mX}_{i})} G(\mX_1^{(t)}, \dots, \mX_N^{(t)})\bigg)\bigg\>\nonumber\\
    &\!\!\!\!\!\!\!\!&\hspace{-1.5cm}+\mu^2\bigg(\frac{1}{\ol{\sigma}^2(\calX^\star)}\sum_{i=1}^{N-1}\bigg\|\calP_{\text{T}_{L({\mX}_i)} \text{St}}\bigg(\nabla_{L({\mX}_{i})} G(\mX_1^{(t)}, \dots, \mX_N^{(t)})\bigg)\bigg\|_F^2+\bigg\|\nabla_{L({\mX}_{N})} G(\mX_1^{(t)}, \dots, \mX_N^{(t)}) \bigg\|_2^2\bigg),\nonumber\\
\end{eqnarray}
where the gradient with respect to each factor $\mX_{i}(s_i)$ can be computed as
\begin{align*}
\nabla_{\mX_i(s_i)}G(\mX_1^{(t)}, \dots, \mX_N^{(t)})
=\frac{1}{m}\sum_{k=1}^{m} &(\<\calA_k,\calX^{(t)}\>-y_k - \epsilon_k) \sum_{s_1,\ldots,s_{i-1},s_{i+1},\ldots,s_N } \Big( \calA_k(s_1,\dots,s_N)\cdot \\
&
\mX_{i-1}^{(t)}(s_{i-1})^\top \cdots
     \mX_{1}^{(t)}(s_{1})^\top \mX_{N}^{(t)}(s_{N})^\top   \cdots  \mX_{i+1}^{(t)}(s_{i+1})^\top\Big).
\end{align*}

\paragraph{Upper bound of the third term in \eqref{Expansion of distance in the noisy tensor sensing}}
To upper bound $\|\nabla_{L({\mX}_{i})} G(\mX_1^{(t)},\dots,\mX_N^{(t)})\|_F$, we first analyze the difference in the gradient caused by noise, using the same analysis in \eqref{the tail function of fixed gaussian random variable 3}. Specifically, with the same $\epsilon$-net argument in \eqref{epsilon net variable},
\begin{equation}
\begin{split}
&\left\| \nabla_{L({\mX}_{i})} g(\mX_1^{(t)}, \dots, \mX_N^{(t)}) - \nabla_{L({\mX}_{i})} G(\mX_1^{(t)}, \dots, \mX_N^{(t)}) \right\|_F\\
= & \max_{\mH_i\in\R^{r_{i-1}\times d_i\times r_i} \atop \|\mH_i\|_F\leq 1}\
 \left< \nabla_{L({\mX}_{i})} g(\mX_1^{(t)}, \dots, \mX_N^{(t)}) - \nabla_{L({\mX}_{i})} G(\mX_1^{(t)},\dots,\mX_N^{(t)}), L(\mH_i) \right\> \\
= & \max_{\mH_i\in\R^{r_{i-1}\times d_i\times r_i} \atop\|\mH_i\|_F\leq 1}\bigg\<\frac{1}{m}\sum_{k=1}^{m}\epsilon_k\calA_k, [\mX_1^{(t)},\dots, \mX_{i-1}^{(t)}, \mH_i, \mX_{i+1}^{(t)}, \dots,\mX_N^{(t)} ]   \bigg\> \nonumber\\
 \le &  \begin{cases} \frac{c_i\ol r\sqrt{(1+\delta_{3\ol r})N\ol d(\log N) }\gamma \ol{\sigma}(\calX^\star)}{\sqrt{m}}, & i = 1,\dots, N-1, \\ \frac{c_N\ol r\sqrt{(1+\delta_{3\ol r})N\ol d(\log N) }\gamma}{\sqrt{m}}, & i = N, \end{cases}
\end{split}
\end{equation}
where the last inequality holds with probability at least $1 - 2Ne^{-\Omega(N\ol d\ol r^2 \log N)}$ with $c_i,i\in [N]$ being positive constants, and is derived by using \eqref{KRONECKER PRODUCT VECTORIZATION11 - 3} that $\|[\mX_1^{(t)},\dots, \mH_i, \dots,\mX_N^{(t)} ]\|_F = \|(L(\mX_1^{(t)})\ol \otimes \\ \cdots \ol \otimes L(\mH_i)){\calX^{(t)}}^{\geq i+1} \|_F \leq \|L(\mX_1^{(t)})\ol \otimes \cdots \ol \otimes L(\mH_i)\|_F\|{\calX^{(t)}}^{\geq i+1}\|\leq \frac{3\ol{\sigma}(\calX^\star)}{2}, i\in[N-1]$ and $\|[\mX_1^{(t)},\dots, \mX_{N-1}^{(t)}, \mH_N ]\|_F=\|\mH_N\|_F\leq 1$.

This together with the bound for $\big\| \nabla_{L({\mX}_{i})} g(\mX_1^{(t)}, \dots, \mX_N^{(t)})  \big\|_F$ in \eqref{Riemannian GRADIENT DESCENT SQUARED TERM 1 to N noiseless sensing 1} gives
\begin{eqnarray}
\label{upper bound of error term in the squared noisy}
&\!\!\!\!\!\!\!\!&\bigg\| \nabla_{L({\mX}_{i})} G(\mX_1^{(t)}, \dots, \mX_N^{(t)})  \bigg\|_F\nonumber\\
&\!\!\!\!\leq\!\!\!\!& \bigg\| \nabla_{L({\mX}_{i})} g(\mX_1^{(t)}, \dots, \mX_N^{(t)})  \bigg\|_F + \left\| \nabla_{L({\mX}_{i})} g(\mX_1^{(t)}, \dots, \mX_N^{(t)}) - \nabla_{L({\mX}_{i})} G(\mX_1^{(t)}, \dots, \mX_N^{(t)}) \right\|_F\nonumber\\
&\!\!\!\!\leq\!\!\!\!& \begin{cases}
    \frac{3\ol{\sigma}(\calX^\star)}{2}(1+\delta_{3\ol r})\|\calX^{(t)}-\calX^\star\|_F + \frac{c_i\ol r\sqrt{(1+\delta_{3\ol r})N\ol d(\log N) }\gamma \ol{\sigma}(\calX^\star)}{\sqrt{m}}, & i = 1,\dots, N-1,\\
    (1+\delta_{3\ol r})\|\calX^{(t)}-\calX^\star\|_F + \frac{c_N\ol r\sqrt{(1+\delta_{3\ol r})N\ol d(\log N) }\gamma}{\sqrt{m}}, & i = N.
  \end{cases}
\end{eqnarray}

We now plug the above into the third term in  \eqref{the tail function of fixed gaussian random variable 3} to get
\begin{eqnarray}
    \label{UPPER BOUND OF SQUARED TERM OF Riemannian GD IN noisy SENSING Conclusion}
    &\!\!\!\!\!\!\!\!&\frac{1}{\ol{\sigma}^2(\calX^\star)}\sum_{i=1}^{N-1}\bigg\|\calP_{\text{T}_{L({\mX}_i)} \text{St}}\bigg(\nabla_{L({\mX}_{i})} G(\mX_1^{(t)}, \dots, \mX_N^{(t)})\bigg)\bigg\|_F^2+\bigg\|\nabla_{L({\mX}_{N})} G(\mX_1^{(t)}, \dots, \mX_N^{(t)}) \bigg\|_2^2\nonumber\\
    &\!\!\!\!\leq\!\!\!\!&\frac{1}{\ol{\sigma}^2(\calX^\star)}\sum_{i=1}^{N-1}\bigg\|\nabla_{L({\mX}_{i})} G(\mX_1^{(t)}, \dots, \mX_N^{(t)})\bigg\|_F^2+\bigg\|\nabla_{L({\mX}_{N})} G(\mX_1^{(t)}, \dots, \mX_N^{(t)}) \bigg\|_2^2\nonumber\\
    &\!\!\!\!\leq\!\!\!\!&\frac{9N-5}{2}(1+\delta_{3\ol r})^2\|\calX^{(t)} - \calX^\star\|_F^2 + O\bigg(\frac{(1+\delta_{3\ol r})N^2\ol d\ol r^2(\log N) \gamma^2 }{m}\bigg).
\end{eqnarray}

\paragraph{Lower bound of the second term in \eqref{Expansion of distance in the noisy tensor sensing}}
To  apply the same approach as in \eqref{Lower BOUND OF cross TERM OF Riemannian GD IN SENSING Conclusion} for establishing a lower bound for the second term in \eqref{Expansion of distance in the noisy tensor sensing}, we first need to establish upper bounds for two terms involving noise. To begin, following the derivation of \eqref{PROJECTION ORTHOGONAL IN RIEMANNIAN UPPER BOUND in the sensing},  we can get
   \begin{eqnarray}
    \label{PROJECTION ORTHOGONAL IN RIEMANNIAN UPPER BOUND in the noisy sensing}
    &\!\!\!\!\!\!\!\!&\sum_{i=1}^{N-1}\bigg\<\calP^{\perp}_{\text{T}_{L({\mX}_i)} \text{St}}(L(\mX_i^{(t)})-L_{\mR^{(t)}}(\mX_i^\star)), \nabla_{L({\mX}_{i})} G(\mX_1^{(t)}, \dots, \mX_N^{(t)}) \bigg\>\nonumber\\
    &\!\!\!\!\leq\!\!\!\!&\frac{1}{2}\sum_{i=1}^{N-1}\|L(\mX_i^{(t)})\|\|L(\mX_i^{(t)})-L_{\mR^{(t)}}({\mX}_i^\star)\|_F^2\bigg|\bigg|\nabla_{L({\mX}_{i})} G(\mX_1^{(t)}, \dots, \mX_N^{(t)})\bigg|\bigg|_F\nonumber\\
    &\!\!\!\!\leq\!\!\!\!&\frac{3\ol{\sigma}(\calX^\star)}{2}\sum_{i=1}^{N-1}\|L(\mX_i^{(t)})-L_{\mR^{(t)}}({\mX}_i^\star)\|_F^2\|\calX^{(t)}-\calX^\star\|_F\nonumber\\
    &&\!\!\!\!\!\!\!\!+ \sum_{i=1}^{N-1}\frac{c_i\ol r\sqrt{(1+\delta_{3\ol r})N\ol d (\log N) }\gamma \ol{\sigma}(\calX^\star)}{\sqrt{m}}\|L(\mX_i^{(t)})-L_{\mR^{(t)}}({\mX}_i^\star)\|_F^2 \nonumber\\
    &\!\!\!\!\leq\!\!\!\!&\frac{1}{20}\|\calX^{(t)}-\calX^\star\|_F^2+46(N-1)\ol{\sigma}^2(\calX^\star)\sum_{i=1}^{N-1}\|L(\mX_i^{(t)})-L_{\mR^{(t)}}(\mX_i^\star)\|_F^4\nonumber\\
    &\!\!\!\!\!\!\!\!&+ \sum_{i=1}^{N-1} \frac{c_i^2(1+\delta_{3\ol r}) N\ol d\ol r^2(\log N)\gamma^2 }{16m} \nonumber\\
    &\!\!\!\!\leq\!\!\!\!&\frac{1}{20}\|\calX^{(t)}-\calX^\star\|_F^2+\frac{46(N-1)}{\ol{\sigma}^2(\calX^\star)}\text{dist}^4(\{\mX_i^{(t)} \},\{ \mX_i^\star \}) + O\bigg(\frac{(1+\delta_{3\ol r}) N^2\ol d\ol r^2(\log N)\gamma^2 }{m}\bigg),\nonumber\\
    \end{eqnarray}
where the second inequality uses \eqref{upper bound of error term in the squared noisy}. In addition, recalling the notations of $\va_k=\text{vec}(\calA_k)$ and $\vh^{(t)}$ defined in \eqref{H_T IN THE CROSS TERM}, then with probability $1 - 2e^{-\Omega(N^3\ol d\ol r^2 \log N)}$, we have
    \begin{eqnarray}
    \label{PROJECTION ORTHOGONAL IN RIEMANNIAN UPPER BOUND in the noisy sensing 1}
    &\!\!\!\!\!\!\!\!&\frac{1}{m}\sum_{k=1}^{m}\<\epsilon_k\va_k,  L(\mX_1^{(t)}) \ol \otimes \cdots \ol \otimes L(\mX_N^{(t)})-L_{\mR^{(t)}}(\mX_1^\star) \ol \otimes \cdots \ol \otimes L_{\mR^{(t)}}({\mX}_N^\star) + \vh^{(t)}\>\nonumber\\
    &\!\!\!\!\leq\!\!\!\!& \frac{C (N+3)\ol r \sqrt{(1+\delta_{(N+3)\ol r})N\ol d(\log N)}\gamma}{\sqrt{m}} \|L(\mX_1^{(t)}) \ol \otimes \cdots \ol \otimes L(\mX_N^{(t)})\nonumber\\
    &\!\!\!\!\!\!\!\!&-L_{\mR^{(t)}}(\mX_1^\star) \ol \otimes \cdots \ol \otimes L_{\mR^{(t)}}({\mX}_N^\star) + \vh^{(t)}\|_F\nonumber\\
    &\!\!\!\!\leq\!\!\!\!& \frac{5C^2(1+\delta_{(N+3)\ol r})N(N+3)^2\ol d\ol r^2(\log N) \gamma^2}{m} + \frac{1}{10}\|\calX^{(t)} - \calX^\star\|_F^2 + \frac{1}{10}\|\vh^{(t)}\|_F^2\nonumber\\
    &\!\!\!\!\leq\!\!\!\!& \frac{5C^2(1+\delta_{(N+3)\ol r})N(N+3)^2\ol d\ol r^2(\log N) \gamma^2}{m} + \frac{1}{10}\|\calX^{(t)} - \calX^\star\|_F^2\nonumber\\
    &\!\!\!\!\!\!\!\!&+ \frac{9N(N-1)}{80\ol{\sigma}^2(\calX^\star)}\text{dist}^4(\{\mX_i^{(t)} \},\{ \mX_i^\star \}),
    \end{eqnarray}
where the first inequality follows the same $\epsilon$-net argument used in \eqref{epsilon net variable} and the fact that the TT ranks of the second term in the cross term is $((N+3)r_1,\dots, (N+3)r_{N-1})$, and the last inequality uses \eqref{H_T IN THE CROSS TERM UPPER BOUND}.

Using \eqref{PROJECTION ORTHOGONAL IN RIEMANNIAN UPPER BOUND in the noisy sensing}, we can proceed with the analysis similar to \eqref{Lower BOUND OF cross TERM OF Riemannian GD IN SENSING Conclusion} to obtain the following derivation
\begin{eqnarray}
   \label{Lower BOUND OF cross TERM OF Riemannian GD IN noisy SENSING Conclusion}
    &\!\!\!\!\!\!\!\!&\sum_{i=1}^{N} \bigg\< L(\mX_i^{(t)})-L_{\mR^{(t)}}(\mX_i^\star),\calP_{\text{T}_{L({\mX}_i)} \text{St}}\bigg(\nabla_{L({\mX}_{i})} G(\mX_1^{(t)}, \dots, \mX_N^{(t)})\bigg)\bigg\>\nonumber\\
    &\!\!\!\!\geq\!\!\!\!&(\frac{9}{20}-\frac{3\delta_{(N+3)\ol r}}{2})\|\calX^{(t)}-\calX^\star\|_F^2- \frac{1+\delta_{(N+3)\ol r}}{2}\|\vh^{(t)}\|_F^2-\frac{46(N-1)}{\ol{\sigma}^2(\calX^\star)}\text{dist}^4(\{\mX_i^{(t)} \},\{ \mX_i^\star \})\nonumber\\
    &\!\!\!\!\!\!\!\!&-\frac{1}{m}\sum_{k=1}^{m}\<\epsilon_k\va_k,  L(\mX_1^{(t)}) \ol \otimes \cdots \ol \otimes L(\mX_N^{(t)})-L_{\mR^{(t)}}(\mX_1^\star) \ol \otimes \cdots \ol \otimes L_{\mR^{(t)}}({\mX}_N^\star) + \vh^{(t)}\>\nonumber\\
    &\!\!\!\!\!\!\!\!&- O\bigg(\frac{(1+\delta_{3\ol r}) N^2\ol d\ol r^2(\log N)\gamma^2 }{m}\bigg)\nonumber\\
    &\!\!\!\!\geq\!\!\!\!&\frac{7 - 30\delta_{(N+3)\ol r}}{40}\|\calX^{(t)}-\calX^\star\|_F^2-\frac{129N^2+7231N-7360}{160\ol{\sigma}^2(\calX^\star)}\text{dist}^4(\{\mX_i^{(t)} \},\{ \mX_i^\star \})\nonumber\\
    &\!\!\!\!\!\!\!\!& - O\bigg(\frac{(1+\delta_{(N+3)\ol r}) N^3\ol d\ol r^2(\log N)\gamma^2 }{m}\bigg)\nonumber\\
    &\!\!\!\!\geq\!\!\!\!& \frac{(7 - 30\delta_{(N+3)\ol r})\underline{\sigma}^2(\calX^\star)}{1280(N+1+\sum_{i=2}^{N-1}r_i)\ol{\sigma}^2(\calX^\star)}\text{dist}^2(\{\mX_i^{(t)} \},\{ \mX_i^\star \}) + \frac{7 - 30\delta_{(N+3)\ol r}}{80}\|\calX^{(t)}-\calX^\star\|_F^2\nonumber\\
    &\!\!\!\!\!\!\!\!&- O\bigg(\frac{(1+\delta_{(N+3)\ol r}) N^3\ol d\ol r^2(\log N)\gamma^2 }{m}\bigg),
\end{eqnarray}
where we use $\delta_{(N+3)\ol r}\leq\frac{7}{30}$, \eqref{PROJECTION ORTHOGONAL IN RIEMANNIAN UPPER BOUND in the noisy sensing 1} and \eqref{H_T IN THE CROSS TERM UPPER BOUND} in the second inequality, and the last line follows {Lemma} \ref{LOWER BOUND OF TWO DISTANCES} and the initial condition $\text{dist}^2(\{\mX_i^{(0)} \},\{ \mX_i^\star \})\leq \frac{(7 - 30\delta_{(N+3)\ol r})\underline{\sigma}^2(\calX^\star)}{8(N+1+\sum_{i=2}^{N-1}r_i)(129N^2+7231N-7360)} $.

\paragraph{Contraction}

Taking \eqref{UPPER BOUND OF SQUARED TERM OF Riemannian GD IN noisy SENSING Conclusion} and \eqref{Lower BOUND OF cross TERM OF Riemannian GD IN noisy SENSING Conclusion} into \eqref{Expansion of distance in the noisy tensor sensing}, with probability $1 - 2Ne^{-\Omega(N\ol d\ol r^2 \log N)}  - 2e^{-\Omega(N^3\ol d\ol r^2 \log N)}$, we can get
\begin{eqnarray}
    \label{Conclusion of Riemannian gradient descent in the noisy sensing}
    \text{dist}^2(\{\mX_i^{(t+1)} \},\{ \mX_i^\star \})&\!\!\!\!\leq\!\!\!\!&\bigg(1-\frac{(7 - 30\delta_{(N+3)\ol r})\underline{\sigma}^2(\calX^\star)}{1280(N+1+\sum_{i=2}^{N-1}r_i)\ol{\sigma}^2(\calX^\star)}\mu\bigg)\text{dist}^2(\{\mX_i^{(t)} \},\{ \mX_i^\star \})\nonumber\\
    &\!\!\!\!\!\!\!\!&+\bigg(\frac{9N-5}{2}(1+\delta_{3\ol r})^2\mu^2  -\frac{7 - 30\delta_{(N+3)\ol r}}{40}\mu \bigg)\|\calX^{(t)}-\calX^\star\|_F^2\nonumber\\
    &\!\!\!\!\!\!\!\!&+ O\bigg(\frac{(1+\delta_{(N+3)\ol r})N^2\ol d\ol r^2(\log N) \gamma^2}{m}(\mu N + \mu^2)\bigg)\nonumber\\
    &\!\!\!\!\leq\!\!\!\!&\bigg(1-\frac{(7 - 30\delta_{(N+3)\ol r})\underline{\sigma}^2(\calX^\star)}{1280(N+1+\sum_{i=2}^{N-1}r_i)\ol{\sigma}^2(\calX^\star)}\mu\bigg)\text{dist}^2(\{\mX_i^{(t)} \},\{ \mX_i^\star \})\nonumber\\
    &\!\!\!\!\!\!\!\!&+ O\bigg(\frac{(1+\delta_{(N+3)\ol r})N^2\ol d\ol r^2(\log N) \gamma^2}{m}(\mu N + \mu^2)\bigg),
\end{eqnarray}
where we use $\mu\leq\frac{7 - 30\delta_{(N+3)\ol r}}{20(9N-5)(1+\delta_{(N+3)\ol r})^2}$ in the last line. By induction, this further implies that
\begin{eqnarray}
    \label{Simiplified Conclusion of Riemannian gradient descent in the noisy sensing}
    &\!\!\!\!\!\!\!\!&\text{dist}^2(\{\mX_i^{(t+1)} \},\{ \mX_i^\star \})\nonumber\\
    &\!\!\!\!\leq\!\!\!\!&\bigg(1-\frac{(7 - 30\delta_{(N+3)\ol r})\underline{\sigma}^2(\calX^\star)}{1280(N+1+\sum_{i=2}^{N-1}r_i)\ol{\sigma}^2(\calX^\star)}\mu\bigg)^{t+1}
    \text{dist}^2(\{\mX_i^{(0)} \},\{ \mX_i^\star \})\nonumber\\
    &\!\!\!\!\!\!\!\!&+  O\bigg(\frac{(N + \mu)(N+1+\sum_{i=2}^{N-1}r_i)(1+\delta_{(N+3)\ol r}) N^2\ol d\ol r^2(\log N) \ol{\sigma}^2(\calX^\star)\gamma^2 }{m(7 - 30\delta_{(N+3)\ol r})\underline{\sigma}^2(\calX^\star)}\bigg).
\end{eqnarray}

\paragraph{Proof of \eqref{required condition for LN condition}} We can prove it by induction as used in the proof of \eqref{eq:ini-cond-for-XN-factorization}. First note that \eqref{eq:ini-cond-for-XN} holds for $t = 0$. We now assume it holds for all $t \le t'$, which implies that $\sigma_1^2({\calX^{(t')}}^{\<i \>}) =   \|{\calX^{(t')}}^{\geq i+1} \|^2\leq \frac{9\ol{\sigma}^2(\calX^\star)}{4}, i\in[N-1]$. By invoking \eqref{Simiplified Conclusion of Riemannian gradient descent in the noisy sensing}, we have
\begin{eqnarray}
   &\!\!\!\!\!\!\!\!& \text{dist}^2(\{\mX_i^{(t'+1)}\},\{\mX_i^\star\}) \nonumber\\
    &\!\!\!\! \le\!\!\!\!& \text{dist}^2(\{\mX_i^{(0)} \},\{ \mX_i^\star \}) + O\bigg(\frac{(N + \mu)(N+1+\sum_{i=2}^{N-1}r_i)(1+\delta_{(N+3)\ol r}) N^2\ol d\ol r^2(\log N) \ol{\sigma}^2(\calX^\star)\gamma^2 }{m(7 - 30\delta_{(N+3)\ol r})\underline{\sigma}^2(\calX^\star)}\bigg)\nonumber\\
   &\!\!\!\! \le\!\!\!\!& \frac{\underline{\sigma}^2(\calX^\star)}{180N},\nonumber
\end{eqnarray}
as long as $
    m\geq C \frac{N^5\ol d \ol r^3 (\log N) \ol{\sigma}^2(\calX^\star)\gamma^2}{\underline{\sigma}^4(\calX^\star)} $ with a universal constant $C$. Consequently, \eqref{required condition for LN condition} also holds  at $ t= t'+1$. By induction, we can conclude that \eqref{required condition for LN condition} holds for all $t\ge 0$. This completes the proof.

\end{proof}

\vskip 0.2in

\begin{thebibliography}{84}
\providecommand{\natexlab}[1]{#1}
\providecommand{\url}[1]{\texttt{#1}}
\expandafter\ifx\csname urlstyle\endcsname\relax
  \providecommand{\doi}[1]{doi: #1}\else
  \providecommand{\doi}{doi: \begingroup \urlstyle{rm}\Url}\fi

\bibitem[Absil et~al.(2008)Absil, Mahony, and Sepulchre]{absil2008optimization}
P-A Absil, Robert Mahony, and Rodolphe Sepulchre.
\newblock \emph{Optimization algorithms on matrix manifolds}.
\newblock Princeton University Press, 2008.

\bibitem[Acar and Yener(2008)]{AcarUnsup09}
Evrim Acar and B{\"u}lent Yener.
\newblock Unsupervised multiway data analysis: A literature survey.
\newblock \emph{IEEE transactions on knowledge and data engineering},
  21\penalty0 (1):\penalty0 6--20, 2008.

\bibitem[Bengua et~al.(2017)Bengua, Phien, Tuan, and Do]{bengua2017efficient}
Johann~A Bengua, Ho~N Phien, Hoang~Duong Tuan, and Minh~N Do.
\newblock Efficient tensor completion for color image and video recovery:
  Low-rank tensor train.
\newblock \emph{IEEE Transactions on Image Processing}, 26\penalty0
  (5):\penalty0 2466--2479, 2017.

\bibitem[Bierm{\'e} and Lacaux(2015)]{bierme2015modulus}
Hermine Bierm{\'e} and C{\'e}line Lacaux.
\newblock Modulus of continuity of some conditionally sub-gaussian fields,
  application to stable random fields.
\newblock \emph{Bernoulli}, 21\penalty0 (3):\penalty0 1719--1759, 2015.

\bibitem[Bro(1997)]{Bro97}
Rasmus Bro.
\newblock Parafac. {T}utorial and applications.
\newblock \emph{Chemometrics and intelligent laboratory systems}, 38\penalty0
  (2):\penalty0 149--171, 1997.

\bibitem[Budzinskiy and Zamarashkin(2021)]{budzinskiy2021tensor}
Stanislav Budzinskiy and Nikolai Zamarashkin.
\newblock Tensor train completion: local recovery guarantees via riemannian
  optimization.
\newblock \emph{arXiv preprint arXiv:2110.03975}, 2021.

\bibitem[Cai et~al.(2019)Cai, Li, Poor, and Chen]{cai2019nonconvex}
Changxiao Cai, Gen Li, H~Vincent Poor, and Yuxin Chen.
\newblock Nonconvex low-rank tensor completion from noisy data.
\newblock \emph{Advances in neural information processing systems}, 32, 2019.

\bibitem[Cai et~al.(2022)Cai, Li, and Xia]{Cai2022provable}
Jian-Feng Cai, Jingyang Li, and Dong Xia.
\newblock Provable tensor-train format tensor completion by riemannian
  optimization.
\newblock \emph{Journal of Machine Learning Research}, 23\penalty0
  (123):\penalty0 1--77, 2022.

\bibitem[Cand{\`e}s and Plan(2011)]{CandsTIT11}
Emmanuel~J Cand{\`e}s and Yaniv Plan.
\newblock Tight oracle inequalities for low-rank matrix recovery from a minimal
  number of noisy random measurements.
\newblock \emph{IEEE Transactions on Information Theory}, 57\penalty0
  (4):\penalty0 2342--2359, 2011.

\bibitem[Cand{\`e}s and Wakin(2008)]{candes2008introduction}
Emmanuel~J Cand{\`e}s and Michael~B Wakin.
\newblock An introduction to compressive sampling.
\newblock \emph{IEEE signal processing magazine}, 25\penalty0 (2):\penalty0
  21--30, 2008.

\bibitem[Cand{\`e}s et~al.(2006)Cand{\`e}s, Romberg, and Tao]{candes2006robust}
Emmanuel~J Cand{\`e}s, Justin Romberg, and Terence Tao.
\newblock Robust uncertainty principles: Exact signal reconstruction from
  highly incomplete frequency information.
\newblock \emph{IEEE Transactions on information theory}, 52\penalty0
  (2):\penalty0 489--509, 2006.

\bibitem[Cand{\`e}s et~al.(2015)Cand{\`e}s, Li, and
  Soltanolkotabi]{candes2015phase}
Emmanuel~J Cand{\`e}s, Xiaodong Li, and Mahdi Soltanolkotabi.
\newblock Phase retrieval via wirtinger flow: Theory and algorithms.
\newblock \emph{IEEE Transactions on Information Theory}, 61\penalty0
  (4):\penalty0 1985--2007, 2015.

\bibitem[Cichocki(2014)]{cichocki2014tensor}
Andrzej Cichocki.
\newblock Tensor networks for big data analytics and large-scale optimization
  problems.
\newblock \emph{arXiv preprint arXiv:1407.3124}, 2014.

\bibitem[Cichocki et~al.(2015)Cichocki, Mandic, De~Lathauwer, Zhou, Zhao,
  Caiafa, and Phan]{CichockiMagTensor15}
Andrzej Cichocki, Danilo Mandic, Lieven De~Lathauwer, Guoxu Zhou, Qibin Zhao,
  Cesar Caiafa, and Huy~Anh Phan.
\newblock Tensor decompositions for signal processing applications: {F}rom
  two-way to multiway component analysis.
\newblock \emph{IEEE signal processing magazine}, 32\penalty0 (2):\penalty0
  145--163, 2015.

\bibitem[De~Silva and Lim(2008)]{de2008tensor}
Vin De~Silva and Lek-Heng Lim.
\newblock Tensor rank and the ill-posedness of the best low-rank approximation
  problem.
\newblock \emph{SIAM Journal on Matrix Analysis and Applications}, 30\penalty0
  (3):\penalty0 1084--1127, 2008.

\bibitem[Ding et~al.(2022)Ding, Qin, Jiang, Zhou, and Zhu]{ding2022validation}
Lijun Ding, Zhen Qin, Liwei Jiang, Jinxin Zhou, and Zhihui Zhu.
\newblock A validation approach to over-parameterized matrix and image
  recovery.
\newblock \emph{arXiv preprint arXiv:2209.10675}, 2022.

\bibitem[Donoho(2006)]{donoho2006compressed}
David~L Donoho.
\newblock Compressed sensing.
\newblock \emph{IEEE Transactions on information theory}, 52\penalty0
  (4):\penalty0 1289--1306, 2006.

\bibitem[Frolov and Oseledets(2017)]{frolov2017tensor}
Evgeny Frolov and Ivan Oseledets.
\newblock Tensor methods and recommender systems.
\newblock \emph{Wiley Interdisciplinary Reviews: Data Mining and Knowledge
  Discovery}, 7\penalty0 (3):\penalty0 e1201, 2017.

\bibitem[Grotheer et~al.(2021)Grotheer, Li, Ma, Needell, and
  Qin]{grotheer2021iterative}
Rachel Grotheer, Shuang Li, Anna Ma, Deanna Needell, and Jing Qin.
\newblock Iterative hard thresholding for low cp-rank tensor models.
\newblock \emph{Linear and Multilinear Algebra}, pages 1--17, 2021.

\bibitem[Guo et~al.(2011)Guo, Kotsia, and Patras]{guo2011tensor}
Weiwei Guo, Irene Kotsia, and Ioannis Patras.
\newblock Tensor learning for regression.
\newblock \emph{IEEE Transactions on Image Processing}, 21\penalty0
  (2):\penalty0 816--827, 2011.

\bibitem[Han et~al.(2022)Han, Willett, and Zhang]{Han20}
Rungang Han, Rebecca Willett, and Anru~R Zhang.
\newblock An optimal statistical and computational framework for generalized
  tensor estimation.
\newblock \emph{The Annals of Statistics}, 50\penalty0 (1):\penalty0 1--29,
  2022.

\bibitem[Hao et~al.(2020)Hao, Zhang, and Cheng]{hao2020sparse}
Botao Hao, Anru~R Zhang, and Guang Cheng.
\newblock Sparse and low-rank tensor estimation via cubic sketchings.
\newblock In \emph{International Conference on Artificial Intelligence and
  Statistics}, pages 1319--1330. PMLR, 2020.

\bibitem[H{\aa}stad(1989)]{haastad1989tensor}
Johan H{\aa}stad.
\newblock Tensor rank is np-complete.
\newblock In \emph{Automata, Languages and Programming: 16th International
  Colloquium Stresa, Italy, July 11--15, 1989 Proceedings 16}, pages 451--460.
  Springer, 1989.

\bibitem[Hillar and Lim(2013)]{hillar2013most}
Christopher~J Hillar and Lek-Heng Lim.
\newblock Most tensor problems are np-hard.
\newblock \emph{Journal of the ACM (JACM)}, 60\penalty0 (6):\penalty0 1--39,
  2013.

\bibitem[Holtz et~al.(2012)Holtz, Rohwedder, and Schneider]{holtz2012manifolds}
Sebastian Holtz, Thorsten Rohwedder, and Reinhold Schneider.
\newblock On manifolds of tensors of fixed tt-rank.
\newblock \emph{Numerische Mathematik}, 120\penalty0 (4):\penalty0 701--731,
  2012.

\bibitem[Hore et~al.(2016)Hore, Vinuela, Buil, Knight, McCarthy, Small, and
  Marchini]{HoreNature16}
Victoria Hore, Ana Vinuela, Alfonso Buil, Julian Knight, Mark~I McCarthy,
  Kerrin Small, and Jonathan Marchini.
\newblock Tensor decomposition for multiple-tissue gene expression experiments.
\newblock \emph{Nature genetics}, 48\penalty0 (9):\penalty0 1094--1100, 2016.

\bibitem[Jameson et~al.(2024)Jameson, Qin, Goldar, Wakin, Zhu, and
  Gong]{jameson2024optimal}
Casey Jameson, Zhen Qin, Alireza Goldar, Michael~B Wakin, Zhihui Zhu, and
  Zhexuan Gong.
\newblock Optimal quantum state tomography with local informationally complete
  measurements.
\newblock \emph{arXiv preprint arXiv:2408.07115}, 2024.

\bibitem[Jiang et~al.(2022)Jiang, Chen, and Ding]{jiang2022algorithmic}
Liwei Jiang, Yudong Chen, and Lijun Ding.
\newblock Algorithmic regularization in model-free overparametrized asymmetric
  matrix factorization.
\newblock \emph{arXiv preprint arXiv:2203.02839}, 2022.

\bibitem[Jin et~al.(2017)Jin, Ge, Netrapalli, Kakade, and Jordan]{Jin17}
Chi Jin, Rong Ge, Praneeth Netrapalli, Sham~M Kakade, and Michael~I Jordan.
\newblock How to escape saddle points efficiently.
\newblock In \emph{International conference on machine learning}, pages
  1724--1732. PMLR, 2017.

\bibitem[Khrulkov et~al.(2017)Khrulkov, Novikov, and
  Oseledets]{khrulkov2017expressive}
Valentin Khrulkov, Alexander Novikov, and Ivan Oseledets.
\newblock Expressive power of recurrent neural networks.
\newblock \emph{arXiv preprint arXiv:1711.00811}, 2017.

\bibitem[Kolda and Bader(2009)]{kolda2009tensor}
Tamara~G Kolda and Brett~W Bader.
\newblock Tensor decompositions and applications.
\newblock \emph{SIAM review}, 51\penalty0 (3):\penalty0 455--500, 2009.

\bibitem[Kuznetsov and Oseledets(2019)]{kuznetsov2019tensor}
Maxim~A Kuznetsov and Ivan~V Oseledets.
\newblock Tensor train spectral method for learning of hidden markov models
  (hmm).
\newblock \emph{Computational Methods in Applied Mathematics}, 19\penalty0
  (1):\penalty0 93--99, 2019.

\bibitem[Latorre(2005)]{latorre2005image}
Jose~I Latorre.
\newblock Image compression and entanglement.
\newblock \emph{arXiv preprint quant-ph/0510031}, 2005.

\bibitem[Li and Zhang(2017)]{li2017parsimonious}
Lexin Li and Xin Zhang.
\newblock Parsimonious tensor response regression.
\newblock \emph{Journal of the American Statistical Association}, 112\penalty0
  (519):\penalty0 1131--1146, 2017.

\bibitem[Li et~al.(2020{\natexlab{a}})Li, Li, Zhu, Tang, and
  Wakin]{li2020global}
Shuang Li, Qiuwei Li, Zhihui Zhu, Gongguo Tang, and Michael~B Wakin.
\newblock The global geometry of centralized and distributed low-rank matrix
  recovery without regularization.
\newblock \emph{IEEE Signal Processing Letters}, 27:\penalty0 1400--1404,
  2020{\natexlab{a}}.

\bibitem[Li et~al.(2020{\natexlab{b}})Li, Zhu, Man-Cho~So, and Vidal]{Li20}
Xiao Li, Zhihui Zhu, Anthony Man-Cho~So, and Rene Vidal.
\newblock Nonconvex robust low-rank matrix recovery.
\newblock \emph{SIAM Journal on Optimization}, 30\penalty0 (1):\penalty0
  660--686, 2020{\natexlab{b}}.

\bibitem[Li et~al.(2021)Li, Chen, Deng, Qu, Zhu, and Man-Cho~So]{LiSIAM21}
Xiao Li, Shixiang Chen, Zengde Deng, Qing Qu, Zhihui Zhu, and Anthony
  Man-Cho~So.
\newblock Weakly convex optimization over {S}tiefel manifold using {R}iemannian
  subgradient-type methods.
\newblock \emph{SIAM Journal on Optimization}, 31\penalty0 (3):\penalty0
  1605--1634, 2021.

\bibitem[Lidiak et~al.(2022)Lidiak, Jameson, Qin, Tang, Wakin, Zhu, and
  Gong]{lidiak2022quantum}
Alexander Lidiak, Casey Jameson, Zhen Qin, Gongguo Tang, Michael~B Wakin,
  Zhihui Zhu, and Zhexuan Gong.
\newblock Quantum state tomography with tensor train cross approximation.
\newblock \emph{arXiv preprint arXiv:2207.06397}, 2022.

\bibitem[Lu and Li(2020)]{lu2020phase}
Yue~M Lu and Gen Li.
\newblock Phase transitions of spectral initialization for high-dimensional
  non-convex estimation.
\newblock \emph{Information and Inference: A Journal of the IMA}, 9\penalty0
  (3):\penalty0 507--541, 2020.

\bibitem[Luo et~al.(2019)Luo, Alghamdi, and Lu]{luo2019optimal}
Wangyu Luo, Wael Alghamdi, and Yue~M Lu.
\newblock Optimal spectral initialization for signal recovery with applications
  to phase retrieval.
\newblock \emph{IEEE Transactions on Signal Processing}, 67\penalty0
  (9):\penalty0 2347--2356, 2019.

\bibitem[Ma et~al.(2021)Ma, Li, and Chi]{Ma21TSP}
Cong Ma, Yuanxin Li, and Yuejie Chi.
\newblock Beyond procrustes: Balancing-free gradient descent for asymmetric
  low-rank matrix sensing.
\newblock \emph{IEEE Transactions on Signal Processing}, 69:\penalty0 867--877,
  2021.

\bibitem[Ma et~al.(2019)Ma, Zhang, Zhang, Duan, Hou, Zhou, and
  Song]{ma2019tensorized}
Xindian Ma, Peng Zhang, Shuai Zhang, Nan Duan, Yuexian Hou, Ming Zhou, and
  Dawei Song.
\newblock A tensorized transformer for language modeling.
\newblock \emph{Advances in neural information processing systems}, 32, 2019.

\bibitem[Novikov et~al.(2015)Novikov, Podoprikhin, Osokin, and
  Vetrov]{novikov2015tensorizing}
Alexander Novikov, Dmitrii Podoprikhin, Anton Osokin, and Dmitry~P Vetrov.
\newblock Tensorizing neural networks.
\newblock \emph{Advances in neural information processing systems}, 28, 2015.

\bibitem[Novikov et~al.(2021)Novikov, Panov, and Oseledets]{novikov2021tensor}
Georgii~S Novikov, Maxim~E Panov, and Ivan~V Oseledets.
\newblock Tensor-train density estimation.
\newblock In \emph{Uncertainty in artificial intelligence}, pages 1321--1331.
  PMLR, 2021.

\bibitem[Ohliger et~al.(2013)Ohliger, Nesme, and Eisert]{ohliger2013efficient}
Matthias Ohliger, Vincent Nesme, and Jens Eisert.
\newblock Efficient and feasible state tomography of quantum many-body systems.
\newblock \emph{New Journal of Physics}, 15\penalty0 (1):\penalty0 015024,
  2013.

\bibitem[Oseledets(2011)]{Oseledets11}
Ivan~V Oseledets.
\newblock Tensor-train decomposition.
\newblock \emph{SIAM Journal on Scientific Computing}, 33\penalty0
  (5):\penalty0 2295--2317, 2011.

\bibitem[Park et~al.(2017)Park, Kyrillidis, Carmanis, and
  Sanghavi]{park2017non}
Dohyung Park, Anastasios Kyrillidis, Constantine Carmanis, and Sujay Sanghavi.
\newblock Non-square matrix sensing without spurious local minima via the
  burer-monteiro approach.
\newblock In \emph{Artificial Intelligence and Statistics}, pages 65--74. PMLR,
  2017.

\bibitem[Qi et~al.(2022)Qi, Yang, Chen, and Tejedor]{qi2022exploiting}
Jun Qi, Chao-Han~Huck Yang, Pin-Yu Chen, and Javier Tejedor.
\newblock Exploiting low-rank tensor-train deep neural networks based on
  riemannian gradient descent with illustrations of speech processing.
\newblock \emph{arXiv preprint arXiv:2203.06031}, 2022.

\bibitem[Qin and Zhu(2024)]{qin2024computational}
Zhen Qin and Zhihui Zhu.
\newblock Computational and statistical guarantees for tensor-on-tensor
  regression with tensor train decomposition.
\newblock \emph{arXiv preprint arXiv:2406.06002}, 2024.

\bibitem[Qin et~al.(2024)Qin, Jameson, Gong, Wakin, and Zhu]{qin2024quantum}
Zhen Qin, Casey Jameson, Zhexuan Gong, Michael~B Wakin, and Zhihui Zhu.
\newblock Quantum state tomography for matrix product density operators.
\newblock \emph{IEEE Transactions on Information Theory}, 70\penalty0
  (7):\penalty0 5030--5056, 2024.

\bibitem[Rauhut et~al.(2015)Rauhut, Schneider, and Stojanac]{rauhut2015tensor}
Holger Rauhut, Reinhold Schneider, and {\v{Z}}eljka Stojanac.
\newblock Tensor completion in hierarchical tensor representations.
\newblock In \emph{Compressed sensing and its applications}, pages 419--450.
  Springer, 2015.

\bibitem[Rauhut et~al.(2017)Rauhut, Schneider, and Stojanac]{Rauhut17}
Holger Rauhut, Reinhold Schneider, and {\v{Z}}eljka Stojanac.
\newblock Low rank tensor recovery via iterative hard thresholding.
\newblock \emph{Linear Algebra and its Applications}, 523:\penalty0 220--262,
  2017.

\bibitem[Recht et~al.(2010)Recht, Fazel, and Parrilo]{recht2010guaranteed}
Benjamin Recht, Maryam Fazel, and Pablo~A Parrilo.
\newblock Guaranteed minimum-rank solutions of linear matrix equations via
  nuclear norm minimization.
\newblock \emph{SIAM review}, 52\penalty0 (3):\penalty0 471--501, 2010.

\bibitem[Schollw{\"o}ck(2011)]{schollwock2011density}
Ulrich Schollw{\"o}ck.
\newblock The density-matrix renormalization group in the age of matrix product
  states.
\newblock \emph{Annals of physics}, 326\penalty0 (1):\penalty0 96--192, 2011.

\bibitem[Sidiropoulos et~al.(2000)Sidiropoulos, Giannakis, and
  Bro]{SidiropoulosBlind}
Nicholas~D Sidiropoulos, Georgios~B Giannakis, and Rasmus Bro.
\newblock Blind {PARAFAC} receivers for {DS-CDMA} systems.
\newblock \emph{IEEE Transactions on Signal Processing}, 48\penalty0
  (3):\penalty0 810--823, 2000.

\bibitem[Sidiropoulos et~al.(2017)Sidiropoulos, De~Lathauwer, Fu, Huang,
  Papalexakis, and Faloutsos]{SidiropoulosTSPTENSOR17}
Nicholas~D Sidiropoulos, Lieven De~Lathauwer, Xiao Fu, Kejun Huang, Evangelos~E
  Papalexakis, and Christos Faloutsos.
\newblock Tensor decomposition for signal processing and machine learning.
\newblock \emph{IEEE Transactions on Signal Processing}, 65\penalty0
  (13):\penalty0 3551--3582, 2017.

\bibitem[Smilde et~al.(2005)Smilde, Geladi, and Bro]{Smilde04}
Age~K Smilde, Paul Geladi, and Rasmus Bro.
\newblock \emph{Multi-way analysis: applications in the chemical sciences}.
\newblock John Wiley \& Sons, 2005.

\bibitem[St{\"o}ger and Soltanolkotabi(2021)]{stoger2021small}
Dominik St{\"o}ger and Mahdi Soltanolkotabi.
\newblock Small random initialization is akin to spectral learning:
  Optimization and generalization guarantees for overparameterized low-rank
  matrix reconstruction.
\newblock \emph{Advances in Neural Information Processing Systems},
  34:\penalty0 23831--23843, 2021.

\bibitem[Stoudenmire and Schwab(2016)]{stoudenmire2016supervised}
Edwin Stoudenmire and David~J Schwab.
\newblock Supervised learning with tensor networks.
\newblock \emph{Advances in neural information processing systems}, 29, 2016.

\bibitem[Tjandra et~al.(2017)Tjandra, Sakti, and
  Nakamura]{tjandra2017compressing}
Andros Tjandra, Sakriani Sakti, and Satoshi Nakamura.
\newblock Compressing recurrent neural network with tensor train.
\newblock In \emph{2017 International Joint Conference on Neural Networks
  (IJCNN)}, pages 4451--4458. IEEE, 2017.

\bibitem[Tong(2022)]{tong2022scaled}
Tian Tong.
\newblock \emph{Scaled gradient methods for ill-conditioned low-rank matrix and
  tensor estimation}.
\newblock PhD thesis, Carnegie Mellon University, 2022.

\bibitem[Tong et~al.(2021{\natexlab{a}})Tong, Ma, and
  Chi]{tong2021accelerating}
Tian Tong, Cong Ma, and Yuejie Chi.
\newblock Accelerating ill-conditioned low-rank matrix estimation via scaled
  gradient descent.
\newblock \emph{J. Mach. Learn. Res.}, 22:\penalty0 150--1, 2021{\natexlab{a}}.

\bibitem[Tong et~al.(2021{\natexlab{b}})Tong, Ma, Prater-Bennette, Tripp, and
  Chi]{TongTensor21}
Tian Tong, Cong Ma, Ashley Prater-Bennette, Erin Tripp, and Yuejie Chi.
\newblock Scaling and scalability: Provable nonconvex low-rank tensor
  estimation from incomplete measurements.
\newblock \emph{arXiv preprint arXiv:2104.14526}, Nov. 2021{\natexlab{b}}.

\bibitem[Tu et~al.(2016)Tu, Boczar, Simchowitz, Soltanolkotabi, and
  Recht]{Tu16}
Stephen Tu, Ross Boczar, Max Simchowitz, Mahdi Soltanolkotabi, and Ben Recht.
\newblock Low-rank solutions of linear matrix equations via procrustes flow.
\newblock In \emph{International Conference on Machine Learning}, pages
  964--973. PMLR, 2016.

\bibitem[Tucker(1966)]{Tucker66}
Ledyard~R Tucker.
\newblock Some mathematical notes on three-mode factor analysis.
\newblock \emph{Psychometrika}, 31\penalty0 (3):\penalty0 279--311, 1966.

\bibitem[Verstraete and Cirac(2006)]{verstraete2006matrix}
Frank Verstraete and J~Ignacio Cirac.
\newblock Matrix product states represent ground states faithfully.
\newblock \emph{Physical review b}, 73\penalty0 (9):\penalty0 094423, 2006.

\bibitem[Verstraete et~al.(2008)Verstraete, Murg, and
  Cirac]{verstraete2008matrix}
Frank Verstraete, Valentin Murg, and J~Ignacio Cirac.
\newblock Matrix product states, projected entangled pair states, and
  variational renormalization group methods for quantum spin systems.
\newblock \emph{Advances in physics}, 57\penalty0 (2):\penalty0 143--224, 2008.

\bibitem[Vidal et~al.(2022)Vidal, Zhu, and Haeffele]{vidal2022optimization}
Ren{\'e} Vidal, Zhihui Zhu, and Benjamin~D Haeffele.
\newblock Optimization landscape of neural networks.
\newblock \emph{Mathematical Aspects of Deep Learning}, page 200, 2022.

\bibitem[Wang et~al.(2019{\natexlab{a}})Wang, Song, Wu, Lai, and
  Jin]{wang2019latent}
Andong Wang, Xulin Song, Xiyin Wu, Zhihui Lai, and Zhong Jin.
\newblock Latent schatten tt norm for tensor completion.
\newblock In \emph{ICASSP 2019-2019 IEEE International Conference on Acoustics,
  Speech and Signal Processing (ICASSP)}, pages 2922--2926. IEEE,
  2019{\natexlab{a}}.

\bibitem[Wang et~al.(2019{\natexlab{b}})Wang, Zhao, Wang, and
  Li]{wang2019tensor}
Junli Wang, Guangshe Zhao, Dingheng Wang, and Guoqi Li.
\newblock Tensor completion using low-rank tensor train decomposition by
  riemannian optimization.
\newblock In \emph{2019 Chinese Automation Congress (CAC)}, pages 3380--3384.
  IEEE, 2019{\natexlab{b}}.

\bibitem[Wang et~al.(2017)Wang, Zhang, and Gu]{Wang17}
Lingxiao Wang, Xiao Zhang, and Quanquan Gu.
\newblock A unified computational and statistical framework for nonconvex
  low-rank matrix estimation.
\newblock In \emph{Artificial Intelligence and Statistics}, pages 981--990.
  PMLR, 2017.

\bibitem[Wang et~al.(2016)Wang, Aggarwal, and Aeron]{wang2016tensor}
Wenqi Wang, Vaneet Aggarwal, and Shuchin Aeron.
\newblock Tensor completion by alternating minimization under the tensor train
  (tt) model.
\newblock \emph{arXiv preprint arXiv:1609.05587}, 2016.

\bibitem[Xia and Yuan(2019)]{XiaTC19}
Dong Xia and Ming Yuan.
\newblock On polynomial time methods for exact low-rank tensor completion.
\newblock \emph{Foundations of Computational Mathematics}, 19\penalty0
  (6):\penalty0 1265--1313, 2019.

\bibitem[Xu et~al.(2023)Xu, Shen, Chi, and Ma]{xu2023power}
Xingyu Xu, Yandi Shen, Yuejie Chi, and Cong Ma.
\newblock The power of preconditioning in overparameterized low-rank matrix
  sensing.
\newblock \emph{arXiv preprint arXiv:2302.01186}, 2023.

\bibitem[Yang et~al.(2017)Yang, Krompass, and Tresp]{yang2017tensor}
Yinchong Yang, Denis Krompass, and Volker Tresp.
\newblock Tensor-train recurrent neural networks for video classification.
\newblock In \emph{International Conference on Machine Learning}, pages
  3891--3900. PMLR, 2017.

\bibitem[Yu et~al.(2017)Yu, Zheng, Anandkumar, and Yue]{yu2017long}
Rose Yu, Stephan Zheng, Anima Anandkumar, and Yisong Yue.
\newblock Long-term forecasting using tensor-train rnns.
\newblock \emph{Arxiv}, 2017.

\bibitem[Yuan et~al.(2019{\natexlab{a}})Yuan, Li, Mandic, Cao, and
  Zhao]{yuan2019tensor}
Longhao Yuan, Chao Li, Danilo Mandic, Jianting Cao, and Qibin Zhao.
\newblock Tensor ring decomposition with rank minimization on latent space: An
  efficient approach for tensor completion.
\newblock In \emph{Proceedings of the AAAI conference on artificial
  intelligence}, volume~33, pages 9151--9158, 2019{\natexlab{a}}.

\bibitem[Yuan et~al.(2019{\natexlab{b}})Yuan, Zhao, Gui, and Cao]{yuan2019high}
Longhao Yuan, Qibin Zhao, Lihua Gui, and Jianting Cao.
\newblock High-order tensor completion via gradient-based optimization under
  tensor train format.
\newblock \emph{Signal Processing: Image Communication}, 73:\penalty0 53--61,
  2019{\natexlab{b}}.

\bibitem[Zhang et~al.(2021{\natexlab{a}})Zhang, Fattahi, and
  Zhang]{zhang2021preconditioned}
Jialun Zhang, Salar Fattahi, and Richard~Y Zhang.
\newblock Preconditioned gradient descent for over-parameterized nonconvex
  matrix factorization.
\newblock \emph{Advances in Neural Information Processing Systems},
  34:\penalty0 5985--5996, 2021{\natexlab{a}}.

\bibitem[Zhang et~al.(2021{\natexlab{b}})Zhang, Ma, Qi, and
  Jin]{zhang2021designing}
Jing Zhang, Xiaoli Ma, Jun Qi, and Shi Jin.
\newblock Designing tensor-train deep neural networks for time-varying mimo
  channel estimation.
\newblock \emph{IEEE Journal of Selected Topics in Signal Processing},
  15\penalty0 (3):\penalty0 759--773, 2021{\natexlab{b}}.

\bibitem[Zhang et~al.(2018)Zhang, Yang, Chen, and Li]{zhang2018tensor}
Qingchen Zhang, Laurence~T Yang, Zhikui Chen, and Peng Li.
\newblock A tensor-train deep computation model for industry informatics big
  data feature learning.
\newblock \emph{IEEE Transactions on Industrial Informatics}, 14\penalty0
  (7):\penalty0 3197--3204, 2018.

\bibitem[Zhou et~al.(2013)Zhou, Li, and Zhu]{zhou2013tensor}
Hua Zhou, Lexin Li, and Hongtu Zhu.
\newblock Tensor regression with applications in neuroimaging data analysis.
\newblock \emph{Journal of the American Statistical Association}, 108\penalty0
  (502):\penalty0 540--552, 2013.

\bibitem[Zhu et~al.(2018)Zhu, Li, Tang, and Wakin]{Zhu18TSP}
Zhihui Zhu, Qiuwei Li, Gongguo Tang, and Michael~B Wakin.
\newblock Global optimality in low-rank matrix optimization.
\newblock \emph{IEEE Transactions on Signal Processing}, 66\penalty0
  (13):\penalty0 3614--3628, 2018.

\bibitem[Zhu et~al.(2021)Zhu, Li, Tang, and Wakin]{Zhu21TIT}
Zhihui Zhu, Qiuwei Li, Gongguo Tang, and Michael~B Wakin.
\newblock The global optimization geometry of low-rank matrix optimization.
\newblock \emph{IEEE Transactions on Information Theory}, 67\penalty0
  (2):\penalty0 1308--1331, 2021.

\end{thebibliography}

\end{document}